\documentclass{article}
\usepackage[T1]{fontenc}
\PassOptionsToPackage{table}{xcolor}
\usepackage{preprint_style,times}
\usepackage{amsmath,amssymb,amsfonts,amsthm,mathtools,bm}
\usepackage{booktabs,tabularx,array,multirow}
\usepackage{graphicx}
\usepackage{wrapfig}
\usepackage{algorithm}
\usepackage[noend]{algpseudocode}
\usepackage{xcolor}
\definecolor{DeepCiteBlue}{RGB}{0,51,153}
\usepackage{caption}
\usepackage{microtype}
\usepackage{hyperref}
\usepackage{url}
\hypersetup{
  colorlinks=true,
  citecolor=DeepCiteBlue,
  urlcolor=blue,
  linkcolor=black,
  filecolor=blue,
  pdfborder={0 0 0}
}
\allowdisplaybreaks
\newcommand{\R}{\mathbb{R}}
\newcommand{\E}{\mathbb{E}}
\newcommand{\Prob}{\mathbb{P}}
\newcommand{\cE}{\mathcal{E}}

\newcommand{\cL}{\mathcal{L}}
\newcommand{\cN}{\mathcal{N}}

\newcommand{\ba}{\bm a}
\newcommand{\op}{\mathrm{op}}

\newcommand{\diag}{\operatorname{diag}}

\let\standardtop\top
\newcommand{\loweredtop}[2]{\raisebox{-0.25ex}{\ensuremath{#1\standardtop}}}
\renewcommand{\top}{{\mathord{\mathpalette\loweredtop\relax}}}

\newcommand{\lemsr}[2]{\ensuremath{#1{\scriptstyle\,\pm #2}}}
\newcommand{\lembsr}[2]{\ensuremath{\mathbf{#1}{\scriptstyle\,\pm #2}}}

\newtheorem{theorem}{Theorem}[section]
\newtheorem{proposition}[theorem]{Proposition}
\newtheorem{corollary}[theorem]{Corollary}

\theoremstyle{definition}

\newtheorem{remark}[theorem]{Remark}
\newtheorem{assumption}[theorem]{Assumption}
\theoremstyle{plain}

\title{Control-Geometry Straightening for\\
Sampling-Based Latent Planning}
\author{%
\textbf{Ziang Fu}$^{1}$ \quad \textbf{Ning Ning}$^{1}$\hspace{-0.12em}\thanks{Corresponding author: \texttt{patning@tamu.edu}}\\[5pt]
{\normalfont $^{1}$Texas A\&M University}
}
\preprintfinalcopy

\begin{document}
% Hide main context entries in the appendix-only table of contents.
\addtocontents{toc}{\protect\setcounter{tocdepth}{-1}}
\maketitle
\vspace{-15pt}
\lhead{Preprint}

\begin{abstract}
Joint-embedding predictive architectures enable planning with latent
world models, but accurate transition prediction alone does not ensure
that the planning objective is easy to optimize.
We introduce Control-Geometry Straightening (CGS), a single auxiliary
loss that learns planner-friendly representations by directly
straightening control geometry for sampling-efficient planning.
CGS matches pairwise cosine similarities among actions to those among
corresponding latent differences only using local transitions from pixel--action pairs.
The loss can be applied across world-model architectures using end-to-end
learned or pretrained representations.
Under linear-dynamics, our theoretical analysis
connects this objective to temporal straightening and more balanced
terminal-cost curvature across the full planning horizon, yielding
finite-budget guarantees for MPPI, local contraction results for CEM,
and convergence bounds for gradient descent.
Across four control environments and multiple planners, CGS improves
planning with fewer sampled candidates and refinement steps, achieving
success-rate gains up to $20$ and $12.6$ percentage points over
LeWorldModel (LeWM) and its temporal-straightening variant (LeWM+TS),
respectively, with sampling-based planners using $128$ candidates
per update.
Probes, comparisons with DINO-WM architecture, and planner-side ablations clarify how
latent motion organization, state dependence, and dynamical context
shape planning behavior.
Straightening control geometry thus makes good action sequences
easier to find under limited planning budgets. See code \href{https://github.com/jaon11/CGS}{\textcolor{DeepCiteBlue}{here}}.
\end{abstract}

% ===== BEGIN intro.tex =====
\section{Introduction}
\label{sec:introduction}

Latent world models enable agents to plan by predicting the consequences of
actions in a learned representation space
\citep{hafner2019planet,lecun2022path}.
Joint-embedding predictive architectures (JEPAs) predict future embeddings
\citep{assran2023ijepa,bardes2024vjepa}.
DINO-WM demonstrates the effectiveness of prediction on pretrained visual
features \citep{zhou2025dinowm}, while LeWorldModel (LeWM) enables stable
end-to-end learning of the encoder and predictor directly from pixels
\citep{maes2026lewm}. Across these approaches, planning relies on anti-collapsing predictive representations, while the geometry exposed to action optimization remains underspecified.

Sampling-based planners address nonconvex action-sequence optimization by evaluating predicted rollouts in parallel. CEM and MPPI optimize action sequences through elite selection or cost-weighted updates iteratively
\citep{deboer2005cem,williams2017mppi,yi2024covo}. A central bottleneck is
{planning time}, especially under tight sampling budgets, since each refinement requires many multi-step model
rollouts. We therefore ask: \emph{what representation geometry
makes good action sequences easier to find with limited sampled candidates
and refinement steps?}

Moving from predictive to \emph{planner-friendly representations} requires more than accurate transition prediction: the representation geometry must also make the planning objective easier to optimize.
Temporal Straightening (TS) regularizes latent-trajectory curvature, making latent distances reflect feasible progress and improving conditioning for gradient-based (GD) planning
\citep{wang2026ts}.
LeWM already exhibits implicit temporal straightening without a curvature loss
\citep{maes2026lewm}; adding explicit TS shows environment-dependent effects in straightness and planning
(Figure~\ref{fig:lewm-ts-cosine-effect}; Table~\ref{tab:lewm-main-k128}).
Yet TS aligns consecutive latent differences without explicitly structuring how actions shape latent transitions.
This motivates directly shaping \emph{control geometry}: how action effects are represented and accumulate through learned dynamics to determine the terminal-cost landscape.

\begin{figure*}[t]
    \centering
    \includegraphics[width=\textwidth]{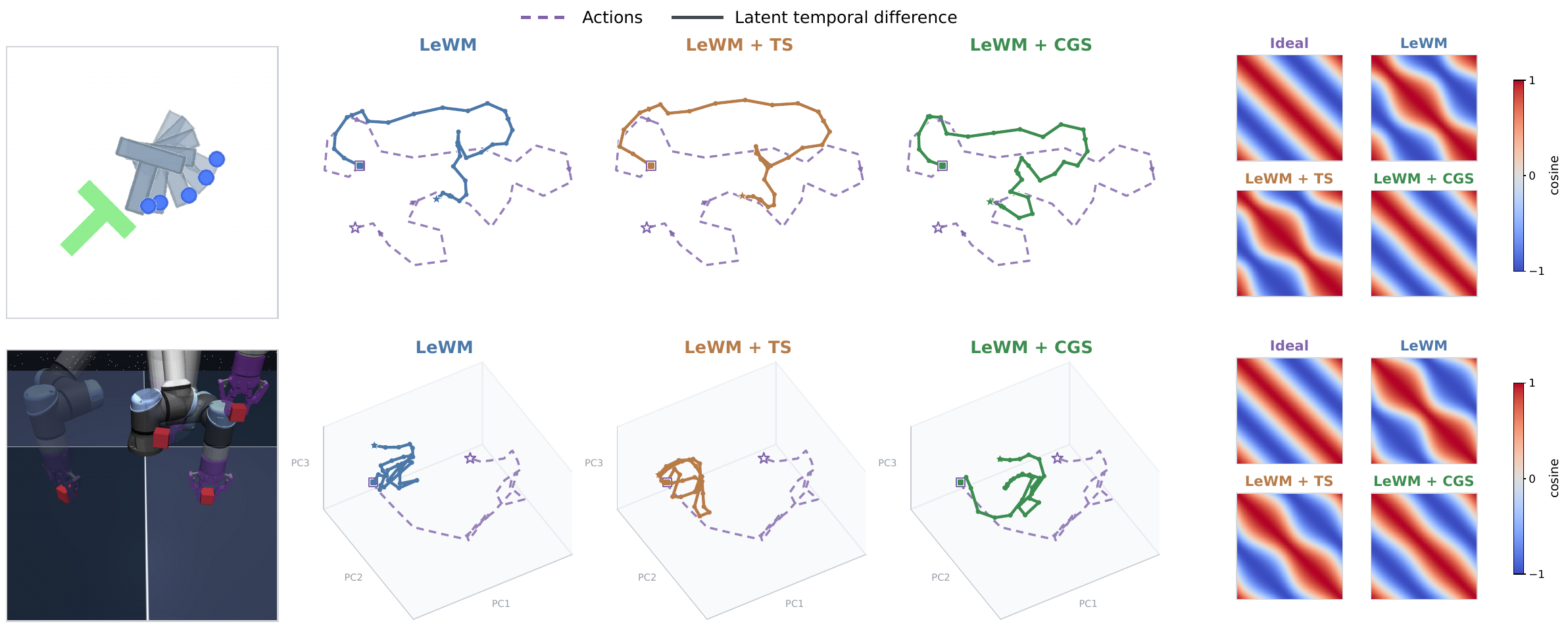}
    \caption{\textbf{CGS preserves action-related turns and control-response
    geometry.}
    PushT (top) and Cube (bottom), with LeWM (blue), LeWM+TS (brown),
    and LeWM+CGS (green).
    \textbf{Middle: transition shaping.} CGS
    makes latent transitions follow action-related turns more closely than LeWM and
    TS.
    \textbf{Right: control-response geometry.} The heatmaps show pairwise
    cosines of predicted endpoint changes induced by action-sequence
    perturbations; \emph{Ideal} gives the angular relations among the
    perturbations themselves. CGS more closely preserves these relations in
    the tested control plane.}
    \label{fig:pusht-cgs-intro}
\end{figure*}

We introduce \textbf{Control-Geometry Straightening (CGS)}, which learns
this geometry from local transitions by matching pairwise cosine
similarities among actions to those among their corresponding
latent differences. This single auxiliary loss applies directly to latent transitions and
transfers across both end-to-end and pretrained world-model representations. Our linear-dynamics theoretical analysis shows how CGS
\emph{straightens} this geometry, promoting more balanced terminal-cost
curvature across the full planning horizon for more effective action search.
Experiments show substantial success-rate gains over LeWM and LeWM+TS
across multiple planners, together with improved performance using fewer sampled candidates and refinement steps.

Our contributions are threefold:
\begin{itemize}
    \setlength{\itemsep}{1pt}
    \setlength{\parsep}{0pt}
    \setlength{\parskip}{0pt}

    \item \textbf{Planner-friendly representation learning.}
    We introduce CGS, a single auxiliary loss that matches action and
    latent-difference cosine similarities to learn planner-friendly
    representations only using local transitions from pixel--action pairs, directly straightening
    control geometry for more effective multi-step planning under
    limited planning budgets.

    \item \textbf{Geometry-to-planning theory.}
    Under linear-dynamics, we develop a theoretical analysis linking
    world-model representation geometry to the optimization behavior of
    planners through finite-horizon planning curvature. This yields finite-budget
    guarantees for MPPI, local contraction results for CEM, and convergence bounds
    for gradient descent.

    \item \textbf{Empirical validation and characterization.}
    Across multiple planners, CGS improves planning performance and sample
    efficiency, reaching success-rate gains up to $20$ and $12.6$ percentage
    points over LeWM and LeWM+TS, respectively. Probes, cross-architecture
    comparisons with DINO-WM, and targeted planner-side ablations further
    clarify how latent motion organization, state dependence, and dynamical
    context shape planning behavior.
\end{itemize}

\section{Related Work}
\label{sec:related-work}

\noindent\textbf{Latent world models and JEPA.}
Latent world models learn compact dynamics for planning
\citep{ha2018worldmodels,hafner2019planet,hansen2024tdmpc2}.
JEPAs predict embeddings without reconstructing pixels
\citep{lecun2022path,assran2023ijepa,bardes2024vjepa}.
DINO-WM and V-JEPA~2-AC build on pretrained visual features, while PLDM
and LeWM learn representations and dynamics jointly
\citep{zhou2025dinowm,assran2025vjepa2,sobal2025pldm,maes2026lewm}.
Much of this development addresses representation collapse through stable
targets, variance regularization, or distribution matching
\citep{assran2023ijepa,bardes2022vicreg,balestriero2025lejepa}.
These mechanisms stabilize predictive learning, while leaving underdetermined
how representation geometry should support planning by making good action
sequences easier to find~\citep{li2026predictive}.

\noindent\textbf{Planner-friendly representations.}
Earlier work learns locally linear latent dynamics for trajectory optimization
or representations through differentiable planners
\citep{watter2015embed,zhang2019solar,srinivas2018upn}.
Recent JEPA methods add goal-conditioned values, budget-conditioned
reachability, or future-goal inverse control
\citep{destrade2026value,li2026predictive,sun2026intact}.
Inverse-dynamics objectives retain action information in latent
transitions~\citep{ivashkov2026sensorimotor,zhang2026delta}.
Temporal Straightening aligns consecutive latent displacements to reduce
trajectory curvature \citep{wang2026ts}, without structuring
how actions shape latent transitions.
Relational and similarity-preserving objectives in vision
knowledge distillation transfer structure between teacher and student
representations~\citep{park2019relational,tung2019similarity}.

\noindent\textbf{Sampling-based MPC.}
CEM fits proposals to low-cost elites; MPPI weights candidates
by rollout cost~\citep{deboer2005cem,williams2017mppi}.
iCEM reuses samples, CoVO-MPC adapts covariance, and
Biased-MPPI and PRISM use controller or learned-prior proposals
\citep{pinneri2021icem,yi2024covo,trevisan2024biased,wang2026prism}.
Theory studies MPPI updates and finite-sample stability
\citep{fazlyab2026mppi,yoon2026finite}.
Sampling efficiency also depends on the horizon, effective action dimension,
cost landscape, and proposal design~\citep{yoon2022sampling}.
See Appendix~\ref{app:related-work} for more discussions about extended related work.

\section{Control-Geometry Straightening}
\label{sec:cgs}

For our main analysis, we formulate CGS using the end-to-end LeWM architecture
\citep{maes2026lewm} and offline pixel--action trajectories $(o_t,a_t)$.
CGS learns control geometry by matching pairwise
cosine similarities among actions to those among their corresponding latent
differences. As developed below, this local geometric constraint straightens
the finite-horizon planning geometry. Figure~\ref{fig:cgs_diag} shows the training and planning pipeline;
the SIGReg term is omitted from the diagram.

\begin{figure}[t]
    \centering
    \includegraphics[width=\linewidth]{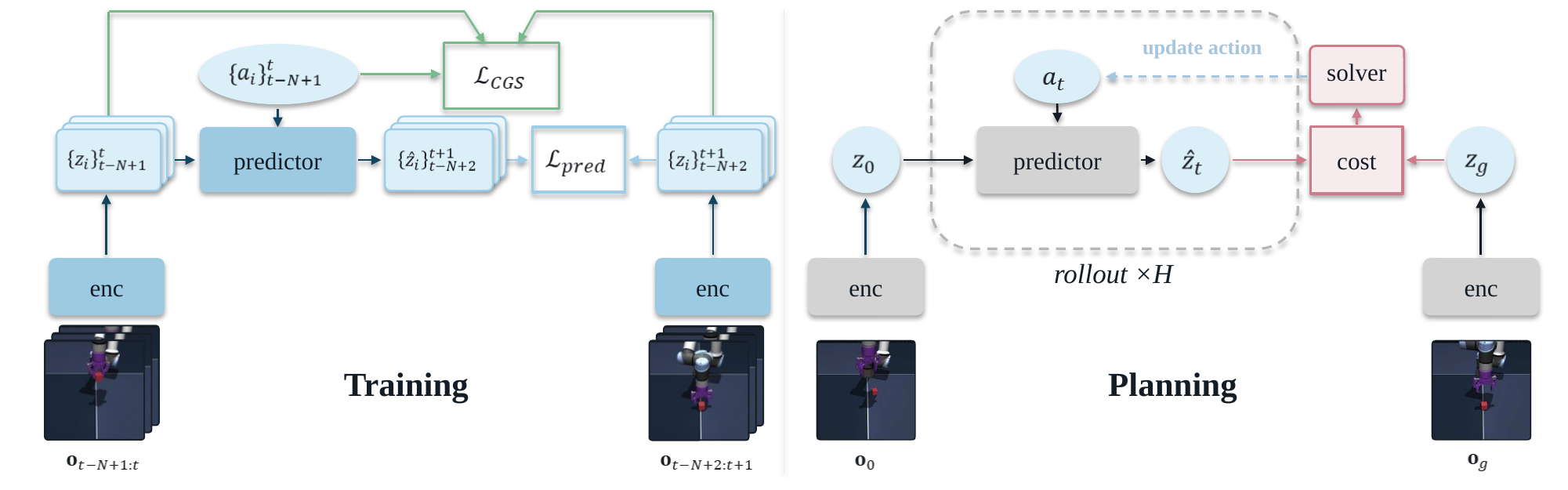}
    \caption{\textbf{CGS training and planning.}
    During training, we minimize the prediction loss between predicted and
    encoded target embeddings, together with a CGS loss that matches pairwise
    cosine similarities of standardized actions and consecutive latent
    differences; SIGReg is omitted from the diagram. During planning, we roll
    out candidate action sequences for $H$ steps using the frozen predictor
    and refine them with a sampling-based planner to minimize the distance
    between the predicted terminal embedding and the goal embedding.}
    \label{fig:cgs_diag}
\end{figure}

\subsection{Latent world model and planning setup}
\label{sec:cgs-setup}

Following LeWM \citep{maes2026lewm}, our latent world model has two components:
\begin{equation}
\begin{aligned}
    \text{Encoder:\quad}& z_t = f_\theta(o_t), \\
    \text{Predictor:\quad}& \widehat z_{t-N+2:t+1}
    =
    g_\phi\!\left(
        z_{t-N+1:t},
        a_{t-N+1:t}
    \right).
\end{aligned}
\label{eq:cgs-world-model}
\end{equation}
The encoder $f_\theta$ is a tiny ViT \citep{dosovitskiy2021vit}. It constructs embedding $z_t \in \mathbb{R}^d$ with observation $o_t$ from the output [CLS] token through an MLP projection step. The transformer predictor $g_{\phi}$, followed by an MLP projector, takes a history of $N$ frame representations and corresponding actions $a_t\in\mathbb{R}^{d_a}$ as input. Equation~\eqref{eq:cgs-world-model} expresses prediction in sequence-to-sequence form, with $N$ input frames and $N$ next-step
predictions. Specifically, temporal causal masking is used in the predictor to restrict each next-step prediction to inputs at or before the current time, as shown in Eq.~\eqref{eq:causal-masking}:
\begin{equation}
\widehat z_{t-N+2} = g_\phi\!\left(
        z_{t-N+1},
        a_{t-N+1}\right), \ \dots \ ,\ \widehat z_{t+1} =  g_\phi\!\left(
        z_{t-N+1:t},
        a_{t-N+1:t}
    \right).
\label{eq:causal-masking}
\end{equation}

At planning time, we optimize an action sequence $\ba=(a_0,\ldots,a_{H-1})\in\R^{H\times d_a}$ with frozen world model as the dynamical model. We encode the initial and goal observation pixels to $z_0$ and $z_g$, and search for an action sequence whose predicted latent rollout minimizes the goal-matching cost
\begin{equation}
C(\ba) = \tfrac12 \|\widehat F_H(\ba)-z_g\|_2^2,
\label{eq:cgs-planning-cost}
\end{equation}
where the terminal latent prediction over a horizon $H$ is denoted as $\widehat F_H(\ba)$.

\subsection{Learning Control Geometry by Pairwise Similarity Matching}
\label{sec:cgs-objective}

We seek to straighten control geometry in the latent space by matching the pairwise cosine similarities of latent differences to those of actions. For each transition $t$, let
$\Delta z_t=f_\theta(o_{t+1})-f_\theta(o_t)$.
Consider a history window with $N\geq 2$ transitions, and define
\(u_t=\frac{a_t}{\|a_t\|_2}\),\(v_t=\frac{\Delta z_t}{\|\Delta z_t\|_2}.\)
Stack them as the rows of $U\in\R^{N\times d_a}$ and
$V\in\R^{N\times d}$. Their \emph{cosine Gram matrices} are
\begin{equation}
    K_A=UU^\top,
    \qquad
    K_Z=VV^\top,
\label{eq:cgs-relational-grams}
\end{equation}
Thus, $K_A$ and $K_Z$ contain pairwise action and latent-difference cosines, respectively. CGS minimizes
\begin{equation}
    \cL_{\rm CGS}
    =
    \frac{\left\|K_Z-K_A\right\|_F^2}{N(N-1)}
    = \frac{1}{N(N-1)}\sum_{\substack{t\ne s}}
    \left[\cos(\Delta z_t,\Delta z_s)-\cos(a_t,a_s)\right]^2,
\label{eq:cgs-objective}
\end{equation}
where $\|\cdot\|_F$ denotes the Frobenius norm. Both matrices have unit
diagonal, so the loss averages squared discrepancies over $N(N-1)$
off-diagonal entries. CGS uses action directions as a geometric reference for latent transitions:
aligned, orthogonal, and opposing actions encourage corresponding angular
relationships between latent differences. The pairwise matching form resembles similarity-preserving knowledge
distillation~\citep{tung2019similarity}. Here, actions instead provide
the reference geometry for latent transitions, and the matched relations
encode control geometry for planning. As developed in Section~\ref{sec:cgs_theory}, this local
geometric constraint promotes more balanced finite-horizon planning curvature
and more effective action search under limited sampling budgets.

\textbf{Relation to temporal straightening.}
When consecutive actions point in the same direction, $u_t^\top u_{t+1}=1$,
TS and CGS both encourage consecutive latent differences to align.
The adjacent CGS term is then $(1-v_t^\top v_{t+1})^2$, the squared angular
TS loss in \cite{wang2026ts}.

\subsection{Training objective}
\label{sec:cgs-implementation}

We jointly train the encoder and predictor with CGS and the LeWM objectives.
Using the $N$-frame history window, the prediction loss is
\begin{equation}
    \cL_{\rm pred}
    =
    (Nd)^{-1}\,
        \left\|
            \widehat z_{t-N+2:t+1}
            -
            z_{t-N+2:t+1}
        \right\|_F^2.
\label{eq:cgs-prediction-loss}
\end{equation}
The factor $(Nd)^{-1}$ averages squared errors over the $N$ next-step
predictions and $d$ latent coordinates. The base objective of LeWM
\citep{maes2026lewm} is
\begin{equation}
    \cL_{\rm LeWM}
    =
    \cL_{\rm pred}
    +
    \lambda_{\rm SIGReg}\cL_{\rm SIGReg},
\label{eq:cgs-lewm-objective}
\end{equation}
where $\cL_{\rm SIGReg}$ is LeWM's SIGReg anti-collapse term, which promotes
feature diversity by encouraging an isotropic Gaussian distribution of latent embeddings \citep{maes2026lewm}. The full training objective is
\begin{equation}
    \cL
    =
    \cL_{\rm LeWM}
    +
    \lambda_{\rm CGS}
    \cL_{\rm CGS},
\label{eq:cgs-lewm-training}
\end{equation}
where $\lambda_{\rm CGS} \geq 0$ controls the strength of the control geometry straightening. To help preserve the predictive geometry while gradually introducing control-geometry regularization, we linearly ramp the CGS weight from zero to the final value $\lambda_{\rm CGS}$ over the first $5\%$ of training updates. We use the default history length $N=3$ throughout training, following LeWM~\citep{maes2026lewm}. Additional ablations and discussion of the ramp schedule and history length appear in Appendix~\ref{app:training-ablations}.

% ===== END CGS_new.tex =====
% ===== BEGIN theory_new.tex =====
\section{Theory: From CGS to Sampling-Efficient Planning}
\label{sec:cgs_theory}

In this section, we analyze how CGS supports sampling-efficient planning under linear latent dynamics. We connect CGS to temporal straightening and balanced action effects, then derive consequences for finite-budget planning over a finite horizon.
Full proofs, theoretical results of CEM and gradient-based planners, and finite-budget
detailed discussions are provided in Appendix~\ref{app:cgs_full_theory}.

\paragraph{Setup.}
For analysis, we study linear latent dynamics,
\begin{equation}
  z_{t+1}=Az_t+Ba_t,
  \label{eq:linear-dynamics} 
\end{equation}
with
\(A\in\mathbb R^{d\times d}\) and
\(B\in\mathbb R^{d\times d_a}\).
We first consider \(d_a=d\) with \(B\) invertible; see
Remark~\ref{rem:cgs_low_dim_actions} for \(d_a<d\). For goal-reaching planning, we optimize the action sequence
$\ba=(a_0,\ldots,a_{H-1})\in\R^{H\times d_a}$, \(n_H = H \times d_a\), to minimize the terminal latent goal-matching objective 
\[
  C(\mathbf a)=\tfrac12\|z_H-z_g\|_2^2,\quad z_H = \widehat F_H(\ba),
\]
with latent goal $z_g$ and terminal prediction $\widehat F_H(\ba)$ as in Eq.~\eqref{eq:cgs-planning-cost}, where planning horizon $H \geq 2$.

\subsection{Latent Control Geometry Learned by CGS}
\label{sec:cgs_main_geometry}

Letting \(D:=A-I_d\), with \(I_d\) denoting the \(d\times d\)-dimensional identity matrix, the latent difference is \(\Delta z_t=Dz_t+Ba_t\).
Under the fixed-radius assumption in Eq.~\eqref{eq:app_cgs_fixed_radius},
the population CGS loss is a positive multiple of
\(\mathcal L_G(c):=\mathbb E[(\Delta z^\top\Delta z'-c\,a^\top a')^2]\),
where \((\Delta z,a)\) and \((\Delta z',a')\) are independent
transition samples and \(c>0\) is a scale factor set by fixed radii of actions and latent differences.

\begin{assumption}[Bounded features and joint coverage]
\label{ass:cgs_joint_coverage}
Following \citet{yin2022nearoptimal}, we assume there exists a feature map \(\psi\) satisfying
(i) boundedness, \(\|\psi(z_t,a_t)\|_2^2\leq \bar \rho<\infty\) a.s., and
(ii) coverage, \(\lambda_{\min}(\Sigma_x)\geq\rho>0\), where
\(\Sigma_x:=\mathbb E[\psi(z_t,a_t)\psi(z_t,a_t)^\top]\).
Here we use \(\psi(z_t,a_t)=x_t=[z_t^\top,a_t^\top]^\top\).
\end{assumption}

\begin{proposition}[Temporal straightening and control isotropy under CGS]
\label{prop:cgs_joint_straightening}
Under Assumption~\ref{ass:cgs_joint_coverage}, \(\rho I_{2d}\preceq\Sigma_x\preceq\bar\rho I_{2d}\). With
\(\mathcal R(A,B):=\|D^\top D\|_F^2+2\|D^\top B\|_F^2+\|B^\top B-cI_{d_a}\|_F^2\),
we have
\begin{equation}
    \rho^2\mathcal R(A,B)\leq\mathcal L_G(c)
    \leq\bar\rho^2\mathcal R(A,B).
    \label{eq:cgs_main_joint_equivalence}
\end{equation}
In this model class, \(\mathcal L_G(c)=0\) iff \(A=I_d\) and \(B^\top B=cI_{d_a}\); Proof is in Appendix~\ref{app:cgs_joint_straightening_proof}.
\end{proposition}

Under Assumption~\ref{ass:cgs_joint_coverage}, CGS encourages temporal straightening (\(A\approx I_d\)), reduces drift--control coupling (\(D^\top B\approx0\)) and promotes balanced action effects (\(B^\top B\approx cI_{d_a}\)).
During joint training,
prediction and other losses may favor different representation geometry, so the
CGS loss may remain nonzero. CGS can still improve control geometry under these constraints; see Appendix~\ref{app:cgs_prediction_constraint}.

\subsection{From Finite-Horizon Planning Geometry to Finite-Budget MPPI}
\label{sec:cgs_main_planner}

Unrolling \eqref{eq:linear-dynamics} gives \(z_H=A^Hz_0+\Gamma_H\mathbf a\), where
\(\Gamma_H=[A^{H-1}B,\ldots,B]\),
and the planning Hessian is
\(G_H:=\nabla_{\mathbf a}^2C=\Gamma_H^\top\Gamma_H\).
Since \(B\) is invertible, \(G_H\) has rank \(d\), and its
positive eigenvalues coincide with those of the controllability
Gramian \(W_H:=\Gamma_H\Gamma_H^\top\).
We measure conditioning by the ratio of the largest to smallest positive eigenvalues. Let \(P_H:=H^{-1}(\mathbf1_H\mathbf1_H^\top)\otimes I_{d_a}\), where \(\mathbf1_H\in\R^H\) is the all-ones vector.
This rank-\(d_a\) orthogonal projector maps onto the subspace of
action-sequence perturbations constant across time.
When \(A=I_d\) and \(B^\top B=cI_{d_a}\), \(z_H=z_0+B\sum_ta_t\) and \(G_H=cHP_H\),
yielding equal curvature along all directions in this active subspace, termed \textit{planning isotropy}.
With isotropic Gaussian proposals, population MPPI updates then contract
all active directions equally, balancing refinement across action directions.
The rightmost heatmaps of Figure~\ref{fig:pusht-cgs-intro} illustrate CGS's strong performance in \textit{planning isotropy} (Appendix~\ref{app:geometry-heatmaps}).

\begin{corollary}[From CGS loss to drift and control bounds]
\label{cor:cgs_loss_to_epsilon}
Under Assumption~\ref{ass:cgs_joint_coverage},
Proposition~\ref{prop:cgs_joint_straightening} implies the following.
If \(\mathcal L_G(c)\leq\ell\), define
\(\varepsilon=(\sqrt\ell/\rho)\max\{1,(2c)^{-1/2},c^{-1}\}\).
Then \(\|D\|_{\op}^2\leq\varepsilon\),
\(\|D^\top B\|_{\op}\leq\sqrt c\,\varepsilon\), and
\(\|B^\top B-cI_{d_a}\|_{\op}\leq c\varepsilon\).
Here \(\|\cdot\|_{\op}\) is the spectral norm. These quantify worst-direction
geometric deviations. We call these the
\(\varepsilon\)-CGS bounds. Proof is in Appendix~\ref{app:cgs_loss_to_geometry}.
\end{corollary}

\begin{theorem}[CGS controls the planning Hessian]
\label{thm:cgs_horizon_curvature}
Suppose the three bounds in Corollary~\ref{cor:cgs_loss_to_epsilon}
hold for \(0\leq\varepsilon<1\), there exists $\xi_H$ (defined in Eq.~\eqref{eq:cgs_horizon_distortion_app}) such that
\begin{equation}
    \left\|\frac{G_H}{cH}-P_H\right\|_{\op}
    \leq\xi_H(\varepsilon),\quad
    \xi_H(\varepsilon)
    =(2H^2-3H+2)\varepsilon+O_H(\varepsilon^{3/2}),
    \label{eq:cgs_main_horizon_hessian_bound}
\end{equation}
where \(O_H\) hides constants depending
only on the fixed horizon \(H\).
If \(\xi_H:=\xi_H(\varepsilon)<1\), then, writing \(r=d_a=d\) and
\(\sigma_1(G_H)\geq\cdots\geq\sigma_{n_H}(G_H)\), the first \(r\)
singular values (equivalently, eigenvalues) lie in
\([cH(1-\xi_H),cH(1+\xi_H)]\), all remaining singular values are zero, and
\(\kappa_{\rm eff}(G_H):=\sigma_1(G_H)/\sigma_r(G_H)
\leq(1+\xi_H)/(1-\xi_H)\). Proof is in Appendix~\ref{app:cgs_horizon_curvature_proof}.
\end{theorem}

For fixed \(H\geq2\) and the same control-isotropy tolerance, CGS directly
controls \(\|D^\top B\|_{\op}=O(\varepsilon)\) via its explicit mixed drift--control defect bound, while temporal straightening
gives the indirect \(O(\sqrt{\varepsilon})\) bound. This sharpens the Hessian
perturbation from \(O_H(\sqrt{\varepsilon})\) to \(O_H(\varepsilon)\); see
Appendix~\ref{app:cgs_horizon_proof}.

\begin{remark}[Low-dimensional actions]
\label{rem:cgs_low_dim_actions}
When \(d_a<d\), the bound on \(\|G_H/(cH)-P_H\|_{\op}\) in
Theorem~\ref{thm:cgs_horizon_curvature} remains valid on the full
action-sequence space \(\mathbb R^{n_H}\). For \(\xi_H<1\),
its condition-number bound holds on the subspace spanned by the leading
\(d_a\) eigenvectors of \(G_H\), and the subsequent planner analyses
are carried out on these directions. See
Appendix~\ref{app:cgs_low_dim_actions}.
\end{remark}

\begin{theorem}[MPPI population mean update]
\label{thm:cgs_mppi_population}
Assume \(\xi_H<1\). Let \(\tau>0\) be the MPPI temperature and
\(\sigma_{\mathrm p}>0\) the sampling standard deviation.
Let \(\mathbf a^\star\) minimize \(C\), condition on
the proposal \(\mathcal N(\mu,\sigma_{\mathrm p}^2I_{n_H})\), and set
\(\nu=\sigma_{\mathrm p}^{-1}(\mu-\mathbf a^\star)\),
\(M_H=G_H/(cH)\), and \(s=\sigma_{\mathrm p}^2cH/\tau\). Let
\(\mathcal S_H:=\operatorname{range}(G_H)\), which has dimension \(r=d\), and contraction
\(q_{\rm M}:=\|J_{\rm M}|_{\mathcal S_H}\|_{\op}\), defined as the spectral norm of the Jacobian restricted to the active subspace \(\mathcal S_H\). The population map,
Jacobian, and active contraction are
\begin{equation}
\begin{aligned}
    T_{\rm M}(\nu)&=J_{\rm M}\nu,\qquad
    J_{\rm M}=(I_{n_H}+sM_H)^{-1},\\[-2pt]
    q_{\rm M}&=\frac{1}{1+(\sigma_{\mathrm p}^2/\tau)\sigma_r(G_H)}
    \leq\frac{1}{1+s(1-\xi_H)}.
\end{aligned}
    \label{eq:cgs_main_mppi_update}
\end{equation}
Moreover,
\(\kappa(J_{\rm M}|_{\mathcal S_H})\leq[1+s(1+\xi_H)]/[1+s(1-\xi_H)]\); at
\(M_H=P_H\), all active directions contract uniformly by \((1+s)^{-1}\). Proof is in Appendix~\ref{app:cgs_mppi_population_proof}.
\end{theorem}

\paragraph{Finite-budget planning cost.}
Let $\mu_i$ be the mean after $i$ MPPI
updates and $\Delta_i:=C(\mu_i)-C(\mathbf a^\star)$ its planning cost
gap.
Fix $\sigma_{\mathrm p}$ and $\tau$, use $K$ independent samples
per update, and impose the minimum-$K$ condition in
Eq.~\eqref{eq:app_mppi_minimum_K}.
With probability at least $1-\zeta$, the cost gap after
$I$ updates satisfies
$\sqrt{\Delta_I}\leq q_{\rm M}^{I}\sqrt{\Delta_0}
+(1-q_{\rm M}^{I})b_K/(1-q_{\rm M})$
whenever $\|\mu_i-\mathbf a^\star\|_2\leq\sigma_{\mathrm p}R_\nu$
for all $0\leq i<I$, where $R_\nu$ is defined in Eq.~\eqref{eq:app_planner_first_exit}.
Here $b_K:=\sigma_{\mathrm p}\mathcal A_{\rm M}(R_\nu)
\sqrt{cH(1+\xi_H)r\log(4rI/\zeta)/(2K)}$
bounds the per-update sampling contribution to the square-root cost gap,
and $\mathcal A_{\rm M}(R_\nu)$ is a dimension-dependent importance-weight factor.
With $q_{\rm M}<1$, more updates suppress the initial error, while more samples reduce the sampling term.
Appendix~\ref{app:cgs_mppi_cost_bounds} gives the factor's definition,
the proof, and sufficient budgets for a prescribed target error.

\begin{wrapfigure}{R}{0.50\textwidth}
  \vspace{-20pt}
  \centering
  \includegraphics[width=\linewidth]{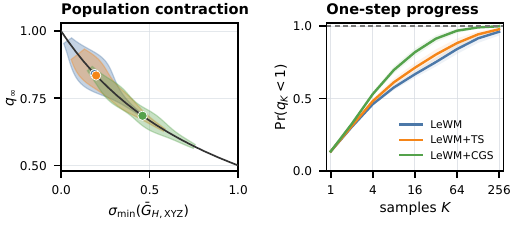}
  \caption{\textbf{MPPI contraction and finite-sample improvement.}
(a) CGS lowers population contraction $q_\infty$;
$\sigma_{\min}$ measures the weakest curvature.
(b) CGS increases $\Pr(q_K<1)$, where $q_K$ is the
$K$-sample contraction factor.
Definitions and protocol: Appendix~\ref{app:cgs_figure3_audit}.}
  \label{fig:cgs_mppi_mechanism}
\end{wrapfigure}

Theorem~\ref{thm:cgs_mppi_population} shows how CGS balances population MPPI updates
by promoting comparable curvature across active action directions.
With the cost scale calibrated, this spectral balance can support faster population contraction and a higher probability of progress under a limited sampling budget.
In trace-normalized quadratic probes of Cube's XYZ translation subspace,
CGS yields a lower population contraction factor $q_\infty$ and a higher
one-step progress probability $\Pr(q_K<1)$ at the same sample budget per update across models
(Figure~\ref{fig:cgs_mppi_mechanism}).

\begin{remark}[CEM and gradient-based planner]
\label{rem:cgs_other_planners}
Appendix~\ref{app:cgs_cem_population} shows that the same projector geometry of the planning Hessian
makes the local CEM mean update approximately scalar on its active subspace,
reducing directional imbalance in elite selection. For gradient descent planner, Appendix~\ref{app:cgs_ts_style} also gives
planning-cost convergence bounds on the leading controllable
directions across iterative updates, linking CGS Hessian geometry to first-order
action optimization.
\end{remark}
% ===== END theory_new.tex =====
% ===== BEGIN main_results_new.tex =====
\section{Main Results: Planning Performance and Sample Efficiency}
\label{sec: res}
\label{sec:main-results}
% Main table: ../main_summary.md, K=128; rounded to one decimal for display.
% Full sweep: ../full_sweep.md; figure from ../analysis_oi/figures/full_sweep/.

\subsection{Experimental Setup}
\label{sec:main-experimental-design}

\paragraph{Environments and models.}
We evaluate PushT~\citep{chi2023diffusion}, Cube~\citep{park2025ogbench},
TwoRooms~\citep{sobal2025pldm}, and Reacher~\citep{tunyasuvunakool2020dmcontrol}
using offline pixel--action data.
All LeWM-based models follow the training setup of
\citet{maes2026lewm}. We keep the base LeWM hyperparameters fixed across
methods, including a shared SIGReg weight of $0.09$, and vary only the
coefficient introduced by each additional regularizer.
For LeWM+TS~\citep{wang2026ts}, we sweep
$\lambda_{\mathrm{TS}}\in\{0.01,0.05,0.1,0.5\}$, yielding selected
weights of $0.01$ on PushT, $0.1$ on Cube and Reacher, and $0.5$ on
TwoRooms.
CGS uses $\lambda_{\mathrm{CGS}}=0.5$ by default and $0.01$ on PushT.
Each checkpoint remains fixed across planners and planning budgets.
See Appendix~\ref{app:lewm-training-details} for details.

\paragraph{Evaluation and planners.}
We report goal-reaching success rate (SR, \%) as mean $\pm$ sample
standard deviation (SD). MPPI and CEM use $I=30$ optimization steps (iterations),
with $K=1$--$256$ candidates and $K=128$ in
Table~\ref{tab:lewm-main-k128}. Gradient descent (GD) uses 100 Adam
optimization steps, independent of $K$. Prior-guided (PG) variants,
PG-MPPI and PG-CEM, initialize the proposal mean with a representation-matched
behavior-cloning prior. Evaluation details appear in
Appendix~\ref{app:experimental-details}.

\begin{table*}[!t]
  \centering
  \caption{\textbf{Goal-reaching success rate (SR) at $K=128$.}
    Values are reported as means $\pm$ sample standard deviation (SD), in \%.
    Bold marks the best per planner and environment, including ties.
    Parentheses report gains over LeWM in percentage points; the top two per
    environment are underlined. Prior-guided (PG) variants
    ($\dagger$) use action-prior initialization. Sampling planners
    use $I=30$ optimization steps; gradient descent (GD) uses 100 Adam optimization steps,
    independent of $K$.}
  \label{tab:lewm-main-k128}
  \vspace{-5pt}
  \begingroup
    \fontsize{8}{9}\selectfont
    \setlength{\tabcolsep}{3pt}
    \renewcommand{\arraystretch}{0.90}
    \setlength{\aboverulesep}{0.8pt}
    \setlength{\belowrulesep}{0.8pt}
    \setlength{\heavyrulewidth}{0.6pt}
    \setlength{\lightrulewidth}{0.35pt}
    \providecommand{\mfsr}[2]{\ensuremath{#1{\scriptscriptstyle\,\pm #2}}}
    \providecommand{\mfsrdelta}[3]{\ensuremath{#1{\scriptscriptstyle\,\pm #2}\,{\scriptscriptstyle(#3)}}}
    \begin{tabularx}{\textwidth}{@{}ll*{4}{>{\centering\arraybackslash}X}@{}}
    \toprule
    World model & Planner & PushT & Cube & TwoRooms & Reacher \\
    \midrule
    LeWM & MPPI & \mfsr{64.0}{5.3} & \mfsr{59.3}{4.6} & \mfsr{92.0}{3.5} & \mfsr{65.3}{4.2} \\
     & PG-MPPI$^{\dagger}$ & \mfsr{72.0}{6.0} & \mfsr{89.3}{4.2} & \mfsr{98.0}{2.0} & \mfsr{62.7}{11.0} \\
     & CEM & \mfsr{89.3}{6.1} & \mfsr{63.3}{2.3} & \mfsr{76.7}{10.1} & \mfsr{\mathbf{83.3}}{8.3} \\
     & PG-CEM$^{\dagger}$ & \mfsr{92.0}{5.3} & \mfsr{90.0}{2.0} & \mfsr{93.3}{3.1} & \mfsr{\mathbf{86.7}}{6.1} \\
     & GD & \mfsr{85.3}{5.0} & \mfsr{59.3}{8.1} & \mfsr{41.3}{7.0} & \mfsr{76.7}{4.2} \\
    \midrule
    \rowcolor{black!5} LeWM + TS & MPPI & \mfsrdelta{64.7}{8.3}{+0.7} & \mfsrdelta{75.3}{5.0}{\underline{+16.0}} & \mfsrdelta{\mathbf{97.3}}{3.1}{+5.3} & \mfsrdelta{64.0}{2.0}{-1.3} \\
    \rowcolor{black!5}  & PG-MPPI$^{\dagger}$ & \mfsrdelta{69.3}{9.0}{-2.7} & \mfsrdelta{92.7}{3.1}{+3.3} & \mfsrdelta{98.0}{2.0}{+0.0} & \mfsrdelta{66.7}{2.3}{+4.0} \\
    \rowcolor{black!5}  & CEM & \mfsrdelta{86.7}{5.0}{-2.7} & \mfsrdelta{73.3}{10.3}{+10.0} & \mfsrdelta{\mathbf{97.3}}{1.2}{+20.7} & \mfsrdelta{81.3}{8.1}{-2.0} \\
    \rowcolor{black!5}  & PG-CEM$^{\dagger}$ & \mfsrdelta{90.0}{5.3}{-2.0} & \mfsrdelta{90.0}{2.0}{+0.0} & \mfsrdelta{\mathbf{100.0}}{0.0}{+6.7} & \mfsrdelta{78.0}{2.0}{-8.7} \\
    \rowcolor{black!5}  & GD & \mfsrdelta{82.7}{5.8}{-2.7} & \mfsrdelta{65.3}{4.6}{+6.0} & \mfsrdelta{65.3}{9.9}{\underline{+24.0}} & \mfsrdelta{81.3}{5.0}{+4.7} \\
    \midrule
    \rowcolor{black!5} LeWM + CGS & MPPI & \mfsrdelta{\mathbf{77.3}}{6.4}{\underline{+13.3}} & \mfsrdelta{\mathbf{79.3}}{4.6}{\underline{+20.0}} & \mfsrdelta{\mathbf{97.3}}{1.2}{+5.3} & \mfsrdelta{\mathbf{66.7}}{2.3}{+1.3} \\
    \rowcolor{black!5}  & PG-MPPI$^{\dagger}$ & \mfsrdelta{\mathbf{80.0}}{2.0}{\underline{+8.0}} & \mfsrdelta{\mathbf{96.7}}{3.1}{+7.3} & \mfsrdelta{\mathbf{100.0}}{0.0}{+2.0} & \mfsrdelta{\mathbf{69.3}}{3.1}{\underline{+6.7}} \\
    \rowcolor{black!5}  & CEM & \mfsrdelta{\mathbf{92.0}}{3.5}{+2.7} & \mfsrdelta{\mathbf{76.7}}{9.5}{+13.3} & \mfsrdelta{95.3}{1.2}{+18.7} & \mfsrdelta{80.7}{3.1}{-2.7} \\
    \rowcolor{black!5}  & PG-CEM$^{\dagger}$ & \mfsrdelta{\mathbf{96.0}}{0.0}{+4.0} & \mfsrdelta{\mathbf{96.0}}{4.0}{+6.0} & \mfsrdelta{\mathbf{100.0}}{0.0}{+6.7} & \mfsrdelta{81.3}{3.1}{-5.3} \\
    \rowcolor{black!5}  & GD & \mfsrdelta{\mathbf{89.3}}{6.4}{+4.0} & \mfsrdelta{\mathbf{71.3}}{4.6}{+12.0} & \mfsrdelta{\mathbf{70.7}}{3.1}{\underline{+29.3}} & \mfsrdelta{\mathbf{82.0}}{5.3}{\underline{+5.3}} \\
    \bottomrule
    \end{tabularx}
  \endgroup
  \par\vspace{2pt}
  \begin{minipage}{\textwidth}
    \centering
    \includegraphics[width=\linewidth]{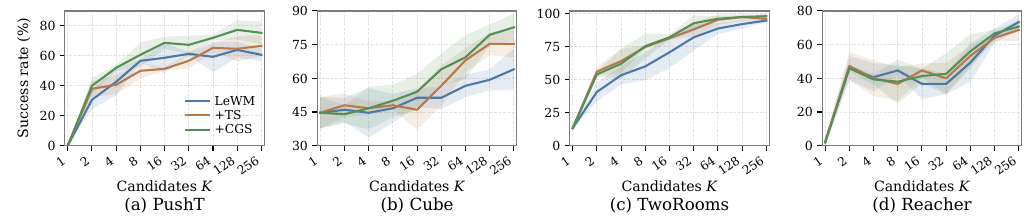}
    \makeatletter
    \def\@captype{figure}
    \makeatother
    \caption{\textbf{MPPI success rate versus sampling budget.}
      Means and $\pm$ one sample SD, with $I=30$ optimization steps and no
      action prior. CGS leads on PushT and at larger budgets on Cube;
      TwoRooms approaches saturation. Reacher shows smaller, budget-dependent differences. Environment-specific temperatures are fixed across budgets;
      See Appendix~\ref{app:lewm-full-sweep} for results of other planners.}
    \label{fig:lewm-full-sweep-mppi}
    \vspace*{-15pt}
  \end{minipage}
\end{table*}

\subsection{Toward Efficient Planning with Action-Structured Latent Dynamics}
\label{sec:main-results-analysis}

Under a shared evaluation protocol, CGS consistently improves over LeWM for all
five planners on PushT, Cube, and TwoRooms at $K=128$
(Table~\ref{tab:lewm-main-k128}), across distinct planning mechanisms. On
PushT and Cube, its MPPI gains are $13.3$ and $20.0$ percentage points over
LeWM, respectively. Combining CGS with a learned action prior further improves
planning, achieving near-perfect success in several settings. TS also provides
a strong baseline, with substantial gains on Cube and TwoRooms and improvements
in several other settings. Figure~\ref{fig:lewm-ts-cosine-effect} shows that
explicit TS further straightens latent trajectories across all four environments,
beyond the implicit straightening reported by
\citet[Appendix H]{maes2026lewm}. On PushT, where this implicit effect is
already strong, TS yields only a subtle additional straightening and
correspondingly modest planning changes.

The budget sweep in Figure~\ref{fig:lewm-full-sweep-mppi} demonstrates
\emph{sampling-efficient planning} under a limited planning budget.
On PushT, CGS improves MPPI even with few samples and retains its advantage
over both baselines as $K$ increases. On Cube, the gap widens with more samples,
reaching around $80\%$ success with CGS versus $60\%$ with LeWM at
$K=128$--$256$, suggesting better action search at a fixed candidate budget.
TwoRooms is easy for MPPI, plausibly because its limited data diversity
and low intrinsic dimensionality simplify action search, yet CGS and TS
outperform LeWM at limited budgets, with CGS outperforming both at $K=32$.
Across these environments, CGS improves both attainable success and the rate at
which planning budgets translate into successful trajectories.

Reacher exhibits a planner-dependent pattern.
At $K=128$, CGS improves GD from $76.7\%$ to $82.0\%$ and MPPI
from $65.3\%$ to $66.7\%$, while CEM remains close to LeWM but slightly
lower ($80.7\%$ versus $83.3\%$).
TS shows a similar pattern, with a clear GD gain and smaller changes
under the sampling-based planners.
Across the sample-budget sweep, CGS gives a modest MPPI
improvement.
Section~\ref{sec:probing} shows that Reacher fingertip motion depends strongly
on the current state beyond the executed action. Appendix~\ref{app:reacher-ref-history}
compares action-only CGS with state-informed transition references that
implicitly capture state and action information. Appendix~\ref{app:cgs_prediction_constraint}
then analyzes how irreducible state-dependent drift can limit straightening and
reshape the control map favored by CGS.
\vspace{-3pt}

\section{Characterizing CGS: Latent Representations and Planning}
\label{sec:analysis}
\vspace{-2pt}

\paragraph{Probing the physical effects of control.}
\label{sec:probing}

With frozen encoders, we probe how latent differences $\Delta z_t$ represent
physical motion on PushT and Reacher. We first ask whether the evaluated
displacement is determined by action alone or requires the current state,
and then measure how explicitly that motion is represented in $\Delta z_t$.
Let $a_t$ denote the action, $s_t$ the physical state, and $\Delta s_t$
the evaluated displacement. Table~\ref{tab:probing-main} reports six
held-out scores, with disjoint training, validation, and test episodes.

\begin{table}[!t]
  \centering
  \begingroup
  \caption{\textbf{Physical-motion probing on PushT and Reacher.}
  $P_A$ and $P_{AS}$ measure how well physical motion is identified from
  action alone or from action and state; these environment-level scores are
  shared across representations. $D_{\rm lin}$ and $D_{\rm MLP}$ decode true
  motion from $\Delta z_t$, while $C_A$ and $B_{AS}$ measure agreement with
  action-only and state-conditioned control effects. Scores are held-out
  $R^2$, except the normalized agreement score $B_{AS}$
  (Appendix~\ref{app:probing}). Higher is better; bold marks the best
  representation per metric using unrounded scores.}
  \label{tab:probing-main}
  \vspace{-5pt}
  % Fixed type size avoids enlarging the six-row table to fill the page.
  \fontsize{8}{9}\selectfont
  \setlength{\tabcolsep}{3pt}
  \renewcommand{\arraystretch}{0.90}
  \setlength{\aboverulesep}{0.8pt}
  \setlength{\belowrulesep}{0.8pt}
  \setlength{\heavyrulewidth}{0.6pt}
  \setlength{\lightrulewidth}{0.35pt}
  \begin{tabular}{@{}llccrrrr@{}}
    \toprule
    \multicolumn{2}{c}{}
      & \multicolumn{2}{c}{Physical identification}
      & \multicolumn{2}{c}{Motion decoding}
      & \multicolumn{2}{c}{Control effects} \\
    \cmidrule(lr){3-4}\cmidrule(lr){5-6}\cmidrule(l){7-8}
    Environment / effect & Representation
      & $P_A$ & $P_{AS}$
      & $D_{\rm lin}$ & $D_{\rm MLP}$ & $C_A$ & $B_{AS}$ \\
    \midrule
    \multirow{3}{*}{PushT / agent}
      & LeWM     & \multirow{3}{*}{0.999} & \multirow{3}{*}{0.999}
      & 0.809 & 0.888 & 0.807 & 0.808 \\
      & LeWM+TS  & & & 0.831 & 0.891 & 0.829 & 0.830 \\
      & LeWM+CGS & & & \textbf{0.903} & \textbf{0.959}
      & \textbf{0.901} & \textbf{0.902} \\
    \midrule
    \multirow{3}{*}{Reacher / fingertip}
      & LeWM     & \multirow{3}{*}{-0.006} & \multirow{3}{*}{0.983}
      & 0.951 & \textbf{0.994} & 0.155 & 0.932 \\
      & LeWM+TS  & & & \textbf{0.974} & 0.994
      & \textbf{0.157} & \textbf{0.959} \\
      & LeWM+CGS & & & 0.899 & 0.993 & 0.041 & 0.884 \\
    \bottomrule
  \end{tabular}
  \vspace{-0.4\baselineskip}
  \endgroup
  % Keep the probe table and cross-backbone figure together, in that order.
  \par\vspace{5pt}
  \begin{minipage}{\textwidth}
\centering
  \includegraphics[width=\textwidth,trim=0 16bp 0 0,clip]{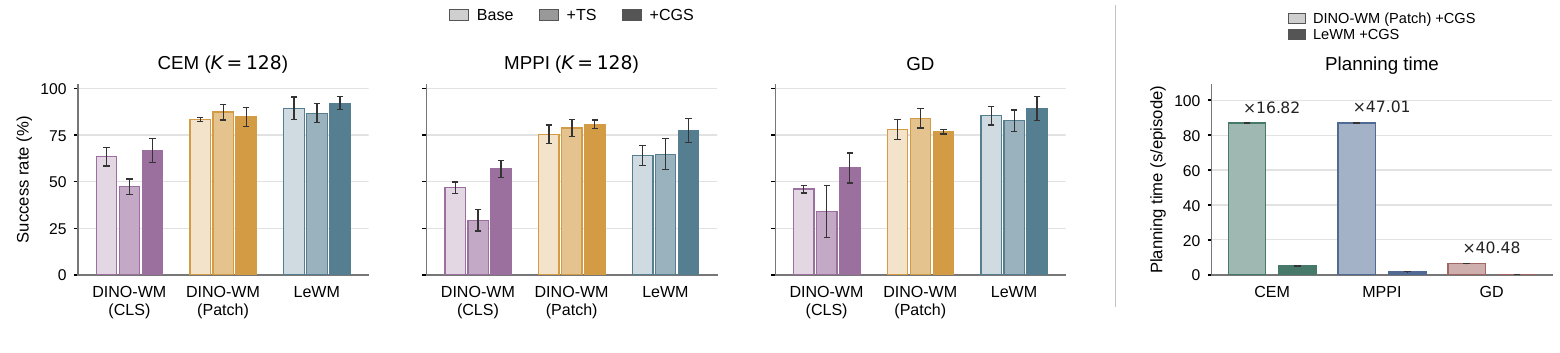}
  \makeatletter
  \def\@captype{figure}
  \makeatother
  % Keep this dense composite figure close to its caption.
  \captionsetup{skip=2.5pt}
  \caption{\textbf{Planning across representation architectures on PushT.}
Left: planning success with DINO-WM (with proprioceptive information) global CLS features, spatial Patch
features, and the end-to-end LeWM latent under CEM, MPPI, and GD,
with Base (w/o TS or CGS), TS, and CGS representation shaping. Right: planner-solve time
for DINO-WM (Patch)+CGS versus LeWM+CGS; annotated ratios are
  DINO-WM/LeWM. Full experimental details are provided in Appendix~C.5.}
  \label{fig:pusht-dino-lewm}
  \end{minipage}
\end{table}

\emph{Physical identification ($P_A,P_{AS}$).}
We predict $\Delta s_t$ from $a_t$ and $(a_t,s_t)$ using MLPs.
On PushT, action alone nearly determines agent displacement
($P_A=P_{AS}=0.999$). On Reacher, adding state raises $R^2$
from $-0.006$ to $0.983$, showing that fingertip motion depends strongly
on the current configuration and is only weakly specified by action alone
for the observed one-step motion.

\emph{Motion decoding ($D_{\rm lin},D_{\rm MLP}$).}
We decode the true displacement $\Delta s_t$ from $\Delta z_t$ with
a ridge regressor or MLP. On PushT, CGS raises linear $R^2$ from
$0.809$ to $0.903$ and nonlinear $R^2$ from $0.888$ to $0.959$.
On Reacher, all three representations retain nearly all motion information
nonlinearly, while linear decoding varies more across the three representations:
TS reaches $0.974$ and CGS $0.899$.

\emph{Control effects ($C_A$, $B_{AS}$).}
We ask whether $\Delta z_t$ reflects what the action does physically.
$C_A$ tests whether a linear readout from $\Delta z_t$ recovers
the displacement predicted from action alone; $B_{AS}$ compares its decoded
displacement with that predicted from action and current state. CGS is
strongest on PushT ($C_A=0.901$, $B_{AS}=0.902$). TS leads both on Reacher,
where the same action can produce different fingertip motion across
configurations. See Appendix~\ref{app:probing} for details.

\paragraph{CGS across representation architectures.}
\label{sec:dino-analysis}

Figure~\ref{fig:pusht-dino-lewm} follows the global-versus-spatial
representation settings studied by DINO-WM and TS~\citep{zhou2025dinowm,wang2026ts}
to examine how CGS interacts with different latent architectures. Models with DINO-WM architecture are trained following \cite{wang2026ts}, with proprioceptive information. CGS improves
planning with both DINO CLS and spatial Patch features, with notably larger
gains for the compact CLS representation on PushT. This comparison also
motivates our use of LeWM in the main experiments: frozen DINO CLS features
are compact but weaker for dynamics-based planning, whereas spatial Patch
features improve planning at substantially higher rollout cost. LeWM instead
learns a compact global latent jointly with predictive dynamics, retaining
strong planning performance while making repeated planner rollouts much
cheaper. Importantly, applying TS to the same LeWM backbone remains stronger
than the global DINO-WM+TS configuration and competitive with Patch+TS under
CEM and GD, establishing LeWM+TS as a strong baseline. CGS then benefits from this efficient learned latent while enjoying significantly faster planning.

\paragraph{Refinement efficiency and BC priors.}
\label{sec:iteration-analysis}
\label{sec:prior-analysis}

\begin{figure*}[t]
  \centering
  \includegraphics[width=\textwidth]{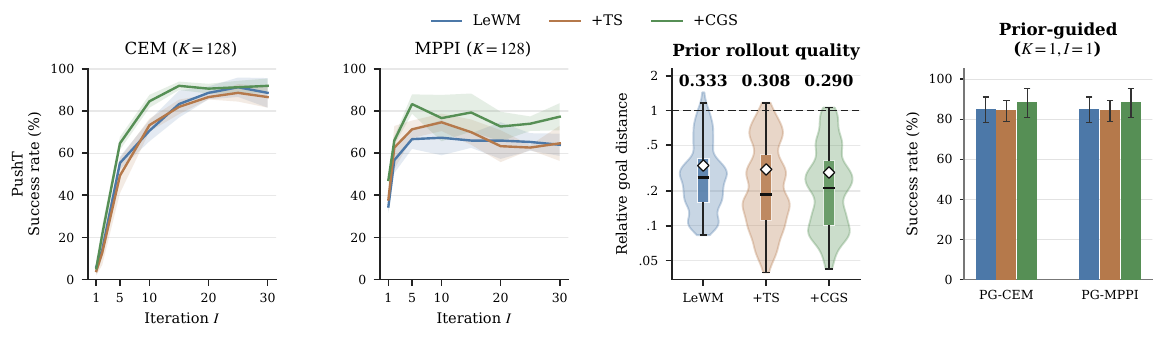}
  \captionsetup{skip=5.5pt}
\caption{\textbf{CGS improves early refinement and BC-prior quality on PushT.}
Left: CEM/MPPI success over optimization steps $I$ (iterations) at $K=128$
without an action prior. Middle: BC-prior/no-op latent goal-distance ratios for the
first plan (lower is better). Right: prior-only success using the prior mean
  as the sole candidate without additional candidate sampling ($K=1$, $I=1$).}
  \label{fig:pusht-iteration-prior}
\end{figure*}

Each CEM/MPPI optimization step refines the sampling proposal to favor lower-cost action sequences. At a fixed
candidate budget, CGS reaches high PushT success in fewer optimization
steps (Figure~\ref{fig:pusht-iteration-prior}), showing gains early in this
refinement process.
CGS also improves the matched BC prior: its rollouts achieve lower latent
goal-distance ratios and higher prior-only success before search begins.
Thus, CGS benefits both iterative refinement and prior-guided initialization.
Results for other environments appear in
Appendix~\ref{app:other-iteration-prior}.
% ===== END analysis_new.tex =====

\section{Conclusion}
\label{sec:conclusion}

We introduced CGS, a simple auxiliary loss that learns planner-friendly
representations by straightening control geometry.
Our linear-dynamics analysis connects this geometry to more balanced
terminal-cost curvature along controllable directions and more efficient
finite-budget optimization.
Experiments across multiple tasks and planners demonstrate gains in
planning success and refinement efficiency using limited planning budget.
Together, these results show that local geometric regularization can shape
the full-horizon optimization landscape without modifying the planner at test time.
More broadly, our findings highlight the importance of how predictive
representations organize action effects for planning.
Extending CGS to state-dependent control effects and richer dynamical
context using only pixel--action data is a natural next step.
Its focus on the geometry used by planners also motivates studying
how to maintain planner-friendly representations under distribution
shift, potentially in combination with test-time
adaptation~\citep{wang2026adajepa}.

% Flush main context floats before the statements and bibliography.
\clearpage

\bibliography{references}
\bibliographystyle{references}

\clearpage
\appendix
% Include appendix sections and subsections only.
\addtocontents{toc}{\protect\setcounter{tocdepth}{2}}
\begingroup
\renewcommand{\contentsname}{Appendix Contents}
\setcounter{tocdepth}{2}
\tableofcontents
\endgroup
\clearpage
% ===== BEGIN appendix_new.tex =====
\section{Extended Related Work}
\label{app:related-work}

\subsection{Latent world models and JEPA}

Latent world models compress visual observations into representations that
support prediction, planning, and imagination. Reconstruction-based models
combine learned dynamics with pixel decoders, while task-oriented models
also learn rewards, values, or policies
\citep{ha2018worldmodels,hafner2019planet,hafner2020dreamer,hansen2022tdmpc,
hansen2024tdmpc2,hafner2025dreamerv3}. JEPAs learn through prediction in
embedding space, removing the need to reconstruct every visual detail
\citep{lecun2022path,assran2023ijepa,bardes2024vjepa}.
DINO-WM uses frozen pretrained image features, and V-JEPA~2-AC adapts
video representations for action-conditioned prediction
\citep{oquabdinov22024,zhou2025dinowm,assran2025vjepa2}.
PLDM and LeWM instead learn visual representations and dynamics jointly
from offline control trajectories~\citep{sobal2025pldm,maes2026lewm}.

\noindent\textbf{Stable prediction and representation collapse.}
A constant encoder can trivially satisfy an embedding prediction loss,
making collapse prevention a central design concern. Target encoders and
stop-gradient operations stabilize predictive targets; variance--covariance
regularization maintains feature diversity; and Gaussian distribution
matching provides an explicit distributional constraint
\citep{assran2023ijepa,bardes2024vjepa,bardes2022vicreg,
balestriero2025lejepa,maes2026lewm}.
Further architectures investigate object-level latent masking and sparse
codes, while AdaJEPA adapts a trained world model at test time
\citep{nam2026causaljepa,kuang2026lpwm,wang2026adajepa}.
These choices address representation content, training stability, and
adaptation. The geometry exposed to downstream action optimization remains
an additional concern: accurate transition prediction and diverse features
alone do not ensure that the planning objective is easy to optimize.
Empirical studies of JEPA planning and reachability-aware training examine
this gap directly~\citep{terver2026drivers,li2026predictive}.
CGS addresses this concern through an auxiliary geometry objective
applicable to end-to-end learned and pretrained representations.

\subsection{Planner-friendly representations}

\noindent\textbf{Latent dynamics and optimization.}
Embed-to-Control, Robust Controllable Embedding, SOLAR, and predictive-coding
approaches learn representations suited to locally linear control
\citep{watter2015embed,banijamali2018robust,zhang2019solar,
levine2020pcc,shu2020predictive}.
Here, \emph{locally linear latent dynamics} means that transitions in a
neighborhood or around a nominal trajectory admit a linear model whose
coefficients may vary across states or time. This structure supports local
trajectory optimization, including iLQR.
Universal Planning Networks backpropagate through an unrolled planner;
CARL couples representations to actor--critic updates and value-weighted
model fitting~\citep{srinivas2018upn,cui2020control}.
DeepMDP and bisimulation-based representations preserve reward and
transition structure, while softly state-invariant models simplify action
effects across states
\citep{gelada2019deepmdp,zhang2021invariant,saanum2024simplifying}.
Other approaches learn feasible latent paths or distances reflecting
directed reachability
\citep{kurutach2018learning,yang2020plan2vec,eysenbach2024inference,
wang2023quasimetric}.

\noindent\textbf{Goal-derived and inverse-control supervision.}
Value-Guided JEPA learns goal-conditioned values; RC-aux adds multi-horizon
prediction and budget-conditioned reachability labels; and
Temporal-Distance-JEPA encodes directed temporal progress
\citep{destrade2026value,li2026predictive,bai2026temporal}.
INTACT uses physical transition intent and future-goal intent to learn an
action distribution that can be deployed directly~\citep{sun2026intact}.
Future observations and temporal offsets in offline trajectories supply
pseudo-goals or reachability targets for these objectives; these are
constructed training signals rather than externally annotated goals.
SCALE uses task-relevant state differences to calibrate latent distance
\citep{hu2026scale}. Sensorimotor World Models and Delta-JEPA use inverse
dynamics to retain action information in endpoint embeddings or latent
displacements~\citep{ivashkov2026sensorimotor,zhang2026delta}.
Related work studies control-irrelevant visual invariance and the action
excitation needed to identify controlled dynamics
\citep{toso2026invariant,zhang2026identifiability}.
The CGS world-model objective uses local transitions from pixel--action
pairs to match action and latent-difference cosines, without future-goal
targets or an inverse action decoder.

\noindent\textbf{Temporal Straightening.}
Temporal Straightening aligns consecutive latent displacements by
penalizing their angular deviation, reducing local trajectory curvature
\citep{wang2026ts}. Straighter latent paths can improve planning-objective
conditioning while preserving a nonlinear predictor.
The regularizer operates directly on consecutive observations, without
pseudo-goal selection, value targets, reachability labels, or contrastive
negative mining.
TS constrains successive displacements along observed paths. CGS uses
the corresponding actions as references for pairwise angular relations
among latent differences. When consecutive actions align, the adjacent
CGS term reduces to the squared angular TS penalty; changing action
directions provides references for latent turns.
Under the assumptions of our linear-dynamics analysis, CGS promotes
temporal straightening, balanced action effects, and reduced drift--control
coupling. These properties connect local transition geometry to more
balanced terminal-cost curvature across the full planning horizon.
Appendix~\ref{app:planning} defines the TS objective, and
Appendix~\ref{app:dino-analysis} describes its encoder instantiations and
our DINO-WM comparisons.

\subsection{Sampling-based MPC}

Sampling-based MPC evaluates candidate action sequences through model rollouts
and repeatedly updates a proposal distribution. CEM fits a distribution to
low-cost elites, while MPPI uses cost-dependent weights over sampled sequences
\citep{rubinstein1999crossentropy,deboer2005cem,kroese2006continuous,
kappen2005path,theodorou2010path,williams2017mppi}.
These planners have been combined with probabilistic ensembles, stochastic
latent dynamics, and predictive feature-space models
\citep{chua2018pets,hafner2019planet,zhou2025dinowm,sobal2025pldm,
maes2026lewm}.

\noindent\textbf{Proposals and sample reuse.}
iCEM combines temporally correlated sampling with elite memory and sample
reuse; hybrid methods interleave CEM and gradient updates
\citep{pinneri2021icem,bharadhwaj2020mpc}.
Biased-MPPI incorporates ancillary controllers, CoVO-MPC adapts covariance
to local cost curvature, and PRISM learns a state--goal-conditioned proposal
prior~\citep{trevisan2024biased,yi2024covo,wang2026prism}.
These approaches use temporal structure, local cost information, or learned
behavior to concentrate evaluations on promising action sequences.
CGS complements proposal design by shaping the control geometry of the
world model during training.

\noindent\textbf{Finite-budget optimization.}
Classical CEM analyses study distribution updates and convergence
\citep{margolin2005cem,costa2007cem}; recent MPPI analyses connect updates
to preconditioned gradient descent and examine closed-loop stability
\citep{fazlyab2026mppi,yoon2026finite}.
Importance-sampling results clarify how mismatch between proposals and
high-weight regions affects the sample budget needed for reliable estimates
\citep{agapiou2017is,chatterjee2018importance}.
In MPC, the horizon, effective action dimension, cost landscape, and
proposal design jointly determine how informative a finite candidate set
will be~\citep{yoon2022sampling}.
Our linear-dynamics analysis connects local representation learning to
finite-horizon terminal-cost curvature, yielding finite-budget guarantees
for MPPI, local contraction results for CEM, and convergence bounds for
gradient descent.

\subsection{Relational and Similarity-Preserving Knowledge Distillation}

Relational Knowledge Distillation transfers inter-example distances and
angles, while similarity-preserving knowledge distillation preserves
pairwise similarities across teacher and student representations
\citep{park2019relational,tung2019similarity}.
They are relevant to CGS through the general idea of matching structure
across different representation spaces.
CGS uses a related pairwise matching form, but observed actions provide
the reference geometry rather than a teacher representation. Specifically,
it matches pairwise action cosines to those among corresponding latent
differences. CGS therefore uses structural matching to
organize control geometry for multi-step planning; our analysis connects
that geometry to more balanced terminal-cost curvature and
sampling-efficient optimization.
\clearpage

\section{Referenced Methods and Planning Algorithms}
\label{app:planning}
\suppressfloats[t]

This section defines the planning algorithms and the TS objective using the
notation of Section~\ref{sec:cgs}. Experimental settings and the DINO-WM
representation comparison are collected in Appendix~\ref{app:additional-planning-results}.

\subsection{Planning algorithms}
\label{app:planning-optimizers}

We write the planner updates using $C(\ba)$ from Eq.~\eqref{eq:cgs-planning-cost}.
In implementation, the cost and MPPI temperature use the matched
normalizations specified in Appendix~\ref{app:planning-datasets};
DINO-WM-specific reductions and proprioceptive terms are specified in
Appendix~\ref{app:dino-training-details}.
Below, $K$ is the number of evaluated candidates per optimization step, $I$ the
number of optimization steps (iterations), and $n_H=H \times d_a$ the flattened action-sequence dimension.
The Gaussian proposal has mean $\mu_i$ and coordinatewise standard deviation
$\sigma_{\mathrm p,i}$. Sequences are flattened for Gaussian updates and
reshaped for model rollout. Following the released LeWM planners, the candidate
set contains the current mean and $K-1$ Gaussian draws~\citep{maes2026lewm}.

\paragraph{Cross-Entropy Method (CEM).}
CEM refits a diagonal Gaussian to the lowest-cost elite
set $\cE_i$~\citep{rubinstein1999crossentropy,deboer2005cem}.
For $M=|\cE_i|\geq2$, the implementation uses the elite mean and sample variance,
\begin{equation}
  \mu_{i+1}=\frac1M\sum_{k\in\cE_i}\ba^{(k)},
  \qquad
  \sigma_{\mathrm p,i+1}^2=\frac1{M-1}
    \sum_{k\in\cE_i}(\ba^{(k)}-\mu_{i+1})^2.
  \label{eq:appendix-cem}
\end{equation}
The square is coordinatewise. CEM starts from zero mean and unit standard
deviation and returns the final mean. For $M=1$, CEM updates the
mean and retains the previous standard deviation,
$\sigma_{\mathrm p,i+1}=\sigma_{\mathrm p,i}$; this keeps the proposal
finite and allows continued sampling. Algorithm~\ref{alg:cem} includes this case.

\begin{algorithm}[!htbp]
  \caption{CEM and prior-guided CEM}
  \label{alg:cem}
  \begin{algorithmic}[1]
    \Require Goal cost $C$, candidates $K$, optimization steps $I$,
      elites $1\leq M\leq K$, optional prior mean $\mu_p$
    \State $\mu\gets\mu_p$ for PG-CEM or $0$ for CEM;
      $\sigma_{\mathrm p}\gets\mathbf1$
    \For{$i=1,\ldots,I$}
      \State $\ba^{(1)}\gets\mu$
      \For{$k=2,\ldots,K$}
        \State $\epsilon^{(k)}\sim\cN(0,\mathrm I_{n_H})$;
          $\ba^{(k)}\gets\mu+\sigma_{\mathrm p}\odot\epsilon^{(k)}$
      \EndFor
      \State $c_k\gets C(\ba^{(k)})$ for $k=1,\ldots,K$
      \State $\cE\gets$ indices of the $M$ smallest costs
      \State $\mu\gets M^{-1}\sum_{k\in\cE}\ba^{(k)}$
      \If{$M>1$}
        \State $\sigma_{\mathrm p}^2\gets(M-1)^{-1}\sum_{k\in\cE}(\ba^{(k)}-\mu)^2$
      \EndIf
    \EndFor
    \State \Return $\mu$
  \end{algorithmic}
\end{algorithm}

\paragraph{Model Predictive Path Integral (MPPI) control.}
We use the MPPI cost-weighted update~\citep{kappen2005path,theodorou2010path,williams2017mppi},
\begin{equation}
  w_k=\frac{\exp[-C(\ba^{(k)})/\tau]}
    {\sum_{j=1}^K\exp[-C(\ba^{(j)})/\tau]},
  \qquad
  \mu_{i+1}=\sum_{k=1}^K w_k\ba^{(k)}.
  \label{eq:appendix-mppi}
\end{equation}
All candidates contribute to the mean update, and the sampling standard
deviation stays fixed at one. Smaller $\tau$ concentrates weight on
lower-cost candidates. We share $\tau$ across training objectives and
budgets within each task and backbone family; Appendix~\ref{app:pusht-temperature-tuning}
reports the temperature analysis.

\paragraph{Prior-guided CEM and MPPI (PG-CEM and PG-MPPI).}
A goal-conditioned behavior-cloning head $\pi_\psi$, trained for each frozen
representation, predicts an action-sequence mean,
\begin{equation}
  \mu_p=\pi_\psi(z_t,z_g)\in\R^{H\times d_a}.
  \label{eq:appendix-pg-mean}
\end{equation}
The PG variants initialize $\mu_0=\mu_p$ and retain the corresponding
planner's unit initial standard deviation and refinement rule. This tests
how a representation supports both a learned proposal mean and subsequent
model-based optimization. An accurate prior mean alone does not ensure a
well-suited action-sequence distribution; CEM and MPPI continue to refine
this distribution through model-based sampling, with CEM updating its
mean and variance and MPPI updating its mean at fixed variance.
Prior training is specified in Appendix~\ref{app:pg-training-details}.

\begin{algorithm}[!htbp]
  \caption{MPPI and prior-guided MPPI}
  \label{alg:mppi-family}
  \begin{algorithmic}[1]
    \Require $C,K,I,\tau$; optional prior mean $\mu_p$
    \State $\mu\gets\mu_p$ for PG-MPPI or $0$ for MPPI; $\sigma_{\mathrm p}\gets\mathbf1$
    \For{$i=1,\ldots,I$}
      \State $\ba^{(1)}\gets\mu$
      \For{$k=2,\ldots,K$}
        \State $\epsilon^{(k)}\sim\cN(0,\mathrm I_{n_H})$;
          $\ba^{(k)}\gets\mu+\sigma_{\mathrm p}\odot\epsilon^{(k)}$
      \EndFor
      \State $c_k\gets C(\ba^{(k)})$ for $k=1,\ldots,K$
      \State $c_{\min}\gets\min_k c_k$;
        $\widetilde w_k\gets\exp[-(c_k-c_{\min})/\tau]$
      \State $w_k\gets\widetilde w_k/\sum_j\widetilde w_j$
      \State $\mu\gets\sum_{k=1}^K w_k\ba^{(k)}$
    \EndFor
    \State \Return $\mu$
  \end{algorithmic}
\end{algorithm}

\paragraph{Gradient-based planner (GD).}
Following the TS planning implementation~\citep{wang2026ts}, GD directly
optimizes the action sequence through the frozen differentiable world model.
Starting from an initial action sequence $\ba_0$, the rollout cost is minimized
by iterative gradient-based updates,
\begin{equation}
  \ba_{i+1}
  =
  \ba_i-\eta_i\nabla_{\ba} C(\ba_i),
  \label{eq:appendix-gd-mpc}
\end{equation}
where $\eta_i$ denotes the optimization step size. In practice, we use Adam
with cosine learning-rate decay, initialize from the standardized zero-control
sequence, and inject no action noise. GD optimizes a single action trajectory and is independent with sampling budget $K$. Experimental settings and the per-solve cost construction
are given in Appendix~C.1.1.

\subsection{Temporal Straightening}
\label{app:temporal-straightening}

Temporal Straightening (TS) regularizes the curvature of latent trajectories
by maximizing the cosine similarity between consecutive latent
displacements~\citep{wang2026ts}. Using the main context notation for global visual
latents, define
\begin{equation}
  \Delta z_t=z_{t+1}-z_t,\qquad
  v_t=\frac{\Delta z_t}{\|\Delta z_t\|_2}.
  \label{eq:appendix-ts-velocity}
\end{equation}
In implementation, we exclude adjacent pairs for which either latent
displacement has norm below $10^{-6}$. Let $\mathcal T$ denote the
remaining valid pairs. TS uses the average cosine-based curvature penalty,
\begin{equation}
  \cL_{\rm TS}=\frac1{|\mathcal T|}\sum_{t\in\mathcal T}
    (1-v_t^\top v_{t+1}),\qquad
  \cL_{\rm LeWM+TS}=\cL_{\rm LeWM}+\lambda_{\rm TS}\cL_{\rm TS},
  \label{eq:appendix-ts-objective}
\end{equation}
with $\lambda_{\rm TS} \geq 0$. Our LeWM adaptation retains the prediction and SIGReg terms in
Eq.~\eqref{eq:cgs-lewm-objective}, with gradients through both branches of
the prediction loss~\citep{maes2026lewm}. The original TS implementation
uses stop-gradient prediction targets~\citep{wang2026ts}; its DINO-based
instantiations and our CLS/Patch implementations are described in
Appendix~\ref{app:dino-training-details}.
Although LeWM exhibits implicit temporal straightening on
PushT~\citep[Appendix H]{maes2026lewm}, explicit TS can further reduce
latent trajectory curvature and improve planning, with clearer planning
gains on Cube and TwoRooms; see Appendix~\ref{app:lewm-ts-cosine-effect}.

\clearpage
\section{Experiment Details and Additional Experimental Results}
\label{app:additional-planning-results}
\label{app:analysis-details}

\begingroup
\raggedbottom
\renewcommand{\arraystretch}{0.96}
\setlength{\aboverulesep}{1pt}
\setlength{\belowrulesep}{1pt}

\subsection{Main-experiment implementation and evaluation}
\label{app:main-experiment-details}

\subsubsection{Experiment details}
\label{app:experimental-details}

\paragraph{Environments and training.}
\label{app:lewm-training-details}
We evaluate LeWM on PushT, Cube, TwoRooms, and Reacher using offline
pixel--action trajectories, $224\times224$ images, with $5$ frame skips, and the original
benchmark success predicates.
The LeWM models and the TS and CGS variants follow the same LeWM \citep{maes2026lewm} training setup and retain its backbones and architecture of ViT-Tiny/14 encoder, 192-dimensional projected CLS latent, and
action-conditioned transformer predictor.
Training uses observation encodings as predictor inputs and computes all
$N$ next-step predictions in one temporally masked forward pass using
Eq.~\eqref{eq:causal-masking}.
One model per environment and method is fixed across all planners
and planning budgets.
For CGS, action-reference or latent-difference vectors with $\ell_2$
norm below $10^{-6}$ are excluded from cosine normalization, and
off-diagonal pairs involving such vectors are omitted from the loss average.
On Cube, the CGS term uses the XYZ
displacement components of the action as its geometry reference when the grasp
command is active, while the full action is retained in the LeWM prediction
objective; see Appendix~\ref{app:cube-ref-ablation} for more discussions.

\paragraph{Regularization coefficients.}
\label{app:regularization-coefficients}
We keep the base LeWM hyperparameters fixed across all LeWM-based
methods, including $\lambda_{\rm SIGReg}=0.09$, and vary only the
coefficient introduced by each additional regularizer.
For LeWM+TS, we sweep
$\lambda_{\mathrm{TS}}\in\{0.01,0.05,0.1,0.5\}$ and use $0.01$ on
PushT, $0.1$ on Cube and Reacher, and $0.5$ on TwoRooms.
CGS uses $\lambda_{\mathrm{CGS}}=0.5$ by default and $0.01$ on PushT.
All selected coefficients are fixed across planners and planning budgets.
Ablations of the CGS ramp, history length, $\lambda_{\rm SIGReg}$ and $\lambda_{\rm CGS}$ are reported in
Appendix~\ref{app:training-ablations}.

\paragraph{Planning and evaluation.}
\label{app:planning-datasets}
During planning, the world model is frozen and predicted latents are fed
back into the history for autoregressive rollout
\citep{maes2026lewm}.
Actions are optimized in the standardized coordinates used during
training and converted to environment coordinates for execution.
Each model step spans five environment steps. Each solve uses $H=5$ model
steps, corresponding to 25 environment steps, executes the resulting
25-step action sequence, and then replans from the new observation.
Episodes contain at most 50 executed environment actions, giving two
closed-loop solves.
Initial observations are sampled from trajectories in the offline dataset, and goals are the observations 25 environment steps later in the same trajectory, \(o_g=o_{t+25}\).
This horizon, goal construction, and execution cadence are shared by all
planners.

MPPI and CEM use $I=30$ optimization steps and
$K\in\{1,2,4,8,16,32,64,128,256\}$ candidates.
MPPI weights all candidates, while CEM retains
$\max(1,\lfloor K/4\rfloor)$ elites.
Each candidate set includes the current proposal mean; thus $K=1$
evaluates only the zero or prior initialization.
The single-elite CEM case follows Algorithm~1, updating the mean while
retaining the previous sampling standard deviation.
LeWM sampling costs use the squared terminal latent error.
The MPPI temperatures are $4$, $64$, $128$, and $64$ on PushT, Cube,
TwoRooms, and Reacher, respectively, and are fixed across training
objectives and planning budgets. Appendix~\ref{app:pusht-temperature-tuning}
reports these sweeps; for each environment,
we select a shared value that balances performance across the three methods.
GD follows the gradient-based planning setup of \citet{wang2026ts}, using
100 Adam steps, an initial learning rate of $0.1$, cosine decay,
zero-control initialization, and no injected action noise, to minimize terminal latent MSE.

\paragraph{Matched behavior-cloning priors.}
\label{app:pg-training-details}
We train a separate lightweight three-layer MLP \citep{wang2026prism} for
each frozen representation to map $(z_t,z_g)$ to the mean of a
25-step action sequence.
Trajectory future observations provide the goals, which are the same as planning goals in evaluation, and checkpoints are
selected by validation loss.
The prior is trained only after freezing the representation and does not
affect world-model training.
PG-MPPI and PG-CEM initialize their proposal mean with this prediction
and otherwise retain the corresponding planner's unit initial variance,
cost, and refinement rule.
Thus each representation is evaluated with its own matched prior.

\subsubsection{Additional results}
\label{app:lewm-full-sweep}

\paragraph{Statistics and sample-budget sweep results.}
All results report mean success rate and sample standard deviation over
seeds $\{0,1,42\}$, with 50 episodes per seed and shared start indices
across methods.
Planner-only timing includes both solves and excludes model loading,
environment stepping, and rendering; the timing protocol is given in
Appendix~C.5.
Tables~3--6 report the candidate sample-budget $K$ results underlying
Table~1 and Figure~4.

\begin{table}[H]
  \centering
  \vspace*{\baselineskip}
  \caption{\textbf{PushT sample-budget sweep.} Success rate (\%, mean $\pm$ sample SD over three evaluation seeds, 50 episodes each). Bold marks the best representation for each planner and budget, including ties. GD is independent of $K$.}
  \label{tab:lewm-full-sweep-pusht}
  \setlength{\tabcolsep}{2.6pt}
  \renewcommand{\arraystretch}{0.96}
  \resizebox{\textwidth}{!}{%
  \begin{tabular}{@{}ll*{9}{c}@{}}
    \toprule
    World model & Planner & $K=1$ & $K=2$ & $K=4$ & $K=8$ & $K=16$ & $K=32$ & $K=64$ & $K=128$ & $K=256$ \\
    \midrule
    \multirow{5}{*}{LeWM} & MPPI & \lembsr{0.0}{0.0} & \lemsr{30.7}{6.1} & \lemsr{42.7}{9.0} & \lemsr{56.7}{1.1} & \lemsr{58.7}{10.1} & \lemsr{61.3}{1.1} & \lemsr{59.3}{10.1} & \lemsr{64.0}{5.3} & \lemsr{60.7}{4.2} \\
     & PG-MPPI & \lemsr{84.7}{6.4} & \lemsr{82.0}{4.0} & \lemsr{82.7}{6.1} & \lemsr{77.3}{5.0} & \lemsr{76.0}{6.0} & \lemsr{78.7}{6.4} & \lemsr{72.7}{8.3} & \lemsr{72.0}{6.0} & \lemsr{62.0}{3.5} \\
     & CEM & \lembsr{0.0}{0.0} & \lembsr{35.3}{3.1} & \lemsr{37.3}{6.4} & \lemsr{56.0}{5.3} & \lemsr{74.0}{7.2} & \lemsr{83.3}{4.2} & \lemsr{83.3}{5.0} & \lemsr{89.3}{6.1} & \lemsr{87.3}{4.2} \\
     & PG-CEM & \lemsr{85.3}{6.1} & \lemsr{84.0}{3.5} & \lemsr{80.7}{5.8} & \lemsr{91.3}{1.1} & \lemsr{89.3}{1.1} & \lemsr{93.3}{1.1} & \lemsr{94.0}{4.0} & \lemsr{92.0}{5.3} & \lemsr{92.0}{2.0} \\
     & GD & \multicolumn{9}{c}{\lemsr{85.3}{5.0}} \\
    \addlinespace[1pt]
    \multirow{5}{*}{LeWM + TS} & MPPI & \lembsr{0.0}{0.0} & \lemsr{38.0}{5.3} & \lemsr{40.7}{6.4} & \lemsr{50.0}{5.3} & \lemsr{51.3}{1.1} & \lemsr{56.7}{5.0} & \lemsr{65.3}{3.1} & \lemsr{64.7}{8.3} & \lemsr{66.7}{7.0} \\
     & PG-MPPI & \lemsr{84.0}{5.3} & \lemsr{84.0}{6.0} & \lemsr{82.7}{5.8} & \lemsr{78.0}{2.0} & \lemsr{76.7}{3.1} & \lemsr{78.7}{3.1} & \lemsr{72.7}{7.6} & \lemsr{69.3}{9.0} & \lemsr{70.7}{3.1} \\
     & CEM & \lembsr{0.0}{0.0} & \lemsr{28.0}{2.0} & \lemsr{34.7}{10.1} & \lemsr{55.3}{6.4} & \lemsr{73.3}{6.1} & \lemsr{85.3}{3.1} & \lembsr{89.3}{5.0} & \lemsr{86.7}{5.0} & \lemsr{89.3}{3.1} \\
     & PG-CEM & \lemsr{84.0}{5.3} & \lemsr{89.3}{1.1} & \lemsr{86.0}{3.5} & \lemsr{89.3}{7.6} & \lembsr{94.0}{4.0} & \lembsr{94.7}{3.1} & \lemsr{92.0}{2.0} & \lemsr{90.0}{5.3} & \lemsr{91.3}{4.2} \\
     & GD & \multicolumn{9}{c}{\lemsr{82.7}{5.8}} \\
    \addlinespace[1pt]
    \multirow{5}{*}{LeWM + CGS} & MPPI & \lembsr{0.0}{0.0} & \lembsr{40.0}{5.3} & \lembsr{52.0}{2.0} & \lembsr{60.7}{8.1} & \lembsr{68.7}{4.2} & \lembsr{67.3}{6.1} & \lembsr{72.0}{2.0} & \lembsr{77.3}{6.4} & \lembsr{75.3}{7.6} \\
     & PG-MPPI & \lembsr{88.0}{7.2} & \lembsr{89.3}{4.6} & \lembsr{88.0}{5.3} & \lembsr{87.3}{6.1} & \lembsr{86.7}{1.1} & \lembsr{84.7}{5.8} & \lembsr{85.3}{1.1} & \lembsr{80.0}{2.0} & \lembsr{84.0}{9.2} \\
     & CEM & \lembsr{0.0}{0.0} & \lembsr{35.3}{1.1} & \lembsr{42.7}{1.1} & \lembsr{68.7}{5.0} & \lembsr{83.3}{4.2} & \lembsr{86.0}{4.0} & \lemsr{88.7}{4.6} & \lembsr{92.0}{3.5} & \lembsr{90.7}{5.8} \\
     & PG-CEM & \lembsr{88.0}{7.2} & \lembsr{90.7}{6.1} & \lembsr{86.7}{5.0} & \lembsr{96.0}{4.0} & \lembsr{94.0}{2.0} & \lemsr{94.0}{5.3} & \lembsr{96.0}{2.0} & \lembsr{96.0}{0.0} & \lembsr{96.0}{4.0} \\
     & GD & \multicolumn{9}{c}{\lembsr{89.3}{6.4}} \\
    \bottomrule
  \end{tabular}}
\end{table}

\begin{table}[H]
  \centering
  \caption{\textbf{Cube sample-budget sweep.} Success rate (\%, mean $\pm$ sample SD over three evaluation seeds, 50 episodes each). Bold marks the best representation for each planner and budget, including ties. GD is independent of $K$.}
  \label{tab:lewm-full-sweep-cube}
  \setlength{\tabcolsep}{2.6pt}
  \renewcommand{\arraystretch}{0.96}
  \resizebox{\textwidth}{!}{%
  \begin{tabular}{@{}ll*{9}{c}@{}}
    \toprule
    World model & Planner & $K=1$ & $K=2$ & $K=4$ & $K=8$ & $K=16$ & $K=32$ & $K=64$ & $K=128$ & $K=256$ \\
    \midrule
    \multirow{5}{*}{LeWM} & MPPI & \lembsr{44.7}{7.0} & \lemsr{46.0}{5.3} & \lemsr{44.7}{11.0} & \lemsr{46.7}{6.4} & \lemsr{51.3}{4.2} & \lemsr{51.3}{5.0} & \lemsr{56.7}{5.0} & \lemsr{59.3}{4.6} & \lemsr{64.0}{9.2} \\
     & PG-MPPI & \lembsr{100.0}{0.0} & \lembsr{75.3}{8.3} & \lembsr{70.0}{5.3} & \lemsr{60.7}{4.2} & \lemsr{61.3}{6.1} & \lemsr{76.7}{2.3} & \lemsr{82.7}{6.4} & \lemsr{89.3}{4.2} & \lemsr{90.0}{2.0} \\
     & CEM & \lembsr{44.7}{7.0} & \lembsr{50.0}{3.5} & \lemsr{45.3}{10.1} & \lembsr{55.3}{9.0} & \lemsr{57.3}{7.0} & \lemsr{62.7}{10.1} & \lemsr{64.0}{6.0} & \lemsr{63.3}{2.3} & \lemsr{64.0}{11.1} \\
     & PG-CEM & \lembsr{100.0}{0.0} & \lemsr{94.0}{0.0} & \lemsr{88.0}{5.3} & \lemsr{91.3}{3.1} & \lembsr{94.7}{1.1} & \lemsr{89.3}{4.2} & \lemsr{88.0}{4.0} & \lemsr{90.0}{2.0} & \lemsr{89.3}{1.1} \\
     & GD & \multicolumn{9}{c}{\lemsr{59.3}{8.1}} \\
    \addlinespace[1pt]
    \multirow{5}{*}{LeWM + TS} & MPPI & \lembsr{44.7}{7.0} & \lembsr{48.0}{5.3} & \lembsr{46.7}{3.1} & \lemsr{48.0}{2.0} & \lemsr{46.0}{8.7} & \lemsr{56.7}{9.0} & \lemsr{68.0}{5.3} & \lemsr{75.3}{5.0} & \lemsr{75.3}{6.1} \\
     & PG-MPPI & \lemsr{98.0}{2.0} & \lemsr{65.3}{5.0} & \lemsr{65.3}{10.3} & \lemsr{62.7}{14.5} & \lembsr{74.7}{5.0} & \lemsr{84.7}{7.6} & \lemsr{87.3}{2.3} & \lemsr{92.7}{3.1} & \lemsr{93.3}{1.1} \\
     & CEM & \lembsr{44.7}{7.0} & \lemsr{48.7}{4.2} & \lemsr{47.3}{1.1} & \lemsr{54.0}{5.3} & \lemsr{61.3}{4.6} & \lemsr{64.7}{7.6} & \lemsr{71.3}{6.4} & \lemsr{73.3}{10.3} & \lemsr{73.3}{9.9} \\
     & PG-CEM & \lemsr{98.0}{2.0} & \lemsr{92.7}{4.6} & \lemsr{86.7}{5.0} & \lemsr{88.7}{4.2} & \lemsr{89.3}{3.1} & \lemsr{88.7}{3.1} & \lemsr{88.7}{1.1} & \lemsr{90.0}{2.0} & \lemsr{90.7}{3.1} \\
     & GD & \multicolumn{9}{c}{\lemsr{65.3}{4.6}} \\
    \addlinespace[1pt]
    \multirow{5}{*}{LeWM + CGS} & MPPI & \lembsr{44.7}{7.0} & \lemsr{44.0}{4.0} & \lembsr{46.7}{9.4} & \lembsr{50.0}{7.2} & \lembsr{54.0}{8.0} & \lembsr{64.0}{6.9} & \lembsr{69.3}{9.9} & \lembsr{79.3}{4.6} & \lembsr{82.7}{6.1} \\
     & PG-MPPI & \lembsr{100.0}{0.0} & \lemsr{60.0}{4.0} & \lemsr{62.7}{7.6} & \lembsr{64.7}{9.9} & \lemsr{72.7}{3.1} & \lembsr{86.7}{3.1} & \lembsr{92.7}{2.3} & \lembsr{96.7}{3.1} & \lembsr{98.7}{1.1} \\
     & CEM & \lembsr{44.7}{7.0} & \lemsr{48.7}{4.2} & \lembsr{48.7}{2.3} & \lembsr{55.3}{4.2} & \lembsr{64.0}{10.4} & \lembsr{68.7}{5.8} & \lembsr{75.3}{11.6} & \lembsr{76.7}{9.5} & \lembsr{82.0}{6.0} \\
     & PG-CEM & \lembsr{100.0}{0.0} & \lembsr{96.0}{2.0} & \lembsr{95.3}{3.1} & \lembsr{92.0}{4.0} & \lemsr{90.7}{5.0} & \lembsr{94.7}{2.3} & \lembsr{93.3}{1.1} & \lembsr{96.0}{4.0} & \lembsr{95.3}{4.2} \\
     & GD & \multicolumn{9}{c}{\lembsr{71.3}{4.6}} \\
    \bottomrule
  \end{tabular}}
\end{table}

\begin{table}[H]
  \centering
  \caption{\textbf{TwoRooms sample-budget sweep.} Success rate (\%, mean $\pm$ sample SD over three evaluation seeds, 50 episodes each). Bold marks the best representation for each planner and budget, including ties. GD is independent of $K$.}
  \label{tab:lewm-full-sweep-tworoom}
  \setlength{\tabcolsep}{2.6pt}
  \renewcommand{\arraystretch}{0.96}
  \resizebox{\textwidth}{!}{%
  \begin{tabular}{@{}ll*{9}{c}@{}}
    \toprule
    World model & Planner & $K=1$ & $K=2$ & $K=4$ & $K=8$ & $K=16$ & $K=32$ & $K=64$ & $K=128$ & $K=256$ \\
    \midrule
    \multirow{5}{*}{LeWM} & MPPI & \lembsr{13.3}{1.1} & \lemsr{40.7}{7.0} & \lemsr{53.3}{7.0} & \lemsr{60.0}{10.6} & \lemsr{70.7}{11.7} & \lemsr{82.0}{10.0} & \lemsr{88.7}{4.2} & \lemsr{92.0}{3.5} & \lemsr{94.7}{1.1} \\
     & PG-MPPI & \lembsr{100.0}{0.0} & \lemsr{77.3}{9.9} & \lemsr{68.7}{4.2} & \lemsr{79.3}{1.1} & \lemsr{87.3}{4.6} & \lemsr{91.3}{6.4} & \lembsr{98.7}{1.1} & \lemsr{98.0}{2.0} & \lembsr{100.0}{0.0} \\
     & CEM & \lembsr{13.3}{1.1} & \lemsr{55.3}{9.0} & \lemsr{57.3}{6.4} & \lemsr{56.7}{3.1} & \lemsr{68.0}{5.3} & \lemsr{79.3}{4.2} & \lemsr{79.3}{9.0} & \lemsr{76.7}{10.1} & \lemsr{78.0}{5.3} \\
     & PG-CEM & \lembsr{100.0}{0.0} & \lemsr{98.0}{3.5} & \lemsr{97.3}{2.3} & \lemsr{99.3}{1.1} & \lemsr{96.0}{4.0} & \lemsr{96.7}{2.3} & \lemsr{96.0}{3.5} & \lemsr{93.3}{3.1} & \lemsr{94.7}{2.3} \\
     & GD & \multicolumn{9}{c}{\lemsr{41.3}{7.0}} \\
    \addlinespace[1pt]
    \multirow{5}{*}{LeWM + TS} & MPPI & \lembsr{13.3}{1.1} & \lembsr{56.0}{3.5} & \lembsr{64.0}{9.2} & \lemsr{74.7}{1.1} & \lemsr{81.3}{1.1} & \lemsr{88.0}{5.3} & \lemsr{95.3}{4.2} & \lembsr{97.3}{3.1} & \lemsr{96.0}{2.0} \\
     & PG-MPPI & \lembsr{100.0}{0.0} & \lemsr{78.7}{4.2} & \lemsr{69.3}{6.4} & \lemsr{78.0}{4.0} & \lemsr{84.7}{7.0} & \lemsr{92.0}{2.0} & \lemsr{95.3}{5.0} & \lemsr{98.0}{2.0} & \lemsr{98.0}{2.0} \\
     & CEM & \lembsr{13.3}{1.1} & \lembsr{68.7}{8.1} & \lemsr{68.7}{7.6} & \lembsr{78.0}{2.0} & \lembsr{88.7}{5.0} & \lembsr{95.3}{3.1} & \lembsr{96.0}{2.0} & \lembsr{97.3}{1.2} & \lembsr{96.7}{2.3} \\
     & PG-CEM & \lembsr{100.0}{0.0} & \lembsr{98.7}{1.1} & \lembsr{100.0}{0.0} & \lembsr{100.0}{0.0} & \lembsr{100.0}{0.0} & \lembsr{100.0}{0.0} & \lembsr{100.0}{0.0} & \lembsr{100.0}{0.0} & \lembsr{100.0}{0.0} \\
     & GD & \multicolumn{9}{c}{\lemsr{65.3}{9.9}} \\
    \addlinespace[1pt]
    \multirow{5}{*}{LeWM + CGS} & MPPI & \lembsr{13.3}{1.1} & \lemsr{54.0}{2.0} & \lemsr{62.0}{11.1} & \lembsr{75.3}{9.4} & \lembsr{82.0}{2.0} & \lembsr{92.7}{6.1} & \lembsr{96.0}{2.0} & \lembsr{97.3}{1.1} & \lembsr{98.0}{2.0} \\
     & PG-MPPI & \lembsr{100.0}{0.0} & \lembsr{86.7}{6.1} & \lembsr{78.7}{7.6} & \lembsr{80.0}{4.0} & \lembsr{90.7}{2.3} & \lembsr{94.0}{3.5} & \lemsr{98.0}{3.5} & \lembsr{100.0}{0.0} & \lembsr{100.0}{0.0} \\
     & CEM & \lembsr{13.3}{1.1} & \lemsr{64.0}{6.0} & \lembsr{71.3}{8.1} & \lemsr{76.0}{5.3} & \lemsr{88.0}{2.0} & \lemsr{93.3}{2.3} & \lemsr{94.0}{2.0} & \lemsr{95.3}{1.1} & \lemsr{96.0}{2.0} \\
     & PG-CEM & \lembsr{100.0}{0.0} & \lemsr{98.0}{2.0} & \lemsr{98.7}{1.1} & \lembsr{100.0}{0.0} & \lembsr{100.0}{0.0} & \lembsr{100.0}{0.0} & \lembsr{100.0}{0.0} & \lembsr{100.0}{0.0} & \lembsr{100.0}{0.0} \\
     & GD & \multicolumn{9}{c}{\lembsr{70.7}{3.1}} \\
    \bottomrule
  \end{tabular}}
\end{table}

\begin{table}[H]
  \centering
  \caption{\textbf{Reacher sample-budget sweep.} Success rate (\%, mean $\pm$ sample SD over three evaluation seeds, 50 episodes each). Bold marks the best representation for each planner and budget, including ties. GD is independent of $K$.}
  \label{tab:lewm-full-sweep-reacher}
  \setlength{\tabcolsep}{2.6pt}
  \renewcommand{\arraystretch}{0.96}
  \resizebox{\textwidth}{!}{%
  \begin{tabular}{@{}ll*{9}{c}@{}}
    \toprule
    World model & Planner & $K=1$ & $K=2$ & $K=4$ & $K=8$ & $K=16$ & $K=32$ & $K=64$ & $K=128$ & $K=256$ \\
    \midrule
    \multirow{5}{*}{LeWM} & MPPI & \lembsr{2.0}{2.0} & \lemsr{46.0}{7.2} & \lembsr{40.7}{8.3} & \lembsr{44.7}{3.1} & \lemsr{36.7}{9.0} & \lemsr{36.7}{6.1} & \lemsr{49.3}{11.4} & \lemsr{65.3}{4.2} & \lembsr{73.3}{4.6} \\
     & PG-MPPI & \lembsr{78.7}{2.3} & \lembsr{67.3}{5.0} & \lembsr{49.3}{1.1} & \lemsr{46.0}{7.2} & \lemsr{42.7}{5.0} & \lemsr{46.0}{5.3} & \lemsr{46.7}{10.1} & \lemsr{62.7}{11.0} & \lembsr{72.0}{7.2} \\
     & CEM & \lembsr{2.0}{2.0} & \lembsr{44.0}{10.0} & \lembsr{41.3}{4.2} & \lemsr{65.3}{5.8} & \lembsr{72.0}{5.3} & \lemsr{74.7}{9.2} & \lembsr{83.3}{7.0} & \lembsr{83.3}{8.3} & \lemsr{85.3}{3.1} \\
     & PG-CEM & \lembsr{78.7}{2.3} & \lembsr{82.0}{2.0} & \lembsr{73.3}{6.1} & \lembsr{79.3}{4.2} & \lembsr{85.3}{3.1} & \lemsr{80.0}{3.5} & \lemsr{81.3}{9.4} & \lembsr{86.7}{6.1} & \lemsr{86.7}{10.1} \\
     & GD & \multicolumn{9}{c}{\lemsr{76.7}{4.2}} \\
    \addlinespace[1pt]
    \multirow{5}{*}{LeWM + TS} & MPPI & \lembsr{2.0}{2.0} & \lembsr{47.3}{8.1} & \lemsr{40.0}{10.6} & \lemsr{36.7}{9.9} & \lembsr{44.7}{4.2} & \lemsr{40.0}{9.2} & \lemsr{53.3}{6.1} & \lemsr{64.0}{2.0} & \lemsr{68.7}{6.4} \\
     & PG-MPPI & \lemsr{66.7}{11.0} & \lemsr{58.0}{6.0} & \lemsr{46.7}{9.9} & \lemsr{41.3}{7.6} & \lemsr{38.7}{9.0} & \lemsr{44.7}{3.1} & \lemsr{52.7}{5.0} & \lemsr{66.7}{2.3} & \lemsr{67.3}{7.6} \\
     & CEM & \lembsr{2.0}{2.0} & \lemsr{42.7}{5.0} & \lemsr{37.3}{1.1} & \lemsr{62.7}{3.1} & \lemsr{70.7}{3.1} & \lemsr{80.0}{8.7} & \lemsr{77.3}{5.0} & \lemsr{81.3}{8.1} & \lemsr{81.3}{7.0} \\
     & PG-CEM & \lemsr{66.7}{11.0} & \lemsr{74.7}{9.9} & \lemsr{66.7}{5.0} & \lemsr{78.7}{13.0} & \lemsr{80.0}{7.2} & \lemsr{79.3}{3.1} & \lembsr{82.0}{8.7} & \lemsr{78.0}{2.0} & \lemsr{84.7}{5.0} \\
     & GD & \multicolumn{9}{c}{\lemsr{81.3}{5.0}} \\
    \addlinespace[1pt]
    \multirow{5}{*}{LeWM + CGS} & MPPI & \lembsr{2.0}{2.0} & \lemsr{46.0}{2.0} & \lemsr{39.3}{3.1} & \lemsr{38.0}{13.1} & \lemsr{41.3}{4.2} & \lembsr{42.7}{12.2} & \lembsr{56.0}{10.0} & \lembsr{66.7}{2.3} & \lemsr{70.7}{2.3} \\
     & PG-MPPI & \lemsr{58.7}{6.1} & \lemsr{64.7}{1.1} & \lemsr{47.3}{7.0} & \lembsr{46.7}{7.0} & \lembsr{46.0}{2.0} & \lembsr{46.7}{4.2} & \lembsr{53.3}{12.9} & \lembsr{69.3}{3.1} & \lemsr{67.3}{7.0} \\
     & CEM & \lembsr{2.0}{2.0} & \lemsr{39.3}{1.1} & \lemsr{38.7}{9.4} & \lembsr{66.7}{3.1} & \lemsr{71.3}{15.0} & \lembsr{81.3}{4.6} & \lemsr{82.0}{2.0} & \lemsr{80.7}{3.1} & \lembsr{86.7}{4.2} \\
     & PG-CEM & \lemsr{58.7}{6.1} & \lemsr{66.7}{7.0} & \lemsr{60.0}{5.3} & \lemsr{74.0}{6.0} & \lemsr{80.7}{6.4} & \lembsr{80.7}{2.3} & \lemsr{79.3}{10.1} & \lemsr{81.3}{3.1} & \lembsr{88.0}{6.9} \\
     & GD & \multicolumn{9}{c}{\lembsr{82.0}{5.3}} \\
    \bottomrule
  \end{tabular}}
\end{table}

\paragraph{Aggregate performance across sample budgets.}
In order to comprehensively compare planning performance across sample-budget sweep result, we use an integration metric to summarize success rate all across sample-budget $K$'s. Normalized AUC (nAUC)
uses trapezoidal integration divided by the horizontal span:
\begin{equation*}
\mathrm{nAUC}_x=\frac{\sum_{j=1}^{m-1}(x_{j+1}-x_j)(s_j+s_{j+1})/2}{x_m-x_1}.
\end{equation*}
Here $s_j$ is SR in percent. We use $x=\log_2 K$ for sample budgets,
nAUC ranges from 0 to 100. In Appendix~\ref{app:planner-iteration-curves} and \ref{app:pusht-temperature-tuning}, we also show nAUC results for refinement SR curve and temperature selection results, with
$x=I$ being planning optimization steps, and $x=\log_2\tau$ being temperature. Integration is linear, so integrating the
mean curve equals averaging per-seed integrals. Updated Reacher entries
use the reported rounded mean curves.

\begin{table}[!htbp]
\centering
\vspace*{\baselineskip}
\caption{\textbf{Complete sample-budget sweep nAUC.}
$K\in\{1,2,4,8,16,32,64,128,256\}$; sampling planners use $I=30$.
nAUC uses $\log_2 K$. Parentheses give percentage-point differences
from LeWM for the same environment and planner. Bold marks the best
representation for each planner and environment. Results average three
evaluation seeds (50 episodes each); GD repeats its $K$-independent SR.}
\label{tab:app-full-sweep-curve-summary}
\begingroup
\fontsize{8}{9.5}\selectfont
\setlength{\tabcolsep}{3pt}
\renewcommand{\arraystretch}{1.06}
\begin{tabular*}{\linewidth}{@{\extracolsep{\fill}}llrrrr@{}}
\toprule
Planner & Model & PushT & Cube & TwoRooms & Reacher \\
\midrule
CEM & LeWM & $62.79$ & $56.54$ & $64.79$ & $\mathbf{63.46}$ \\
 & LeWM + TS & $62.17\,{\scriptstyle(-0.62)}$ & $59.96\,{\scriptstyle(+3.42)}$ & $\mathbf{80.96}\,{\scriptstyle(+16.17)}$ & $61.71\,{\scriptstyle(-1.75)}$ \\
 & LeWM + CGS & $\mathbf{67.75}\,{\scriptstyle(+4.96)}$ & $\mathbf{62.58}\,{\scriptstyle(+6.04)}$ & $79.58\,{\scriptstyle(+14.79)}$ & $63.04\,{\scriptstyle(-0.42)}$ \\
\midrule
PG-CEM & LeWM & $89.17$ & $91.25$ & $96.75$ & $\mathbf{81.33}$ \\
 & LeWM + TS & $90.38\,{\scriptstyle(+1.21)}$ & $89.88\,{\scriptstyle(-1.38)}$ & $\mathbf{99.83}\,{\scriptstyle(+3.08)}$ & $76.88\,{\scriptstyle(-4.46)}$ \\
 & LeWM + CGS & $\mathbf{93.17}\,{\scriptstyle(+4.00)}$ & $\mathbf{94.46}\,{\scriptstyle(+3.21)}$ & $99.58\,{\scriptstyle(+2.83)}$ & $74.50\,{\scriptstyle(-6.83)}$ \\
\midrule
MPPI & LeWM & $50.46$ & $51.29$ & $67.67$ & $44.63$ \\
 & LeWM + TS & $50.00\,{\scriptstyle(-0.46)}$ & $56.08\,{\scriptstyle(+4.79)}$ & $76.42\,{\scriptstyle(+8.75)}$ & $45.17\,{\scriptstyle(+0.54)}$ \\
 & LeWM + CGS & $\mathbf{59.46}\,{\scriptstyle(+9.00)}$ & $\mathbf{58.88}\,{\scriptstyle(+7.58)}$ & $\mathbf{76.88}\,{\scriptstyle(+9.21)}$ & $\mathbf{45.79}\,{\scriptstyle(+1.17)}$ \\
\midrule
PG-MPPI & LeWM & $76.83$ & $76.38$ & $87.58$ & $54.50$ \\
 & LeWM + TS & $77.42\,{\scriptstyle(+0.58)}$ & $78.54\,{\scriptstyle(+2.17)}$ & $86.88\,{\scriptstyle(-0.71)}$ & $51.96\,{\scriptstyle(-2.54)}$ \\
 & LeWM + CGS & $\mathbf{85.92}\,{\scriptstyle(+9.08)}$ & $\mathbf{79.42}\,{\scriptstyle(+3.04)}$ & $\mathbf{91.00}\,{\scriptstyle(+3.42)}$ & $\mathbf{54.62}\,{\scriptstyle(+0.12)}$ \\
\midrule
GD & LeWM & $85.33$ & $59.33$ & $41.33$ & $76.67$ \\
 & LeWM + TS & $82.67\,{\scriptstyle(-2.67)}$ & $65.33\,{\scriptstyle(+6.00)}$ & $65.33\,{\scriptstyle(+24.00)}$ & $81.33\,{\scriptstyle(+4.67)}$ \\
 & LeWM + CGS & $\mathbf{89.33}\,{\scriptstyle(+4.00)}$ & $\mathbf{71.33}\,{\scriptstyle(+12.00)}$ & $\mathbf{70.67}\,{\scriptstyle(+29.33)}$ & $\mathbf{82.00}\,{\scriptstyle(+5.33)}$ \\
\bottomrule
\end{tabular*}
\endgroup
\end{table}

\subsubsection{Cube ablation of reference geometry in CGS.}
\label{app:cube-ref-ablation}
Each Cube environment action is five-dimensional,
consisting of XYZ end-effector displacement, a grasp command, and an arm-angle
command. Because one model step concatenates five environment actions, the
predictor receives a $25$D action block, while the corresponding XYZ
translation channels form a $15$D reference. We compare three CGS reference
geometries: \emph{grasp-active XYZ} (15D), which uses only transitions in which the grasp command is active; \emph{all-transition XYZ} (15D), which uses the same
translation channels on all transitions; and \emph{all-transition full-action}
(25D), which uses all five action coordinates. In the main context
(Section~\ref{sec:main-experimental-design}), we use grasp-active XYZ because
XYZ commands most directly specify cube motion while the cube is grasped,
whereas ungrasped transitions and the grasp/arm-angle channels need not induce
comparable object displacement. The full $25$D action block is nevertheless
retained as predictor input in all variants. The sample-budget sweep uses $I=30$;
the refinement sweep uses $K=128$. MPPI uses $\tau=64$ throughout.

\begin{figure}[H]
\centering
% One float: no caption or independent table float between the image and table.
\includegraphics[width=\linewidth,trim=0 7 0 0,clip]{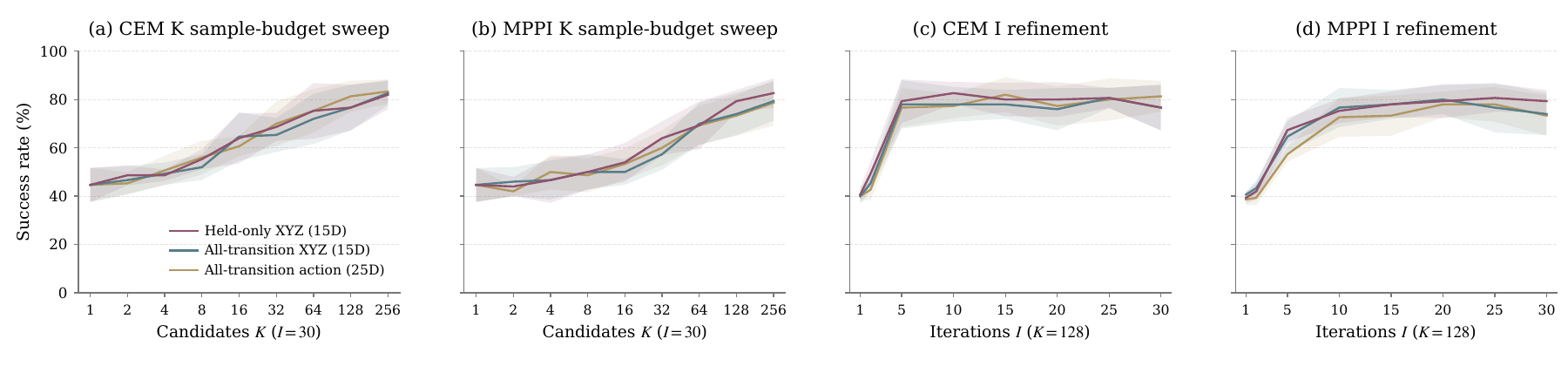}\par
\vspace{-2pt}
\begingroup
\fontsize{8}{9.5}\selectfont
\setlength{\tabcolsep}{3pt}
\renewcommand{\arraystretch}{1.06}
\begin{tabular*}{\linewidth}{@{\extracolsep{\fill}}lrrrr@{}}
\toprule
Reference & CEM ($K$) & MPPI ($K$) & CEM ($I$) & MPPI ($I$) \\
\midrule
grasp-active XYZ15D (main) & $62.58$ & $\mathbf{58.88}$ & $\mathbf{77.40}$ & $\mathbf{73.72}$ \\
All-transition XYZ15D & $61.29\,{\scriptstyle(-1.29)}$ & $57.00\,{\scriptstyle(-1.88)}$ & $75.09\,{\scriptstyle(-2.31)}$ & $72.67\,{\scriptstyle(-1.06)}$ \\
All-transition full25D & $\mathbf{62.92}\,{\scriptstyle(+0.33)}$ & $57.29\,{\scriptstyle(-1.58)}$ & $75.82\,{\scriptstyle(-1.59)}$ & $69.68\,{\scriptstyle(-4.05)}$ \\
\bottomrule
\end{tabular*}
\endgroup
\caption{\textbf{Cube reference geometry: sample-budget and refinement curves with nAUC summaries.}
Lines show mean success and bands show sample SD over three evaluation
seeds (50 episodes each). The table reports mean normalized AUC (nAUC)
over the complete sample-budget and refinement sweeps: $K=1$--$256$ on a $\log_2 K$ axis
and $I=1$--$30$ on a linear $I$ axis. Higher is better; bold marks
the best reference for each planner and sweep, including ties.
Parentheses give percentage-point differences from grasp-active XYZ15D
for the same sweep and planner; negative values indicate decreases.
Both planners use no action prior.}
\label{fig:app-cube-reference-curve-summary}
\end{figure}

grasp-active XYZ has the highest nAUC for MPPI in
sweeps and for CEM in the refinement sweep. The all-transition full-action reference has the highest
nAUC for CEM in the sample-budget sweep.
On Cube, the grasp-active mask and XYZ reference are derived entirely from the observed action vector and require no additional state or task labels. Our full-action ablation shows that CGS remains effective without this reference selection.

\subsubsection{State-informed reference geometry and observation history on Reacher}
\label{app:reacher-ref-history}

We separate two factors that can limit an action-only reference on Reacher:
raw torque geometry does not fully describe the state-dependent physical effect
of an action, and a single image does not fully specify the dynamical state.
We compare LeWM, the main CGS model using the action reference (action-CGS),
and a transition-reference CGS variant (transition-CGS). Transition-CGS uses
a state-informed reference constructed from observed joint-configuration
displacements, capturing the realized transition under the executed action and
current state. Each checkpoint is frozen and evaluated under the same three
observation-history protocols.

\paragraph{Transition-reference training.}
Action-CGS uses the executed action block $a_t$ to define reference geometry.
Transition-CGS instead uses the observed joint-configuration displacement,
which reflects the transition produced by that action in its starting state:
\begin{equation}
 \psi(q_t)=[\sin q_{t,1},\sin q_{t,2},\cos q_{t,1},\cos q_{t,2}]^\top,
 \qquad r_t=W[\psi(q_{t+1})-\psi(q_t)],
 \label{eq:reacher-transition-reference}
\end{equation}
where $q_t$ denotes the two joint angles, i.e., the configuration component
of the Reacher probe state $s_t$. The index $t$ follows model steps, so
consecutive model-step observations are separated by five environment steps.
$W$ is a fixed diagonal normalization computed from training displacements;
the sine and cosine coordinates of each joint share an inverse RMS scale.
The reference $r_t$ therefore reflects the realized state-dependent transition
rather than the executed torque vector alone, with no learned aggregation head
or encoder teacher. Only the reference construction changes: the architecture,
prediction and SIGReg losses, the default $N=3$ history setting, the cosine-Gram
CGS objective, and all other training settings are unchanged. Joint-angle labels
are used only to construct the training reference; planning uses images and
actions without access to these labels.

\paragraph{Observation-history protocols.}
\emph{Single-frame context (native)} supplies only the current image at every
replanning call, with no previous observations or executed actions. Predicted
latents and candidate actions still accumulate causally within each imagined
rollout; we use this historical strategy in most of our experiments unless explained.

\emph{Dataset-initialized history (privileged)} supplies
$[o_{t-2},o_{t-1},o_t]$ and the two connecting executed-action blocks,
$[a_{t-2},a_{t-1}]$. At the first solve, the preceding observations and actions
come from the same dataset episode used to initialize the environment. Missing
episode-start observations are padded with the first frame and missing actions
with raw zero actions before normalization; history never crosses episode
boundaries. All later history comes from the actual online trajectory. These
initial observations and actions are unavailable to a planner starting from an
arbitrary reset with only one image, so we treat dataset-initialized history as
a privileged-context diagnostic rather than a deployable cold-start protocol.

\emph{Online history} reads no dataset prehistory and starts through the native
single-frame path. It records observations and executed actions, enabling the
same three-frame context after at least two model steps have been collected. In
our evaluation, each solve executes five action blocks, so the second solve uses
$[o_3,o_4,o_5]$ and $[a_3,a_4]$, with time measured from the reset. Both the
native and online-history protocols are feasible under the single-image
cold-start setting. In either history-based protocol, past actions are aligned
with the observed transitions and candidate actions begin at the current time;
the predictor retains its causal mask and at most three context tokens during
rollout. No model parameters are updated online.

\paragraph{Evaluation and results.}
We follow the shared sampling-based and GD evaluation protocols in
Appendices~\ref{app:planning-datasets}, using matched
evaluation starts for seeds $\{0,1,42\}$ and 50 episodes per seed.
Sampling-based evaluations use no action prior. Across this ablation, only the
CGS reference and observation-history protocol vary.

\begin{figure}[H]
\centering
\includegraphics[width=\linewidth]{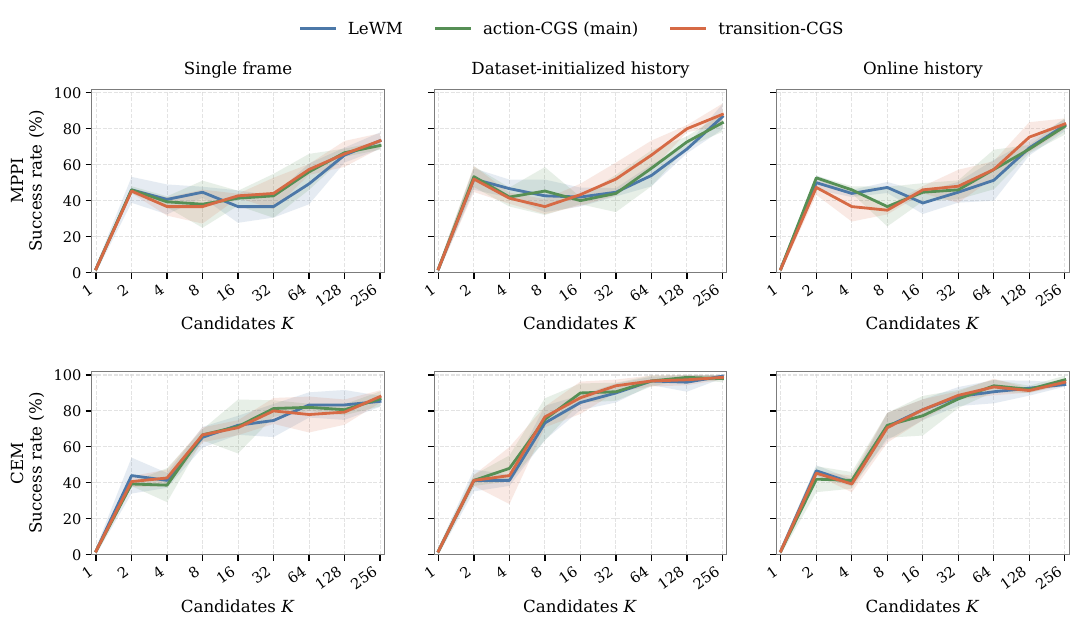}
\caption{\textbf{Reacher sample-budget sweeps under three observation-history protocols.}
Top: MPPI; bottom: CEM. Columns show single-frame input, dataset-initialized
history, and online history, respectively.
All curves use $I=30$ without an action prior; bands show $\pm$ one sample SD
across three evaluation seeds, 50 episodes each. MPPI uses $\tau=64$.
State labels are used only to train
transition-CGS.}
\label{fig:reacher-ref-full-sweep}
\end{figure}

Figure~\ref{fig:reacher-ref-full-sweep} and
Table~\ref{tab:reacher-ref-history} separate the effects of reference
geometry from those of observation history. For MPPI, the three models are
similar under the native single-frame protocol: at $K=128$, LeWM,
action-CGS, and transition-CGS reach $65.3\%$, $66.7\%$, and $66.0\%$,
respectively. The state-informed reference becomes more useful when recent
observation--action history provides dynamical context. With
dataset-initialized history, transition-CGS reaches $80.0\%$ at $K=128$,
compared with $68.7\%$ for LeWM and $72.7\%$ for action-CGS, and also gives
the highest sample-budget nAUC ($51.96$ versus $49.38$ and $49.75$).
With online history, its advantage becomes clearer at larger candidate
budgets, reaching $75.3\%$ at $K=128$, compared with $69.3\%$ for LeWM
and $68.7\%$ for action-CGS. Thus, for MPPI, the benefit of a
state-dependent transition reference becomes clearer when short transition
history makes the current dynamics more observable, particularly at larger
sampling budgets.

CEM shows a different pattern. Adding history substantially improves the
LeWM baseline itself: its $K=128$ success rises from $83.3\%$ with a
single frame to $96.0\%$ with dataset-initialized history and $92.7\%$
with online history. With history available, CEM approaches saturation at
larger candidate budgets (Figure~\ref{fig:reacher-ref-full-sweep}), leaving
limited headroom for either reference variant to provide further gains.
Accordingly, the additional differences among LeWM, action-CGS, and
transition-CGS are small and mixed in both $K=128$ success and
sample-budget nAUC.
GD exhibits a more consistent reference effect. Action-CGS improves over
LeWM in all three history settings, and transition-CGS further increases
mean success to $85.3\%$, $94.0\%$, and $92.7\%$ under single-frame,
dataset-initialized, and online history, respectively, compared with
$76.7\%$, $84.0\%$, and $81.3\%$ for LeWM.

Together, these results suggest that observation history and reference
geometry address distinct limitations on Reacher. History provides
information about recent dynamics that is absent from a single image,
whereas transition-CGS supervises latent geometry using the realized,
state-dependent effect of an action. Their combination is especially useful
for MPPI and GD, while CEM already benefits strongly from history alone.
Transition-CGS, however, requires joint-state information to construct its
training reference. Such labels are privileged and may be unavailable in
the pixel--action datasets targeted by our main setting. We therefore retain
action-CGS as the main method and use transition-CGS as a diagnostic of the
potential benefit of richer state-dependent reference geometries.

% Native action-CGS CEM nAUC delta uses the original unrounded sample-budget aggregate (-0.42).
\begin{table}[H]
\centering
\fontsize{8}{9.5}\selectfont
\vspace*{\baselineskip}
\caption{\textbf{Reacher reference and history ablation: success rate and sample-budget nAUC.} SR is mean $\pm$ sample SD (\%) over three evaluation seeds, 50 episodes each. MPPI/CEM SR uses $K=128$, $I=30$; GD uses 100 Adam steps and is independent of $K$. nAUC summarizes the complete sample-budget sweep. Parentheses give percentage-point differences from LeWM for both CGS models, within the same history, planner, and metric. Bold marks the highest value within each history/metric group. MPPI uses $\tau=64$.}
\label{tab:reacher-ref-history}
\setlength{\tabcolsep}{1.5pt}
\begin{tabular*}{\linewidth}{@{\extracolsep{\fill}}>{\raggedright\arraybackslash}p{0.14\linewidth}lccccc@{}}
\toprule
History & Model & \multicolumn{2}{c}{MPPI} & \multicolumn{2}{c}{CEM} & GD \\
\cmidrule(lr){3-4}\cmidrule(lr){5-6}\cmidrule(l){7-7}
 & & SR & nAUC & SR & nAUC & SR \\
\midrule
\multirow{3}{=}{Single frame} & LeWM & $65.3\,{\scriptstyle\pm 4.2}$ & $44.63$ & $\mathbf{83.3}\,{\scriptstyle\pm 8.3}$ & $\mathbf{63.46}$ & $76.7\,{\scriptstyle\pm 4.2}$ \\
 & action-CGS & $\mathbf{66.7}\,{\scriptstyle\pm 2.3}\,{\scriptstyle(+1.3)}$ & $\mathbf{45.79}\,{\scriptstyle(+1.17)}$ & $80.7\,{\scriptstyle\pm 3.1}\,{\scriptstyle(-2.7)}$ & $63.04\,{\scriptstyle(-0.42)}$ & $82.0\,{\scriptstyle\pm 5.3}\,{\scriptstyle(+5.3)}$ \\
 & transition-CGS & $66.0\,{\scriptstyle\pm 7.2}\,{\scriptstyle(+0.7)}$ & $\mathbf{45.79}\,{\scriptstyle(+1.17)}$ & $79.3\,{\scriptstyle\pm 7.0}\,{\scriptstyle(-4.0)}$ & $62.88\,{\scriptstyle(-0.58)}$ & $\mathbf{85.3}\,{\scriptstyle\pm 7.0}\,{\scriptstyle(+8.7)}$ \\
\midrule
\multirow{3}{=}{\raggedright\scriptsize\mbox{Dataset-initialized}\newline history} & LeWM & $68.7\,{\scriptstyle\pm 2.3}$ & $49.38$ & $96.0\,{\scriptstyle\pm 5.3}$ & $71.75$ & $84.0\,{\scriptstyle\pm 3.5}$ \\
 & action-CGS & $72.7\,{\scriptstyle\pm 3.1}\,{\scriptstyle(+4.0)}$ & $49.75\,{\scriptstyle(+0.37)}$ & $\mathbf{98.7}\,{\scriptstyle\pm 1.1}\,{\scriptstyle(+2.7)}$ & $\mathbf{73.83}\,{\scriptstyle(+2.08)}$ & $90.7\,{\scriptstyle\pm 4.2}\,{\scriptstyle(+6.7)}$ \\
 & transition-CGS & $\mathbf{80.0}\,{\scriptstyle\pm 2.0}\,{\scriptstyle(+11.3)}$ & $\mathbf{51.96}\,{\scriptstyle(+2.58)}$ & $97.3\,{\scriptstyle\pm 3.1}\,{\scriptstyle(+1.3)}$ & $73.46\,{\scriptstyle(+1.71)}$ & $\mathbf{94.0}\,{\scriptstyle\pm 3.5}\,{\scriptstyle(+10.0)}$ \\
\midrule
\multirow{3}{=}{Online history} & LeWM & $69.3\,{\scriptstyle\pm 4.2}$ & $48.42$ & $\mathbf{92.7}\,{\scriptstyle\pm 4.2}$ & $\mathbf{69.79}$ & $81.3\,{\scriptstyle\pm 2.3}$ \\
 & action-CGS & $68.7\,{\scriptstyle\pm 2.3}\,{\scriptstyle(-0.7)}$ & $\mathbf{49.21}\,{\scriptstyle(+0.79)}$ & $92.0\,{\scriptstyle\pm 2.0}\,{\scriptstyle(-0.7)}$ & $69.37\,{\scriptstyle(-0.42)}$ & $89.3\,{\scriptstyle\pm 4.2}\,{\scriptstyle(+8.0)}$ \\
 & transition-CGS & $\mathbf{75.3}\,{\scriptstyle\pm 8.3}\,{\scriptstyle(+6.0)}$ & $48.46\,{\scriptstyle(+0.04)}$ & $91.3\,{\scriptstyle\pm 1.1}\,{\scriptstyle(-1.3)}$ & $\mathbf{69.79}\,{\scriptstyle(+0.00)}$ & $\mathbf{92.7}\,{\scriptstyle\pm 5.0}\,{\scriptstyle(+11.3)}$ \\
\bottomrule
\end{tabular*}
\end{table}

\subsection{Visualizing Action--Effect Geometry and Planning}
\label{app:action-effect-geometry}
\label{app:trajectory-heatmaps}

Figure~\ref{fig:appendix-action-effect-geometry} shows two selected examples
per environment, complementing the planning results with two views of
control geometry: encoded transition trajectories and predicted endpoint
responses to action perturbations.

\subsubsection{Action and latent transition trajectories}
\label{app:geometry-trajectories}

\paragraph{Construction.}
For PushT, Cube, and Reacher, we sample candidate starts from recorded
trajectories and select two illustrative 30-observation windows per benchmark
at a five-step stride, retaining all 29 transitions. Because recorded TwoRooms
episodes are shorter than this 145-step span, we use seeded native rollouts.
Its two examples share a zigzag control family; a local encoder-geometry probe
selects initial regions, and seeds determine control orientation and size.
All are selected qualitative examples, not unfiltered random samples.

For each window, we form standardized five-step action blocks $a_t$ and
latent differences $\Delta z_t=f_\theta(o_{t+1})-f_\theta(o_t)$, then
normalize each vector in its original space as in
Eq.~\eqref{eq:cgs-relational-grams}, and cumulatively sum the unit directions
into action and latent paths. We align latent to action directions (using a
rigid cumulative-path fit for Cube) and project both paths onto action-path
principal components: two for PushT, TwoRooms, and Reacher and three for Cube.
Cube uses the XYZ action channels. The basis is shared across models;
alignment has no scale, and projected steps are not renormalized.

\paragraph{Interpretation and results.}
The purple dashed curve is the action reference; colored curves are encoded
transition paths. The key criterion is preservation of turns and relative
directions: under CGS, $K_Z\approx K_A$ means that similarly directed actions
induce similarly directed latent changes. Because the objective matches
pairwise angles rather than coordinates or step lengths, exact projected
overlap is a stronger criterion than the training loss imposes.

On \textbf{PushT and Cube}, CGS follows the action-reference shape more
faithfully than LeWM or TS, particularly around changes of direction.
TS is partially improved or comparable to LeWM but does not consistently
recover these turns. On simpler \textbf{TwoRooms}, TS and CGS are similar and
both align much better than LeWM, consistent with their near-ceiling MPPI SR
($96.0\%$ and $98.0\%$ at $K=256$; Table~\ref{tab:lewm-full-sweep-tworoom}).
On Reacher, CGS preserves several turns more clearly than LeWM or TS, though
less consistently than on PushT or Cube, suggesting partial improvement in
the action-related organization of this more state-dependent setting.

\subsubsection{Action-perturbation heatmaps and planning isotropy}
\label{app:geometry-heatmaps}

\paragraph{Construction and connection to $G_H$.}
At a shared observation and base action sequence $\mathbf a$, let
$\widehat F_H(\mathbf a)$ be the model's predicted terminal latent for $H=5$
action blocks. Let $Q\in\mathbb R^{(H d_a)\times2}$ have orthonormal columns
spanning the fixed horizon-constant plane of the first two primitive
control coordinates. We sample 32 equally spaced unit-circle directions
$q_j=(\cos\theta_j,\sin\theta_j)^\top$, with $\theta_j=2\pi j/32$, and
apply equal-norm perturbations $\epsilon Qq_j$ in standardized action
coordinates, using $\epsilon=0.05$. The heatmap records pairwise cosines
of the full-dimensional terminal responses
\begin{equation}
 r_j=\widehat F_H(\mathbf a+\epsilon Qq_j)-\widehat F_H(\mathbf a),
 \qquad K^{\rm resp}_{jk}=\frac{r_j^\top r_k}{\|r_j\|_2\,\|r_k\|_2}.
 \label{eq:appendix-response-gram}
\end{equation}
The first PushT example uses frame 13 as its response anchor; every other
example uses frame 1. Each comparison uses the same anchor, base actions,
and perturbations for all three models.

Writing $J_H=\partial\widehat F_H/\partial\mathbf a$, local linearization gives
$r_j\approx\epsilon J_HQq_j$. Consequently,
\begin{equation}
 K^{\rm resp}_{jk}\approx
 \frac{q_j^\top M_Qq_k}
 {\sqrt{q_j^\top M_Qq_j}\sqrt{q_k^\top M_Qq_k}},
 \qquad M_Q=Q^\top J_H^\top J_HQ.
 \label{eq:appendix-plane-geometry}
\end{equation}
For the linear dynamics in Section~\ref{sec:cgs_main_planner},
$J_H=\Gamma_H$ and hence $M_Q=Q^\top G_HQ$; for a nonlinear predictor,
$J_H^\top J_H$ is the Gauss--Newton component of the planning Hessian
(Eq.~\eqref{eq:app_nonlinear_cost_hessian}). Under ideal CGS geometry,
$G_H=cHP_H$ and $P_HQ=Q$, giving $M_Q=cH I_2$. Equal curvature in the
tested plane then preserves the circle's angular relations:
$K^{\rm resp}_{jk}\approx q_j^\top q_k=\cos(\theta_j-\theta_k)$.
Thus similarity to the ideal heatmap is a local diagnostic of planning
isotropy in this plane, rather than a measurement of full-space isotropy
or of the nonlinear residual Hessian.

\paragraph{Interpretation and results.}
The \emph{Ideal} tile is the cosine Gram matrix of a standard unit circle:
nearby directions have cosine close to $1$, orthogonal directions have
cosine $0$, and opposite directions have cosine $-1$. Its regular diagonal
bands provide the reference for every model, with a common color range
$[-1,1]$. On \textbf{PushT and Cube}, CGS retains these bands more faithfully
than LeWM and TS, whose distorted bands reveal uneven angular sensitivity.
On \textbf{TwoRooms}, TS and CGS are close to one another and closer to the
circle reference than LeWM, consistent with their similar trajectory
geometry and near-ceiling SR. On Reacher, the response patterns remain less regular than on PushT or
Cube, but the displayed CGS heatmaps recover the ideal band structure
more closely than LeWM or TS. This indicates partial improvement in local
planning sensitivity despite the stronger state dependence of the task,
consistent with the smaller and more planner-dependent gains observed in
the planning results.

\clearpage
\begin{figure}[H]
  \centering
  \includegraphics[width=\linewidth,height=0.86\textheight,keepaspectratio]{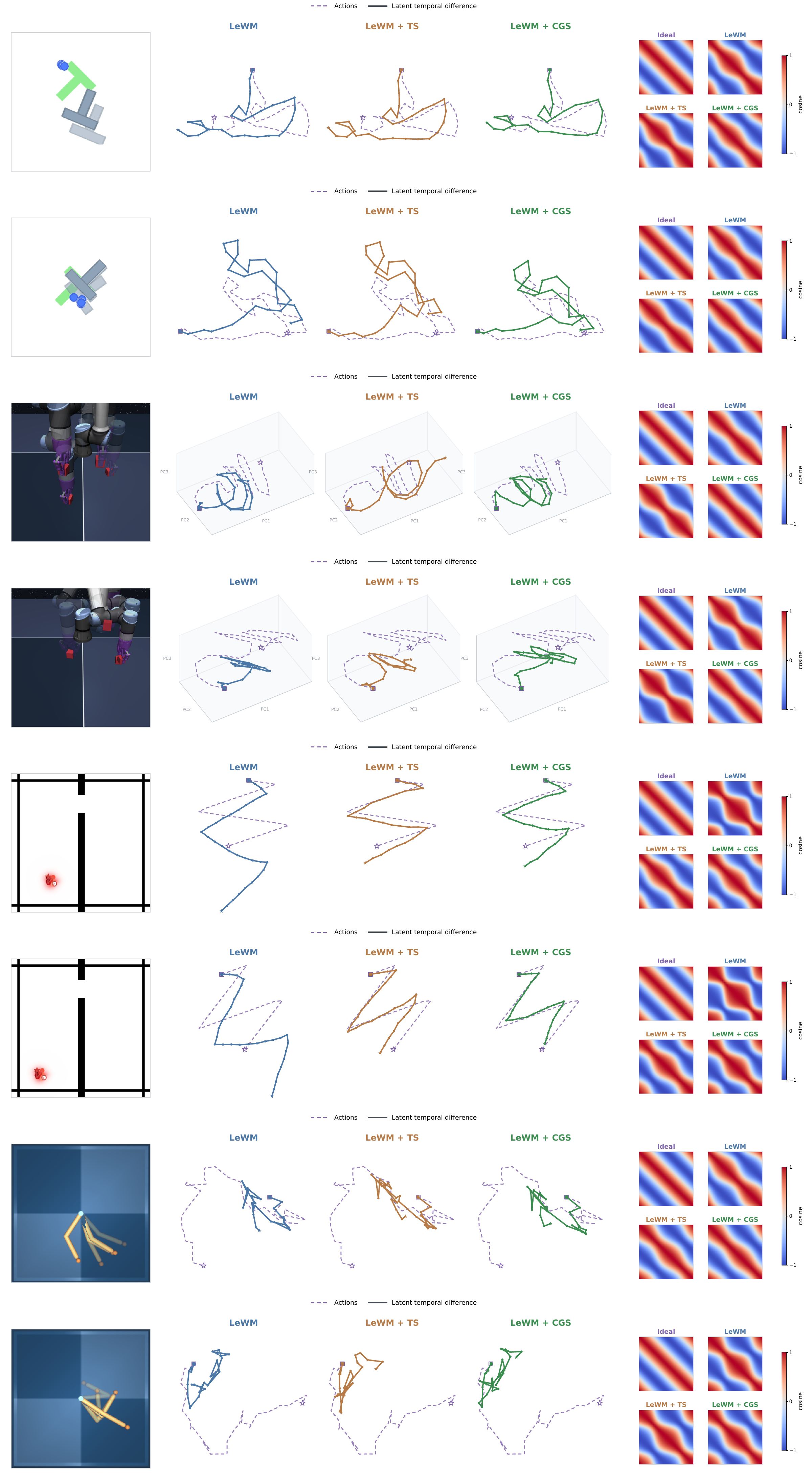}
  \caption{\textbf{Action--effect geometry across four environments.}
  From top to bottom: two examples each of PushT, Cube, TwoRooms, and
  Reacher. \textbf{Left:} environment illustrations.
  \textbf{Middle:} complete 30-frame action references (purple dashed) and
  encoded transition paths for LeWM (blue), TS (brown), and CGS (green),
  with shared action PCA in 2D or 3D; squares and stars mark starts and ends.
  Shape agreement indicates preservation of action-related directions.
  \textbf{Right:} cosine Grams of predicted endpoint responses to circular
  action perturbations; \emph{Ideal} is the unit-circle reference.
  Protocol and interpretation are given in
  Appendix~\ref{app:action-effect-geometry}.}
  \label{fig:appendix-action-effect-geometry}
  \label{fig:appendix-geometry}
\end{figure}
\clearpage

\subsubsection{Qualitative PushT Planning Rollout}
\label{app:qualitative-pusht-rollout}

Figure~\ref{fig:pusht-qualitative-rollout} shows a PushT planning
example under LeWM and LeWM+CGS. With LeWM, the planned trajectory fails to
reach the target configuration, whereas CGS produces a successful rollout. 
The corresponding decoded trajectories (decoder trained a
posteriori) remain qualitatively consistent with
the simulator evolution, indicating that the CGS representation retains
sufficient task-relevant information to support successful planning in this
example. 

\begin{figure*}[t]
  \centering
  \includegraphics[width=\textwidth]
  {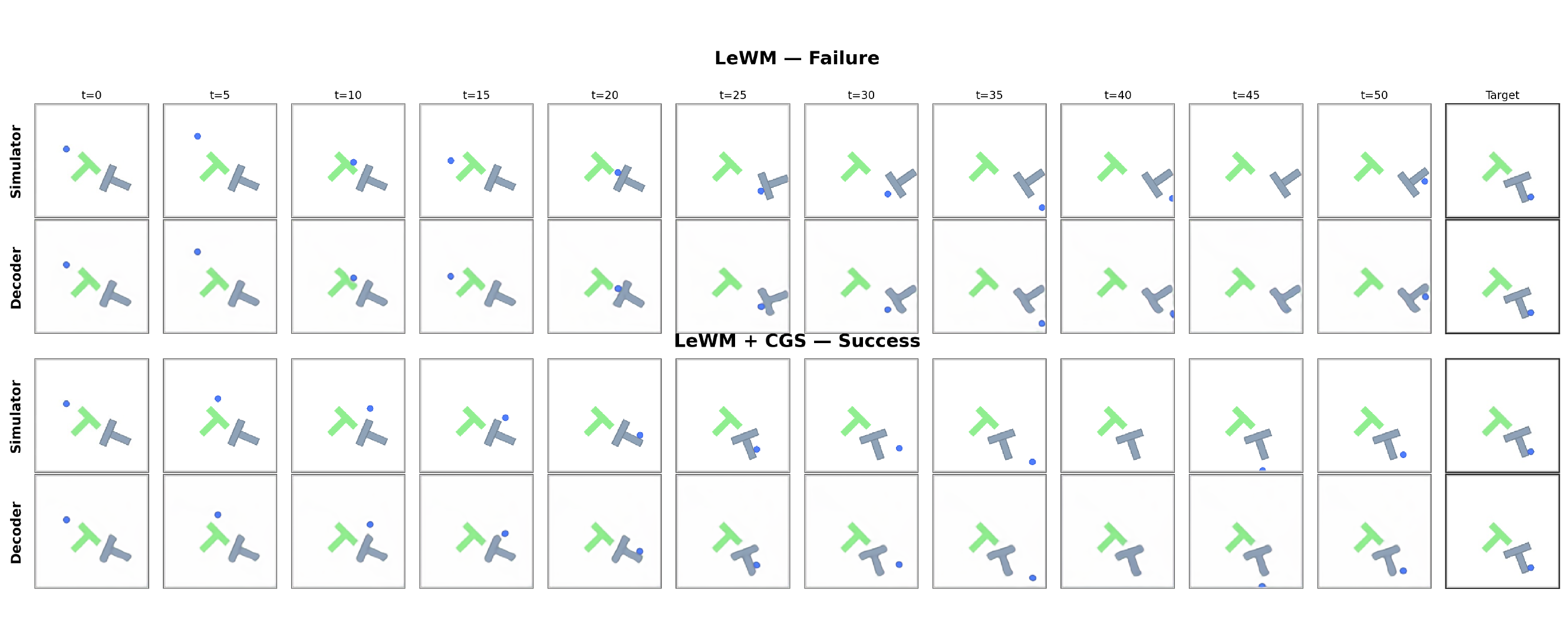}
  \caption{\textbf{Selected qualitative PushT planning rollout.}
  Simulator and decoded trajectories for a LeWM failure (top) and a
  LeWM+CGS success (bottom). Columns show the rollout over environment
  steps with the target configuration. The CGS example reaches
  the target while its decoded trajectory remains qualitatively aligned
  with the simulator rollout, illustrating that the learned representation
  retains sufficient task-relevant information for successful planning in
  this example.}
  \label{fig:pusht-qualitative-rollout}
\end{figure*}

\subsection{Full Action--Representation Probing}
\label{app:probing}

\paragraph{Protocol and probe definitions.}
We use 200 episodes per environment, split by episode into 140 fit,
20 validation, and 40 test episodes with split seed 20260814.
The probe time index follows the frame-skipped sequences used by the world
model, with frame skip 5. Thus, consecutive indexed observations
$o_t$ and $o_{t+1}$ are separated by five raw data frames, and
$a_t\in\mathbb{R}^{d_a}$ denotes the concatenated action vector associated
with that interval; $d_a$ is the resulting action dimension after
frame-skip concatenation. We define
\begin{equation}
    \Delta z_t = z_{t+1}-z_t,
    \qquad
    \Delta s_t = s_{t+1}-s_t,
    \label{eq:probe-deltas}
\end{equation}
with wrapped differences for angular targets.

Here $s_t$ denotes the physical-state features supplied to the probe,
rather than necessarily the complete simulator state. State-conditioned
probes use agent and object pose on PushT, joint positions and velocities
on Cube and Reacher, and agent position on TwoRooms. Encoders are frozen,
and $z_t$ is the global encoder token after the learned projector,
evaluated on real dataset images. Linear probes are ridge regressions with
coefficient $10^{-4}$. Nonlinear probes are two-layer width-128 GELU MLPs
selected by validation loss; reported MLP values use the ensemble prediction
from seeds 7, 17, and 27. Input and target normalizers are fit on the fit
episodes only. Unless otherwise stated, scores are pooled held-out
\mbox{$R^2=1-\mathrm{SSE}/\mathrm{SST}$}, where
$\mathrm{SSE}=\sum_i \|y_i-\hat y_i\|_2^2$ is the sum of squared prediction
errors and
$\mathrm{SST}=\sum_i \|y_i-\bar y\|_2^2$ is the total sum of squares around
the held-out target mean $\bar y$.

\paragraph{Physical identification.}
We first ask how predictable the physical displacement is from action
alone, state alone, or their combination. Let
\begin{equation}
\begin{aligned}
    \widehat{\Delta s}^{\,g}_{A,t}
        &= p^{g}_{A}(a_t), \\
    \widehat{\Delta s}^{\,g}_{S,t}
        &= p^{g}_{S}(s_t), \\
    \widehat{\Delta s}^{\,g}_{AS,t}
        &= p^{g}_{AS}(a_t,s_t),
\end{aligned}  
\end{equation}
where $g\in\{\mathrm{lin},\mathrm{MLP}\}$ denotes the probe class.
We denote their held-out $R^2$ scores by
$P_A^{g}$, $P_S^{g}$, and $P_{AS}^{g}$, respectively.
These are environment-level diagnostics and are therefore shared across
representations.

The partial action score
\begin{equation}
    R^2_{A\mid S}
    =
    1-
    \frac{
        \mathrm{SSE}\!\left(P_{AS}^{\mathrm{MLP}}\right)
    }{
        \mathrm{SSE}\!\left(P_{S}^{\mathrm{MLP}}\right)
    }
    \label{eq:partial-action-score}
\end{equation}
measures the additional predictable effect supplied by action after
conditioning on the state features. These quantities are empirical
held-out predictability diagnostics under the specified probe classes,
rather than Bayes-optimal physical ceilings.

\paragraph{Motion decoding.}
For each frozen representation, we next fit decoders
\begin{equation}
    d_{\mathrm{lin}}(\Delta z_t),
    \qquad
    d_{\mathrm{MLP}}(\Delta z_t)
\end{equation}
to predict the true physical displacement $\Delta s_t$.
Their held-out $R^2$ scores are denoted
$D_{\mathrm{lin}}$ and $D_{\mathrm{MLP}}$.
Thus $D_{\mathrm{MLP}}$ tests whether the evaluated motion information
is accessible to the nonlinear probe, whereas $D_{\mathrm{lin}}$
measures its linear accessibility from the latent difference.

\paragraph{Control-effect accessibility and agreement.}
We finally test whether $\Delta z_t$ makes physically predicted control
effects linearly accessible. Two ridge readouts are fit from $\Delta z_t$:
\begin{equation}
\begin{aligned}
    c_A(\Delta z_t)
        &\approx
        \widehat{\Delta s}^{\,\mathrm{lin}}_{A,t}, \\
    c_{AS}(\Delta z_t)
        &\approx
        \widehat{\Delta s}^{\,\mathrm{MLP}}_{AS,t}.
\end{aligned}
\end{equation}
Their held-out $R^2$ scores are denoted $C_A$ and $C_{AS}$,
respectively. Hence $C_A$ measures accessibility of the linear
action-only physical effect, whereas $C_{AS}$ uses the nonlinear
state-conditioned physical estimate as its target.

The bridge scores measure agreement in physical displacement coordinates
without fitting an additional bridge probe. We reuse the ordinary linear
motion decoder $d_{\mathrm{lin}}$ and define
\begin{equation}
    B_P
    =
    1-
    \frac{
        \sum_t
        \left\|
            d_{\mathrm{lin}}(\Delta z_t)
            -
            \widehat{\Delta s}_{P,t}
        \right\|_2^2
    }{
        \sum_t
        \left\|
            \Delta s_t-\overline{\Delta s}
        \right\|_2^2
    },
    \qquad
    P\in\{A,AS\},
    \label{eq:bridge-score}
\end{equation}
where
\begin{equation}
    \widehat{\Delta s}_{A,t}
        :=
        \widehat{\Delta s}^{\,\mathrm{lin}}_{A,t},
    \qquad
    \widehat{\Delta s}_{AS,t}
        :=
        \widehat{\Delta s}^{\,\mathrm{MLP}}_{AS,t}.
\end{equation}
We denote the resulting scores by $B_A$ and $B_{AS}$.
Unlike $C_A$ and $C_{AS}$, these are normalized agreement scores rather
than $R^2$ values of newly fitted probes.

To match the notation from the main context, the six quantities reported
there correspond to
\begin{equation}
    P_A = P_A^{\mathrm{MLP}},
    \qquad
    P_{AS} = P_{AS}^{\mathrm{MLP}},
    \qquad
    D_{\mathrm{lin}},
    \qquad
    D_{\mathrm{MLP}},
    \qquad
    C_A,
    \qquad
    B_{AS}.
    \label{eq:main-probe-mapping}
\end{equation}
The additional quantities reported here provide complementary diagnostics.
In particular, a high $D_{\mathrm{MLP}}$ together with a substantially
lower $D_{\mathrm{lin}}$ indicates that motion information is retained
but is less linearly accessible under these probe classes; this pattern is
consistent with a more nonlinear organization of the latent representation.

% ------------------------------------------------------------------
% Keep Table 7 here, with your manual table edits.
% ------------------------------------------------------------------
\begin{table*}[!htbp]
  \centering
  \vspace*{\baselineskip}
  \caption{Physical-motion identification. $A$, $S$, and $AS$ use the
  action, current state, and both. $R^2_{A\mid S}$ uses the MLP probes and the
  state-only residual as its denominator. Negative values denote performance
  below the held-out mean predictor.}
  \label{tab:probing-physical-ceilings}
  \scriptsize
  \setlength{\tabcolsep}{3.2pt}
  \resizebox{\textwidth}{!}{%
  \begin{tabular}{@{}llrrrrrrr@{}}
    \toprule
    Environment & Physical change
      & $P_A^{\rm lin}$ & $P_A^{\rm MLP}$
      & $P_S^{\rm lin}$ & $P_S^{\rm MLP}$
      & $P_{AS}^{\rm lin}$ & $P_{AS}^{\rm MLP}$
      & $R^2_{A\mid S}$ \\
    \midrule
    \multirow{3}{*}{PushT}
      & Agent position & 0.999 & 0.999 & 0.122 & 0.266 & 0.999 & 0.999 & 0.999 \\
      & Block position & 0.121 & 0.317 & 0.092 & 0.372 & 0.175 & 0.613 & 0.384 \\
      & Block angle    & -0.012 & 0.074 & 0.045 & 0.271 & 0.044 & 0.543 & 0.372 \\
    \midrule
    \multirow{6}{*}{Cube}
      & Joint position        & 0.946 & 0.966 & 0.718 & 0.834 & 0.960 & 0.997 & 0.981 \\
      & End-effector position & 0.992 & 0.994 & 0.753 & 0.870 & 0.993 & 0.999 & 0.990 \\
      & End-effector yaw      & 0.991 & 0.992 & 0.528 & 0.628 & 0.992 & 0.998 & 0.995 \\
      & Gripper opening       & 0.250 & 0.961 & 0.541 & 0.886 & 0.853 & 0.992 & 0.934 \\
      & Block position        & 0.492 & 0.991 & 0.797 & 0.863 & 0.876 & 0.997 & 0.976 \\
      & Block yaw             & 0.294 & 0.916 & 0.523 & 0.595 & 0.569 & 0.913 & 0.785 \\
    \midrule
    \multirow{2}{*}{Reacher}
      & Joint position     & 0.847 & 0.847 & 0.093 & 0.110 & 0.948 & 0.993 & 0.992 \\
      & Fingertip position & -0.006 & -0.006 & -0.001 & 0.251 & -0.005 & 0.983 & 0.977 \\
    \midrule
    TwoRoom
      & Agent position & 0.952 & 0.953 & 0.047 & 0.057 & 0.952 & 0.965 & 0.963 \\
    \bottomrule
  \end{tabular}%
  }
\end{table*}

\paragraph{Physical action-identifiability diagnostics.}
The physical-prediction results distinguish action dependence from state
dependence under the specified probe classes. PushT agent displacement
and TwoRooms agent displacement are almost completely predictable from
action alone. Reacher fingertip displacement is qualitatively different:
both the linear and nonlinear action-only probes obtain
$R^2=-0.006$, whereas the state-conditioned MLP reaches
$P_{AS}^{\mathrm{MLP}}=0.983$. Cube joint, end-effector, and yaw changes
are also highly predictable from action, while gripper and block effects
show larger linear-to-nonlinear gaps. These physical diagnostics provide
the reference needed to interpret the representation probes below.

% ------------------------------------------------------------------
% Keep Tables 8--10 here, after deleting A -> Delta z manually.
% ------------------------------------------------------------------
\begin{table*}[!htbp]
  \centering
  \captionsetup{skip=7pt}
  \caption{Complete PushT representation probes. $D$ denotes dynamic
  decoding, $C$ control-effect accessibility, and $B$ bridge agreement.
  Subscripts $A$ and $AS$ denote action-only and state-conditioned physical
  targets.}
  \label{tab:probing-pusht-complete}
  \fontsize{8}{8.4}\selectfont
  \renewcommand{\arraystretch}{0.90}
  \begin{tabular*}{\textwidth}{@{\extracolsep{\fill}}llrrrrrr@{}}
    \toprule
    Model & Target & $D_{\rm lin}$ & $D_{\rm MLP}$
      & $C_A$ & $C_{AS}$ & $B_A$ & $B_{AS}$ \\
    \midrule
    LeWM & Agent position & 0.809 & 0.888 & 0.807 & 0.808 & 0.807 & 0.808 \\
    TS   & Agent position & 0.831 & 0.891 & 0.829 & 0.830 & 0.829 & 0.830 \\
    CGS  & Agent position & 0.903 & 0.959 & 0.901 & 0.902 & 0.901 & 0.902 \\
    \addlinespace[1pt]
    LeWM & Block position & 0.813 & 0.864 & 0.791 & 0.602 & 0.256 & 0.603 \\
    TS   & Block position & 0.816 & 0.830 & 0.817 & 0.670 & 0.314 & 0.678 \\
    CGS  & Block position & 0.847 & 0.873 & 0.871 & 0.655 & 0.276 & 0.664 \\
    \addlinespace[1pt]
    LeWM & Block angle & 0.151 & 0.604 & 0.773 & 0.158 & 0.423 & 0.308 \\
    TS   & Block angle & 0.357 & 0.574 & 0.808 & 0.287 & 0.386 & 0.404 \\
    CGS  & Block angle & 0.227 & 0.601 & 0.843 & 0.143 & 0.450 & 0.296 \\
  \bottomrule
  \end{tabular*}
  \vspace{5pt}
\end{table*}

\begin{table*}[!htbp]
  \centering
  \captionsetup{skip=7pt}
  \caption{Complete Cube representation probes. CGS uses grasp-active transitions
  and XYZ controls for its geometry objective. All three models share the
  dataset, episode split, horizon, latent dimension, and probe protocol.}
  \label{tab:probing-cube-complete}
  \fontsize{8}{8.4}\selectfont
  \renewcommand{\arraystretch}{0.88}
  \begin{tabular*}{\textwidth}{@{\extracolsep{\fill}}llrrrrrr@{}}
    \toprule
    Model & Target & $D_{\rm lin}$ & $D_{\rm MLP}$
      & $C_A$ & $C_{AS}$ & $B_A$ & $B_{AS}$ \\
    \midrule
    LeWM             & Joint position & 0.742 & 0.747 & 0.741 & 0.743 & 0.723 & 0.743 \\
    TS               & Joint position & 0.622 & 0.643 & 0.629 & 0.622 & 0.631 & 0.621 \\
    CGS  & Joint position & 0.743 & 0.895 & 0.748 & 0.741 & 0.729 & 0.741 \\
    \addlinespace[1pt]
    LeWM             & End-effector position & 0.895 & 0.969 & 0.893 & 0.893 & 0.891 & 0.893 \\
    TS               & End-effector position & 0.710 & 0.776 & 0.706 & 0.708 & 0.700 & 0.709 \\
    CGS  & End-effector position & 0.931 & 0.982 & 0.930 & 0.930 & 0.927 & 0.930 \\
    \addlinespace[1pt]
    LeWM             & End-effector yaw & -0.039 & -0.023 & -0.036 & -0.040 & -0.026 & -0.042 \\
    TS               & End-effector yaw & -0.025 & -0.020 & -0.024 & -0.026 & -0.010 & -0.026  \\
    CGS  & End-effector yaw & -0.078 & 0.830 & -0.080 & -0.078 & -0.063 & -0.078 \\
    \addlinespace[1pt]
    LeWM             & Gripper opening & 0.490 & 0.905 & 0.205 & 0.494 & 0.502 & 0.498 \\
    TS               & Gripper opening & 0.528 & 0.905 & 0.296 & 0.533 & 0.543 & 0.537 \\
    CGS  & Gripper opening & 0.898 & 0.945 & 0.381 & 0.900 & 0.307 & 0.901 \\
    \addlinespace[1pt]
    LeWM             & Block position & 0.941 & 0.970 & 0.926 & 0.939 & 0.517 & 0.939  \\
    TS               & Block position & 0.946 & 0.976 & 0.813 & 0.943 & 0.513 & 0.944 \\
    CGS  & Block position & 0.918 & 0.985 & 0.948 & 0.916 & 0.526 & 0.916 \\
    \addlinespace[1pt]
    LeWM             & Block yaw & -0.013 & -0.009 & -0.054 & -0.041 & 0.757 & 0.101  \\
    TS               & Block yaw & -0.018 & -0.020 & -0.029 & -0.050 & 0.750 & 0.081 \\
    CGS  & Block yaw & -0.009 & 0.426 & -0.102 & -0.032 & 0.751 & 0.096 \\
    \bottomrule
  \end{tabular*}
\end{table*}

\begin{table*}[!htbp]
  \centering
  \captionsetup{skip=7pt}
  \caption{Complete Reacher and TwoRoom representation probes, using the same
  notation and protocol as Tables~\ref{tab:probing-pusht-complete}
  and~\ref{tab:probing-cube-complete}.}
  \label{tab:probing-reacher-tworoom-complete}
  \fontsize{8}{8.4}\selectfont
  \renewcommand{\arraystretch}{0.90}
  \begin{tabular*}{\textwidth}{@{\extracolsep{\fill}}lllrrrrrr@{}}
    \toprule
    Environment & Model & Target & $D_{\rm lin}$ & $D_{\rm MLP}$
      & $C_A$ & $C_{AS}$ & $B_A$ & $B_{AS}$ \\
    \midrule
    \multirow{6}{*}{Reacher}
      & LeWM & Joint position & 0.355 & 0.988 & 0.368 & 0.357 & 0.447 & 0.364  \\
      & TS   & Joint position & 0.347 & 0.988 & 0.371 & 0.348 & 0.447 & 0.353 \\
      & CGS  & Joint position & 0.123 & 0.988 & 0.135 & 0.127 & 0.248 & 0.134 \\
      \addlinespace[1pt]
      & LeWM & Fingertip position & 0.951 & 0.994 & 0.155 & 0.932 & 0.028 & 0.932 \\
      & TS   & Fingertip position & 0.974 & 0.994 & 0.157 & 0.960 & 0.011 & 0.959  \\
      & CGS  & Fingertip position & 0.899 & 0.993 & 0.041 & 0.883 & 0.074 & 0.884  \\
    \midrule
    \multirow{3}{*}{TwoRoom}
      & LeWM & Agent position & 0.862 & 0.956 & 0.815 & 0.829 & 0.819 & 0.827 \\
      & TS   & Agent position & 0.970 & 0.984 & 0.933 & 0.946 & 0.928 & 0.939  \\
      & CGS  & Agent position & 0.903 & 0.976 & 0.866 & 0.878 & 0.867 & 0.873 \\
    \bottomrule
  \end{tabular*}
\end{table*}

\subsubsection{Representation probes among environments.}

\textbf{PushT: a targeted positive result.}
For agent displacement, action-only and state-conditioned physical
prediction both reach $R^2=0.999$. Relative to LeWM, CGS raises
$D_{\mathrm{lin}}$ from $0.809$ to $0.903$,
$D_{\mathrm{MLP}}$ from $0.888$ to $0.959$, and
$C_A$ from $0.807$ to $0.901$, while also giving the strongest
$B_{AS}=0.902$; TS is intermediate on these agent-motion metrics.
CGS is therefore strongest for the directly action-predictable agent
effect and also improves linear motion decoding and action-effect
accessibility for block translation. The state-conditioned object results
are less uniform: TS gives the strongest block-position bridge score
($0.678$ versus $0.664$ for CGS) and block-angle bridge score
($0.404$ versus $0.296$). These results should be read together with the
lower action-only physical-prediction scores for the object targets:
accessibility of a predictable component does not imply that action alone
determines the full object displacement.

\textbf{Cube.}
All scores below are evaluated on the full held-out test split, although the
CGS geometry objective was applied only to grasp-active transitions during
representation training. CGS gives the strongest nonlinear motion decoding
on all six evaluated targets. For end-effector position,
$D_{\mathrm{lin}}$ rises from $0.895$ with LeWM to $0.931$ with CGS;
for gripper opening, it rises from $0.490$ to $0.898$, while
$B_{AS}$ rises from $0.498$ to $0.901$.
TS retains the strongest block-position agreement
($B_{AS}=0.944$). Both yaw targets remain difficult to decode linearly,
despite substantially stronger nonlinear readout for CGS.

\textbf{Reacher: motion is retained but strongly state dependent.}
For fingertip displacement, both action-only physical probes remain at
$P_A=-0.006$, whereas conditioning on joint positions and velocities
raises the nonlinear score to $P_{AS}=0.983$. Thus the evaluated
fingertip displacement is strongly configuration dependent under these
probe classes. Increasing from the linear action-only predictor to our
nonlinear action-only MLP does not repair this mismatch, while the
state-conditioned predictor does.

At the representation level, direct action-effect accessibility remains
low for fingertip motion
($C_A=0.041$--$0.157$), and CGS obtains
$B_{AS}=0.884$, below LeWM ($0.932$) and TS ($0.959$).
At the same time, all three representations retain nearly all evaluated
motion information under nonlinear decoding:
$D_{\mathrm{MLP}}\approx0.99$ for both joint and fingertip targets.
The main difference is therefore how linearly accessible that information
is from $\Delta z_t$, rather than whether it is present at all.
TS gives the strongest state-conditioned accessibility and agreement on
the fingertip target, while CGS retains the motion information with
weaker linear accessibility.

\textbf{TwoRooms: TS is strongest despite highly identifiable action effects.}
TwoRooms has high action-only and state-conditioned physical-prediction
scores, yet TS is strongest across the reported representation probes,
including
$D_{\mathrm{lin}}=0.970$,
$C_{AS}=0.946$, and
$B_{AS}=0.939$.
CGS improves over LeWM but remains below TS on these diagnostics.
High action identifiability alone is therefore not sufficient for CGS to
outperform temporal straightening.

\subsection{Mechanism-Figure Protocol and Additional Results}
\label{app:cgs_figure3_audit}

The illustration in Figure~\ref{fig:cgs_mppi_mechanism} links measured
Cube planning geometry to one-step MPPI progress. We use 32 shared
grasp-active anchors from distinct episodes and perturb the XYZ commands
uniformly across the 25-action horizon. The orthonormal basis
$U=\mathbf 1_{25}\otimes[I_3,0_{3\times2}]^\top/\sqrt{25}$ maps
three-dimensional perturbations into the full action sequence. At each
recorded action sequence, we differentiate the terminal predictor and form
\begin{equation}
\begin{aligned}
  G_{H,\mathrm{XYZ}}&:=(J_{\widehat F_H}U)^\top(J_{\widehat F_H}U),\\
  \bar G_{H,\mathrm{XYZ}}&:=\frac{G_{H,\mathrm{XYZ}}}{\operatorname{tr}(G_{H,\mathrm{XYZ}})/3}.
\end{aligned}
\label{eq:figure3_xyz_gram}
\end{equation}
The response Gram matrix $G_{H,\mathrm{XYZ}}$ measures local sensitivity to
XYZ actions. Trace normalization fixes the average eigenvalue at one,
allowing spectral-shape comparisons at a common scale. Below, write
$\bar G:=\bar G_{H,\mathrm{XYZ}}$. Since $\bar G$ is positive semidefinite,
$\sigma_{\min}(\bar G)=\lambda_{\min}(\bar G)$ measures its weakest
normalized XYZ curvature.

Let $V$ contain the eigenvectors of $\bar G$ in descending eigenvalue
order. We place the target along its weakest direction,
$u^\star=(0,0,1)^\top$, and evaluate the quadratic cost
\begin{equation}
  \widehat C(u)=(u-u^\star)^\top V^\top\bar G V(u-u^\star),\qquad
  q_K=\frac{\|\widehat\mu_K-u^\star\|_2}{\|u^\star\|_2}.
\end{equation}
MPPI starts from $\mathcal N(0,I_3)$ with temperature $\tau=2$;
$\widehat\mu_K$ is its weighted sample mean. We call $q_K$ the
$K$-sample contraction factor: the ratio of post-update to initial
mean-error norms, with initial mean zero. It is a random error ratio,
not an empirical Jacobian norm, and can exceed one.
Each anchor uses 512
Gaussian batches, with shared sample coordinates across models and
prefix budgets $K\in\{1,2,4,8,16,32,64,128,256\}$. Thus $\Pr(q_K<1)$ is the probability
that one update reduces target error in this quadratic problem. Curves
average over anchors, with 95\% bootstrap intervals.

The quadratic Hessian is $2V^\top\bar G V$, so unit proposal covariance
and $\tau=2$ give the population error map $(I_3+V^\top\bar G V)^{-1}$.
Because the initial error lies along the weakest-curvature direction,
the population counterpart of $q_K$ is
\begin{equation}
  q_\infty=\frac{1}{1+\sigma_{\min}(\bar G_{H,\mathrm{XYZ}})}.
  \label{eq:figure3_population_contraction}
\end{equation}
For full-rank $\bar G$, this equals the contraction factor $q_{\rm M}$
in Theorem~\ref{thm:cgs_mppi_population} applied to this calibrated
three-dimensional quadratic probe, rather than the full nonlinear
planning problem. Thus, larger curvature along the weakest direction
produces greater population progress.
Median $q_\infty$ is $0.841$, $0.835$, and $0.685$ for LeWM, TS,
and CGS. At $K=32$, their one-update progress probabilities are
$75.1\%$, $80.3\%$, and $91.3\%$. This illustration shows how more
balanced action sensitivity can improve MPPI updates at a fixed sample budget.

\subsection{DINO-WM Implementation and Additional Results}
\label{app:dino-analysis}

This section provides the implementation details for the DINO-WM
comparisons in Figure~\ref{fig:pusht-dino-lewm} and reports additional
results on the DINO-WM benchmark settings used by Temporal
Straightening~\citep{wang2026ts}. Following that setup, we evaluate
Base (w/o TS or CGS), TS, and CGS on PushT, Wall, PointMaze--UMaze, and
PointMaze--Medium, using both global CLS and spatial Patch
representations. Within each representation type, all three objectives
share the same frozen DINOv2 backbone, data, predictor architecture,
and planning protocol.

\paragraph{Representations and geometry objectives.}
Temporal Straightening was evaluated with frozen DINOv2 features and
trainable projectors, as well as with ResNet encoders learned from
scratch~\citep{wang2026ts}. Our DINO-based experiments follow the
spatial DINO-WM setting of DINO-WM and TS~\citep{zhou2025dinowm,
wang2026ts}: a frozen DINOv2-S/14 backbone~\citep{oquabdinov22024}
produces a $14\times14$ grid of 384-dimensional patch tokens, which are
projected to eight dimensions per token. The predictor rolls out this
projected spatial grid. TS and CGS regularize a learned
128-dimensional aggregation of patch-token differences, while prediction
and planning retain the full spatial representation.

Our CLS variant instead uses the backbone's 384-dimensional CLS token,
followed by an identity-initialized trainable adapter, and applies the
geometry objective directly to the resulting global latent differences.
The global DINO variant in the original TS study obtains a single vector
by projecting patch features~\citep{wang2026ts}; our CLS variant
therefore differs in the backbone output to which the geometry objective
is applied. In comparison, LeWM learns its visual encoder end to end and
predicts a compact 192-dimensional projected CLS latent
~\citep{maes2026lewm}.

\paragraph{Training protocol.}
\label{app:dino-training-details}
Prediction targets use stop-gradient in the DINO-WM variants. The
projector or CLS adapter, predictor, action encoder, and geometry head
are trainable, while the DINOv2 backbone remains frozen. On PushT, we
use a 90/10 trajectory split, two training epochs, and unit gradient
clipping. Representation learning rates are $10^{-6}$ for Base and
$10^{-5}$ for TS and CGS; the predictor and action encoder use a learning
rate of $5\times10^{-4}$ with weight decay $10^{-2}$. Models condition
on three observations and predict five-step transitions. The PushT CGS
weights are $0.1$ for Patch and CLS.

\paragraph{Planning protocol.}
DINO-WM sampling costs average squared visual prediction errors over
tokens and channels. On PushT, we additionally predict proprioception
and add its MSE to the visual cost with unit weight; proprioception is
not included in either TS or CGS regularization. The GD comparison uses
the optimization schedule in Appendix~\ref{app:planning-datasets}, with
per-step error given by visual MSE plus proprioceptive MSE.

Sampling-based planners use $I=30$ optimization steps, and CEM retains
an elite fraction of $0.25$. GD uses 100 Adam steps. Within each
environment, a single MPPI temperature is shared across CLS/Patch and
Base/TS/CGS: $0.015$ on PushT, $5\times10^{-6}$ on Wall,
$10^{-3}$ on PointMaze--UMaze, and $3\times10^{-3}$ on
PointMaze--Medium. These temperatures are fixed across sampling budgets
and representation objectives. Evaluation uses three evaluation seeds and
50 episodes per seed.

\begin{table*}[!htbp]
  \centering
  \scriptsize
  \setlength{\tabcolsep}{3pt}
  \vspace*{\baselineskip}
  \caption{\textbf{Additional DINO-WM benchmark results.}
  Goal-reaching success rate (\%) for Base, TS, and CGS using frozen
  DINOv2 CLS or Patch representations on the DINO-WM benchmark settings
  used by Temporal Straightening. CEM and MPPI use $I=30$ optimization
  steps, and GD uses 100 Adam steps. A single environment-level MPPI
  temperature is shared across representation types and training
  objectives. Values are means $\pm$ sample SD over three evaluation seeds with
  50 evaluation episodes per seed; PushT additionally uses
  proprioception. Bold marks the best objective for a fixed
  representation type, planner, environment, and sampling budget.}
  \label{tab:dino-full}
  \resizebox{\textwidth}{!}{%
  \begin{tabular}{@{}ll*{12}{c}@{}}
    \toprule
    \multirow{3}{*}{World model / objective}
      & \multirow{3}{*}{Planning method}
      & \multicolumn{3}{c}{Wall}
      & \multicolumn{3}{c}{PointMaze--UMaze}
      & \multicolumn{3}{c}{PointMaze--Medium}
      & \multicolumn{3}{c}{PushT} \\
    \cmidrule(lr){3-5}\cmidrule(lr){6-8}
    \cmidrule(lr){9-11}\cmidrule(l){12-14}
      & & \multicolumn{3}{c}{Sampling budget}
      & \multicolumn{3}{c}{Sampling budget}
      & \multicolumn{3}{c}{Sampling budget}
      & \multicolumn{3}{c}{Sampling budget} \\
      & & $K{=}32$ & $K{=}64$ & $K{=}128$
      & $K{=}32$ & $K{=}64$ & $K{=}128$
      & $K{=}32$ & $K{=}64$ & $K{=}128$
      & $K{=}32$ & $K{=}64$ & $K{=}128$ \\
    \midrule

    \multirow{3}{*}{DINO-WM (CLS)}
      & CEM
      & $74.7{\pm}3.1$ & $76.0{\pm}5.3$ & $74.7{\pm}6.1$
      & $88.0{\pm}3.5$ & $92.7{\pm}3.1$ & $90.7{\pm}5.0$
      & $85.3{\pm}4.2$ & $86.0{\pm}2.0$ & $84.7{\pm}7.0$
      & $47.3{\pm}6.4$ & $56.0{\pm}8.7$ & $63.3{\pm}5.0$ \\
      & MPPI
      & $84.0{\pm}5.3$
      & $\mathbf{86.7{\pm}6.4}$
      & $\mathbf{89.3{\pm}5.0}$
      & $92.0{\pm}2.0$ & $93.3{\pm}2.3$ & $96.7{\pm}4.2$
      & $90.7{\pm}5.0$ & $90.7{\pm}2.3$ & $92.7{\pm}3.1$
      & $31.3{\pm}6.1$ & $35.3{\pm}4.6$ & $46.7{\pm}3.1$ \\
      & GD
      & \multicolumn{3}{c}{$55.3{\pm}2.3$}
      & \multicolumn{3}{c}{$62.0{\pm}5.3$}
      & \multicolumn{3}{c}{$56.0{\pm}4.0$}
      & \multicolumn{3}{c}{$46.0{\pm}2.0$} \\
    \addlinespace[1pt]

    \multirow{3}{*}{+TS (CLS)}
      & CEM
      & $79.3{\pm}6.1$ & $78.7{\pm}4.2$ & $77.3{\pm}8.3$
      & $86.0{\pm}2.0$ & $89.3{\pm}5.0$ & $90.0{\pm}2.0$
      & $87.3{\pm}1.2$ & $84.7{\pm}3.1$ & $89.3{\pm}4.6$
      & $41.3{\pm}5.0$ & $38.0{\pm}2.0$ & $47.3{\pm}4.2$ \\
      & MPPI
      & $\mathbf{88.7{\pm}5.8}$
      & $83.3{\pm}7.6$
      & $84.0{\pm}4.0$
      & $91.3{\pm}4.2$ & $94.7{\pm}1.2$ & $95.3{\pm}3.1$
      & $88.7{\pm}4.2$ & $88.7{\pm}6.1$ & $94.0{\pm}2.0$
      & $29.3{\pm}4.2$ & $29.3{\pm}10.3$ & $29.3{\pm}5.8$ \\
      & GD
      & \multicolumn{3}{c}{$68.0{\pm}4.0$}
      & \multicolumn{3}{c}{$66.7{\pm}3.1$}
      & \multicolumn{3}{c}{$65.3{\pm}8.1$}
      & \multicolumn{3}{c}{$34.0{\pm}14.0$} \\
    \addlinespace[1pt]

    \multirow{3}{*}{+CGS (CLS)}
      & CEM
      & $\mathbf{83.3{\pm}8.1}$
      & $\mathbf{81.3{\pm}7.0}$
      & $\mathbf{83.3{\pm}3.1}$
      & $\mathbf{93.3{\pm}3.1}$
      & $\mathbf{96.0{\pm}4.0}$
      & $\mathbf{96.0{\pm}3.5}$
      & $\mathbf{92.0{\pm}2.0}$
      & $\mathbf{92.7{\pm}3.1}$
      & $\mathbf{93.3{\pm}2.3}$
      & $\mathbf{55.3{\pm}4.6}$
      & $\mathbf{65.3{\pm}6.1}$
      & $\mathbf{66.7{\pm}6.4}$ \\
      & MPPI
      & $86.0{\pm}6.9$
      & $84.7{\pm}5.8$
      & $81.3{\pm}6.1$
      & $\mathbf{96.7{\pm}3.1}$
      & $\mathbf{96.0{\pm}2.0}$
      & $\mathbf{97.3{\pm}3.1}$
      & $\mathbf{94.7{\pm}1.2}$
      & $\mathbf{95.3{\pm}2.3}$
      & $\mathbf{95.3{\pm}3.1}$
      & $\mathbf{40.7{\pm}4.2}$
      & $\mathbf{46.7{\pm}2.3}$
      & $\mathbf{56.7{\pm}4.6}$ \\
      & GD
      & \multicolumn{3}{c}{$\mathbf{81.3{\pm}7.6}$}
      & \multicolumn{3}{c}{$\mathbf{81.3{\pm}4.2}$}
      & \multicolumn{3}{c}{$\mathbf{79.3{\pm}6.1}$}
      & \multicolumn{3}{c}{$\mathbf{57.3{\pm}8.1}$} \\
    \midrule

    \multirow{3}{*}{DINO-WM (Patch)}
      & CEM
      & $90.7{\pm}2.3$ & $81.3{\pm}8.1$ & $84.7{\pm}6.4$
      & $\mathbf{98.0{\pm}2.0}$
      & $\mathbf{98.7{\pm}1.2}$
      & $\mathbf{99.3{\pm}1.2}$
      & $\mathbf{94.7{\pm}1.2}$
      & $\mathbf{92.7{\pm}2.3}$
      & $\mathbf{98.7{\pm}1.2}$
      & $75.3{\pm}2.3$ & $81.3{\pm}4.2$ & $83.3{\pm}1.2$ \\
      & MPPI
      & $\mathbf{90.0{\pm}2.0}$
      & $89.3{\pm}2.3$
      & $86.0{\pm}5.3$
      & $\mathbf{99.3{\pm}1.2}$
      & $96.7{\pm}1.2$
      & $98.7{\pm}1.2$
      & $94.0{\pm}4.0$
      & $\mathbf{97.3{\pm}3.1}$
      & $\mathbf{97.3{\pm}2.3}$
      & $68.0{\pm}4.0$ & $62.0{\pm}4.0$
      & $75.3{\pm}5.0$ \\
      & GD
      & \multicolumn{3}{c}{$83.3{\pm}3.1$}
      & \multicolumn{3}{c}{$94.0{\pm}4.0$}
      & \multicolumn{3}{c}{$\mathbf{88.7{\pm}4.2}$}
      & \multicolumn{3}{c}{$78.0{\pm}5.3$} \\
    \addlinespace[1pt]

    \multirow{3}{*}{+TS (Patch)}
      & CEM
      & $\mathbf{91.3{\pm}1.2}$
      & $87.3{\pm}2.3$
      & $\mathbf{86.0{\pm}2.0}$
      & $94.7{\pm}3.1$ & $97.3{\pm}3.1$ & $95.3{\pm}1.2$
      & $88.7{\pm}3.1$ & $90.7{\pm}1.2$ & $92.0{\pm}3.5$
      & $76.7{\pm}5.0$ & $\mathbf{84.7{\pm}4.6}$
      & $\mathbf{87.3{\pm}4.2}$ \\
      & MPPI
      & $89.3{\pm}1.2$
      & $\mathbf{91.3{\pm}2.3}$
      & $90.7{\pm}3.1$
      & $97.3{\pm}3.1$ & $96.7{\pm}1.2$ & $98.0{\pm}2.0$
      & $94.0{\pm}2.0$ & $93.3{\pm}3.1$ & $95.3{\pm}2.3$
      & $\mathbf{75.3{\pm}4.2}$
      & $\mathbf{80.0{\pm}8.7}$ & $78.7{\pm}4.6$ \\
      & GD
      & \multicolumn{3}{c}{$86.7{\pm}2.3$}
      & \multicolumn{3}{c}{$\mathbf{96.0{\pm}2.0}$}
      & \multicolumn{3}{c}{$87.3{\pm}4.6$}
      & \multicolumn{3}{c}{$\mathbf{84.0{\pm}5.3}$} \\
    \addlinespace[1pt]

    \multirow{3}{*}{+CGS (Patch)}
      & CEM
      & $88.0{\pm}4.0$
      & $\mathbf{88.7{\pm}5.0}$
      & $85.3{\pm}6.4$
      & $96.0{\pm}2.0$ & $96.0{\pm}2.0$ & $98.0{\pm}0.0$
      & $94.0{\pm}2.0$ & $86.0{\pm}2.0$ & $92.0{\pm}3.5$
      & $\mathbf{80.0{\pm}3.5}$ & $81.3{\pm}2.3$
      & $84.7{\pm}5.0$ \\
      & MPPI
      & $88.0{\pm}7.2$
      & $88.7{\pm}1.2$
      & $\mathbf{92.7{\pm}3.1}$
      & $98.7{\pm}1.2$
      & $\mathbf{99.3{\pm}1.2}$
      & $\mathbf{100.0{\pm}0.0}$
      & $\mathbf{94.7{\pm}1.2}$
      & $95.3{\pm}1.2$
      & $96.0{\pm}2.0$
      & $64.7{\pm}4.2$ & $72.7{\pm}6.4$
      & $\mathbf{80.7{\pm}2.3}$ \\
      & GD
      & \multicolumn{3}{c}{$\mathbf{88.7{\pm}4.2}$}
      & \multicolumn{3}{c}{$\mathbf{96.0{\pm}3.5}$}
      & \multicolumn{3}{c}{$88.0{\pm}2.0$}
      & \multicolumn{3}{c}{$76.7{\pm}1.2$} \\
    \bottomrule
  \end{tabular}%
  }
\end{table*}

\paragraph{Additional DINO-WM benchmark results.}
Table~\ref{tab:dino-full} shows that the effect of geometry shaping
depends on both the underlying representation and the planner. CGS gives
particularly consistent improvements with the compact CLS representation.
At $K=128$, relative to Base CLS, CGS improves CEM on all four
environments, improves GD on all four, and improves MPPI on
PointMaze--UMaze, PointMaze--Medium, and PushT. Wall MPPI is the
exception, where the Base representation performs better.
The Patch representation begins from substantially stronger planning
performance and leaves less headroom for geometry regularization. Its
results are correspondingly more mixed: CGS improves several
planner--task combinations, including PushT CEM and MPPI, Wall MPPI,
and PointMaze--UMaze MPPI, while Base or TS remains stronger in other
settings. These additional experiments therefore show that CGS can be
applied to both compact global and spatial DINO-WM representations,
while its benefit is not uniform across every representation--planner
combination.

\paragraph{Cross-architecture planning time.}
We measure DINO-WM (Patch)+CGS and LeWM+CGS on the same GPU, using episode batch size 50 for GD. Times sum the two planner solves and are averaged per
episode; model loading, environment execution, and rendering are
excluded. Sampling uses $K=128$ and $I=30$, while GD uses 100
optimization steps. The DINO-WM PushT planner uses proprioception and
MPPI temperature $0.015$; LeWM uses image latents and temperature $4$.
Figure~\ref{fig:pusht-dino-lewm} combines these timing measurements
with the success rates from the main evaluations.

\begin{table}[!htbp]
  \centering
  \vspace*{\baselineskip}
  \caption{\textbf{PushT planner-solve time at $K=128$.}
  Seconds per episode measured in a timing
  experiment. Times include both planner solves and exclude loading,
  environment execution, and rendering. Ratios characterize the stated
  implementations and batching settings.}
  \label{tab:pusht-dino-lewm-current}
  \setlength{\tabcolsep}{10pt}
  \begin{tabular}{@{}lrrr@{}}
    \toprule
    Model & CEM & MPPI & GD \\
    \midrule
    DINO-WM (Patch)+CGS & 86.852 & 86.959 & 6.516 \\
    LeWM+CGS & 5.163 & 1.850 & 0.161 \\
    DINO/LeWM & $16.8\times$ & $47.0\times$ & $40.5\times$ \\
    \bottomrule
  \end{tabular}
\end{table}

\subsection{Training Ablations}
\label{app:training-ablations}

We vary the ramp fraction, history length, SIGReg weight and CGS weight
on PushT, evaluating each checkpoint with MPPI, CEM, and GD
(Figure~\ref{fig:training-ablation-small-multiples}).
Sampling uses $K=128$, $I=30$, MPPI temperature $4$, and 32 CEM elites;
GD uses 100 steps.
The reference uses $\lambda_{\mathrm{CGS}}=0.01$, a 5\% linear ramp,
$N=3$, and $\lambda_{\mathrm{SIGReg}}=0.09$.

\begin{figure*}[!htbp]
  \centering
  \includegraphics[width=\textwidth]{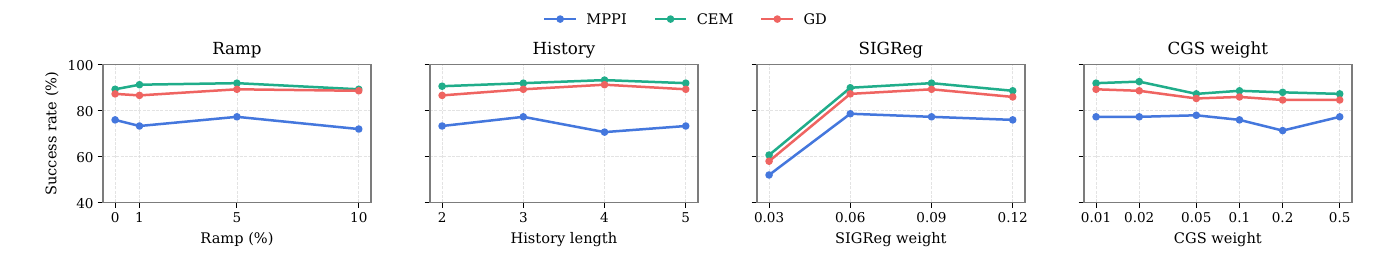}
  \caption{\textbf{PushT training ablations.}
Ramp fraction, history length, SIGReg weight, and CGS weight, evaluated
with MPPI, CEM, and GD.
Means use three evaluation seeds with 50 episodes each; GD is independent of $K$.
Reference points reuse the main PushT checkpoint with
$\lambda_{\mathrm{CGS}}=0.01$, a 5\% ramp, $N=3$, and
$\lambda_{\mathrm{SIGReg}}=0.09$.}
  \label{fig:training-ablation-small-multiples}
\end{figure*}

The 5\% ramp gives the strongest overall performance across the three
planners, with modest differences across ramp choices.
Following LeWM, $N=3$ provides an economical history setting: CEM and GD
gain slightly at $N=4$, while MPPI is strongest at $N=3$.
For SIGReg, weight $0.03$ substantially reduces success, indicating potential representation collapse with small SIGReg strength, while
$0.06$--$0.12$ forms a comparatively stable range around the default
$0.09$.
The CGS-weight sweep is broadly stable for all three planners over the tested
range, and goes down when CGS-weight becomes larger, suggesting conflicts with representation geometry induced by other loss terms, so we choose a conservative weight
$\lambda_{\mathrm{CGS}}=0.01$ on PushT.

\subsubsection{Can Explicit Temporal Straightening Further Benefit LeWM?}
\label{app:lewm-ts-cosine-effect}

\citet[Appendix H]{maes2026lewm} report that LeWM develops temporal
straightening on PushT without an explicit TS loss. Our diagnostic shows
that adding TS further raises the final cosine of consecutive latent
displacements from $0.548$ to $0.613$ on PushT slightly, with a small change in planning success rate in Table~\ref{tab:lewm-full-sweep-pusht}. Cube and TwoRooms
show clearer geometric changes ($0.262\to0.666$ and $-0.341\to-0.168$)
and planning gains (MPPI: $59.3\%\to75.3\%$ and $92.0\%\to97.3\%$).
Reacher remains near zero cosine ($-0.049\to-0.003$), indicating that
explicit TS produces only a small additional straightening effect in this
environment. These results support explicit TS as a useful addition to LeWM,
with the clearest gains on Cube and TwoRooms.
Figure~\ref{fig:lewm-ts-cosine-effect} tracks the cosine similarity of
consecutive latent displacements during training on fixed validation windows
for the LeWM and LeWM+TS variants in Section~\ref{sec:main-results}.

\begin{figure*}[!htbp]
  \centering
  \includegraphics[width=\textwidth]{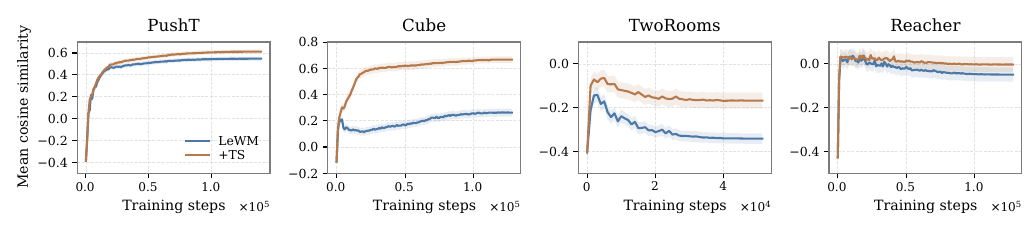}
  \caption{\textbf{Temporal straightening in LeWM.}
  Mean cosine similarity of consecutive latent displacements on fixed
  validation windows. Lines follow one training run per method and environment;
  bands show $\pm1.96$ standard errors across probe episodes.
  TS produces higher final mean cosine similarity on all four environments.}
  \label{fig:lewm-ts-cosine-effect}
\end{figure*}

\subsection{Shared MPPI Temperature Selection}
\label{app:pusht-temperature-tuning}

\begin{figure}[!b]
  \centering
  \begin{minipage}{0.92\textwidth}
    \centering
    \includegraphics[width=\linewidth]{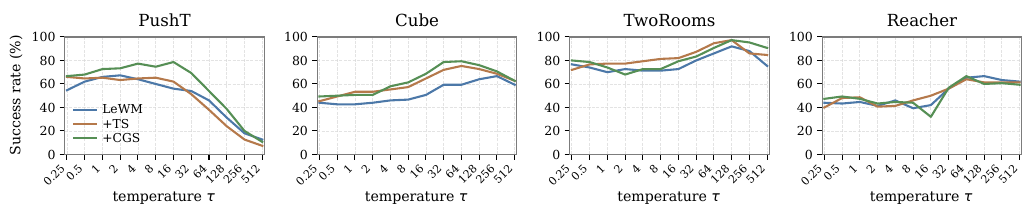}
  \end{minipage}

  \begin{minipage}{0.90\textwidth}
    \centering
    \begingroup
    \fontsize{8}{9.5}\selectfont
    \setlength{\tabcolsep}{3pt}
    \renewcommand{\arraystretch}{1.06}
    \begin{tabular*}{\linewidth}{@{\extracolsep{\fill}}lrrrr@{}}
    \toprule
    Model & PushT & Cube & TwoRooms & Reacher \\
    \midrule
    LeWM & $50.76$ & $52.15$ & $77.64$ & $51.00$ \\
    LeWM + TS & $49.82\,{\scriptstyle(-0.94)}$ & $61.45\,{\scriptstyle(+9.30)}$ & $\mathbf{83.42}\,{\scriptstyle(+5.79)}$ & $\mathbf{51.64}\,{\scriptstyle(+0.64)}$ \\
    LeWM + CGS & $\mathbf{60.48}\,{\scriptstyle(+9.73)}$ & $\mathbf{63.64}\,{\scriptstyle(+11.48)}$ & $81.58\,{\scriptstyle(+3.94)}$ & $50.73\,{\scriptstyle(-0.27)}$ \\
    \bottomrule
    \end{tabular*}
    \endgroup
  \end{minipage}
  \caption{\textbf{MPPI temperature sensitivity and temperature-selection nAUC at $K=128$.}
  \textbf{Top:} Mean success over three evaluation seeds (50 episodes each, $I=30$),
  using one shared temperature per environment and a 0--100\% axis.
  \textbf{Bottom:} nAUC over the displayed $\tau$ values on a $\log_2$ scale;
  parentheses give percentage-point differences from LeWM and bold marks the
  best representation.}
  \label{fig:temperature-k128}
\end{figure}

Figure~\ref{fig:temperature-k128} sweeps the MPPI temperature from
$0.25$ to $512$ at $K=128$ and $I=30$ on all four environments, and summarizes planning stability across wide temperature range in the bottom table.
Temperature controls the concentration of the MPPI importance weights.
At very small values, the weights collapse onto one or a few candidates,
producing an effectively single-mode update that is high-variance and brittle.
At very large values, the weights approach uniformity, so low-quality
candidates contribute substantially and dilute the improvement signal.
For fair selection, we examine the LeWM, LeWM+TS, and LeWM+CGS curves jointly
within each environment and choose one shared temperature from a broad stable
region where all three methods achieve competitive success: $4$ on PushT,
$64$ on Cube, $128$ on TwoRooms, and $64$ on Reacher. Each selected value is
fixed across representations and all sampling budgets in the main results and
Tables~\ref{tab:lewm-full-sweep-pusht}--\ref{tab:lewm-full-sweep-reacher}. nAUC results summarized in the bottom table also shows the stable improvement of CGS on pushT and Cube, comparable stable improvement of CGS with TS on TwoRooms, also similar planning performance across temperature of all three methods on Reacher.

\subsection{Refinement Efficiency and Matched BC Priors on Other Environments}
\label{app:other-iteration-prior}

\paragraph{Refinement efficiency.}
\label{app:planner-iteration-curves}

We evaluate refinement speed at fixed candidate budget $K=128$, using
32 CEM elites and the environment-specific MPPI temperatures from
Appendix~\ref{app:planning-datasets}.
We vary the number of optimization steps over
$I\in\{1,2,5,10,15,20,25,30\}$, with each replanning solve initialized independently.
Each run evaluates the final proposal, so success need not increase
monotonically with $I$. We also summarize the complete no-prior CEM and MPPI refinement curves using nAUC in Table~\ref{tab:app-iteration-curve-summary}.

On Cube, CGS improves both CEM and MPPI within the first few optimization
steps, consistent with its stronger finite-budget planning performance.
On TwoRooms, TS and CGS both reach high success rapidly, with little room
for further separation as the task approaches saturation.
On Reacher, refinement differences are smaller than on Cube and
TwoRooms, but CGS attains the highest aggregate refinement nAUC for
both CEM and MPPI. The gains are modest and not monotonic across
iterations, consistent with a more planner- and budget-dependent
effect.

\begin{figure*}[!htbp]
  \centering
  \includegraphics[width=\textwidth]{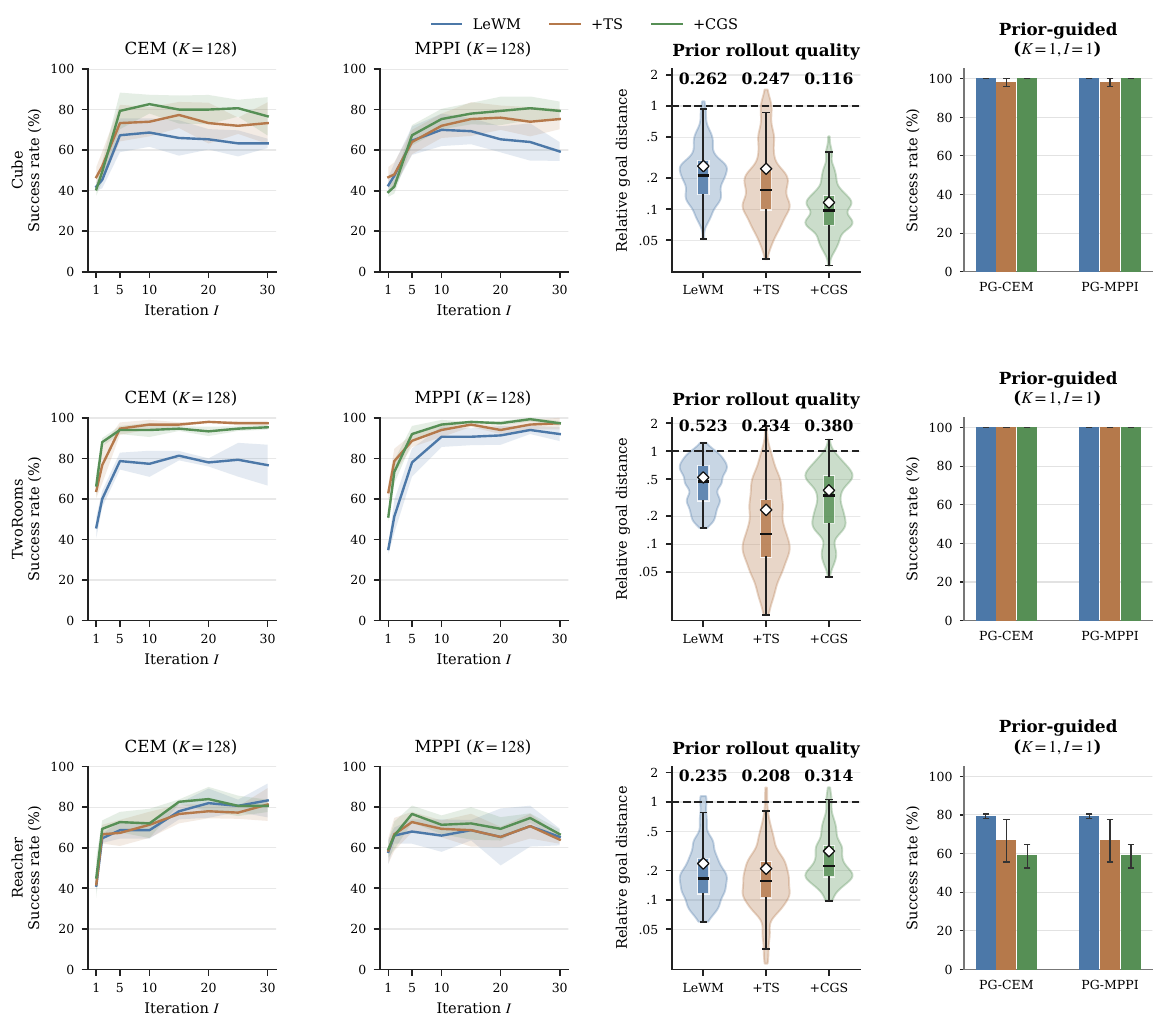}
  \caption{\textbf{Refinement efficiency and matched BC prior quality on Cube, TwoRooms,
  and Reacher.} Left: CEM and MPPI success over $I$
  (iterations) at $K=128$; bands show sample
  SD over three evaluation seeds. Middle: first-solve prior endpoint distance divided
  by zero-action endpoint distance, on the same 150 starts; diamonds and
  labels show means. Right: prior-only success at $K=1$, $I=1$, with the
  prior mean as the sole candidate; error bars show sample SD.}
  \label{fig:other-iteration-prior}
\end{figure*}

\begin{table}[!htbp]
\centering
\caption{\textbf{Refinement-iteration nAUC at $K=128$.}
$I\in\{1,2,5,10,15,20,25,30\}$. nAUC integrates over linear $I$.
Parentheses give percentage-point differences from the matching
LeWM baseline. Bold marks the best representation for each planner and environment.
All entries summarize three evaluation seeds. }
\label{tab:app-iteration-curve-summary}
\begingroup
\fontsize{8}{9.5}\selectfont
\setlength{\tabcolsep}{3pt}
\renewcommand{\arraystretch}{1.06}
\begin{tabular*}{\linewidth}{@{\extracolsep{\fill}}llrrrr@{}}
\toprule
Planner & Model & PushT & Cube & TwoRooms & Reacher \\
\midrule
CEM & LeWM & $73.93$ & $64.00$ & $76.87$ & $75.16$ \\
 & LeWM + TS & $72.23\,{\scriptstyle(-1.70)}$ & $71.98\,{\scriptstyle(+7.98)}$ & $\mathbf{94.85}\,{\scriptstyle(+17.98)}$ & $73.92\,{\scriptstyle(-1.24)}$ \\
 & LeWM + CGS & $\mathbf{80.25}\,{\scriptstyle(+6.32)}$ & $\mathbf{77.40}\,{\scriptstyle(+13.40)}$ & $93.34\,{\scriptstyle(+16.47)}$ & $\mathbf{77.60}\,{\scriptstyle(+2.44)}$ \\
\midrule
MPPI & LeWM & $64.85$ & $64.36$ & $86.06$ & $67.23$ \\
 & LeWM + TS & $67.06\,{\scriptstyle(+2.21)}$ & $70.70\,{\scriptstyle(+6.34)}$ & $92.89\,{\scriptstyle(+6.83)}$ & $68.39\,{\scriptstyle(+1.16)}$ \\
 & LeWM + CGS & $\mathbf{75.71}\,{\scriptstyle(+10.86)}$ & $\mathbf{73.72}\,{\scriptstyle(+9.37)}$ & $\mathbf{94.49}\,{\scriptstyle(+8.44)}$ & $\mathbf{71.44}\,{\scriptstyle(+4.21)}$ \\
\bottomrule
\end{tabular*}
\endgroup
\end{table}

\paragraph{Matched BC-prior quality.}
\label{app:prior-effectiveness}

We separately evaluate the quality of the matched BC prior before
sampling-based refinement.
At the first solve on 150 shared starts, we measure
\begin{equation}
  \rho =
  \frac{\lVert \widehat F_H(\mu_p)-z_g\rVert_2}
       {\lVert \widehat F_H(0)-z_g\rVert_2},
  \label{eq:prior-radius-ratio}
\end{equation}
where $\mu_p$ is the prior-predicted action sequence.
Lower values indicate greater predicted goal progress relative to that
representation's own zero-action baseline.
Prior-only runs execute $\mu_p$ as the sole candidate at $K=1$, $I=1$,
so PG-CEM and PG-MPPI are identical in this setting.
Each prior is trained separately on its frozen representation, as
described in Appendix~\ref{app:pg-training-details}.

On Cube, CGS yields the strongest prior rollout quality, reducing the
mean normalized goal-distance ratio from $0.262$ for LeWM to $0.116$,
while prior-only success is already near saturation.
On TwoRooms, all three matched priors reach $100\%$ prior-only success,
so initialization quality has little room to affect the final result.
On Reacher, the TS- and CGS-matched priors are weaker than the LeWM prior
at initialization. However, across the sample-budget sweep, these differences
narrow as sampling-based refinement incorporates more model-evaluated
candidates. Thus, weaker prior initialization can be partially corrected by
subsequent refinement, separating prior quality from the representation's
performance after model-based search.

\clearpage
\endgroup
\section{Proofs and Extended Theory}
\label{app:cgs_full_theory}

Sections~\ref{app:cgs_cosine_gram}--\ref{app:cgs_horizon_proof}
derive transition geometry and the planning Hessian from CGS, including
the square and rectangular action cases.
Section~\ref{app:cgs_mpc_background} fixes the quadratic planning problem
and Gaussian proposal coordinates. Sections~\ref{app:cgs_mppi_population}--\ref{app:cgs_ts_style}
then analyze MPPI, CEM, and gradient descent, respectively.
Section~\ref{app:cgs_prediction_constraint} discusses other training
objectives. The protocol for Figure~\ref{fig:cgs_mppi_mechanism} is reported
with the additional experimental results in
Appendix~\ref{app:cgs_figure3_audit}.

Throughout this appendix, we use the main context convention
\(C(\mathbf a)=\tfrac12\|z_H-z_g\|_2^2\), so
\(G_H=\nabla_{\mathbf a}^2C=\Gamma_H^\top\Gamma_H\) for linear rollouts.
The corresponding MPPI scale is \(s=\sigma_{\mathrm p}^2cH/\tau\).
Multiplying the cost and temperature by the same positive constant
preserves MPPI weights; positive cost rescaling leaves CEM elite selection
unchanged.

\subsection{From CGS Loss to Transition Geometry}
\label{app:cgs_cosine_gram}

The population statements below use two independent copies of a transition;
they do not assume that successive transitions within a training window are
independent. Applying the identities to temporally dependent pairs requires
a separate dependence argument. We also use the same local matrices
\(A,B\) for encoded transitions and predictor rollouts. Transferring an
encoder-side CGS bound to a separately parameterized predictor requires
agreement of their local transition maps; it does not follow from the CGS
loss alone.

Let \((z,a)\) and \((z',a')\) be independent draws from the same
transition distribution, with \(x=[z^\top,a^\top]^\top\),
\(x'=[z'^\top,a'^\top]^\top\), and the corresponding differences
\(\Delta z=Dz+Ba\) and \(\Delta z'=Dz'+Ba'\).
The population version of the practical cosine CGS objective is
\begin{equation}
    \mathcal L_{\mathrm{CGS}}^{\mathrm{pop}}
    =
    \mathbb E
    \left[
        \left(
        \frac{\Delta z^\top\Delta z'}
             {\|\Delta z\|_2\|\Delta z'\|_2}
        -
        \frac{a^\top a'}
             {\|a\|_2\|a'\|_2}
        \right)^2
    \right].
    \label{eq:app_cgs_cosine_objective}
\end{equation}
Assume the fixed-radius idealization
\begin{equation}
    \|\Delta z_t\|_2=r_\Delta,
    \qquad
    \|a_t\|_2=r_a,
    \qquad
    r_\Delta,r_a>0.
    \label{eq:app_cgs_fixed_radius}
\end{equation}
With \(c=r_\Delta^2/r_a^2\), recall the Gram loss from the main context
\begin{equation}
    \mathcal L_G(c)
    =\mathbb E
      \left[(\Delta z^\top\Delta z'-c\,a^\top a')^2\right].
    \label{eq:app_cgs_gram_objective}
\end{equation}
Then
\begin{equation}
    \mathcal L_{\mathrm{CGS}}^{\mathrm{pop}}
    =r_\Delta^{-4}\mathcal L_G(c).
    \label{eq:app_cgs_cosine_to_gram}
\end{equation}
Thus the population cosine objective is exactly a scaled Gram-matching
objective under Eq.~\eqref{eq:app_cgs_fixed_radius}.
Away from this idealization, an approximate relation can be justified
under quantitatively small relative fluctuations of the transition and
action norms around $r_\Delta$ and $r_a$, respectively; mere boundedness
of the norms does not ensure a small approximation error. The analysis below uses the fixed-radius model.

\subsubsection{Proof of Proposition~\ref{prop:cgs_joint_straightening}}
\label{app:cgs_joint_straightening_proof}

We give the calculation for general \(d,d_a\) under
Assumption~\ref{ass:cgs_joint_coverage}.
The feature coverage condition gives
\(\rho I_{d+d_a}\preceq\Sigma_x\).
For any \(v\in\mathbb R^{d+d_a}\), bounded features and Cauchy--Schwarz give
\[
  v^\top\Sigma_xv=\mathbb E[(v^\top x_t)^2]
\leq\|v\|_2^2\mathbb E\|x_t\|_2^2
\leq\bar\rho\|v\|_2^2.
\]
Thus \(\rho I_{d+d_a}\preceq\Sigma_x\preceq\bar\rho I_{d+d_a}\).
The analysis in the main context specializes to \(d=d_a\) with \(B\) invertible; the Gram
identity and spectral equivalence below do not require a square \(B\).

Let
\begin{equation}
    M_\Delta:=\begin{bmatrix}D&B\end{bmatrix}\in\mathbb R^{d\times(d+d_a)},
    \qquad
    P_c:=
    \begin{bmatrix}
        0_{d\times d}&0_{d\times d_a}\\
        0_{d_a\times d}&cI_{d_a}
    \end{bmatrix}.
    \label{eq:app_cgs_M_Pc}
\end{equation}
Because \(\Delta z_t=M_\Delta x_t\),
\begin{equation}
    \Delta z^\top\Delta z'-c\,a^\top a'
    =x^\top Qx',
    \qquad
    Q:=M_\Delta^\top M_\Delta-P_c.
    \label{eq:app_cgs_bilinear_form}
\end{equation}
For independent samples \(x,x'\),
\begin{equation}
    \mathcal L_G(c)
    =\mathbb E[(x^\top Qx')^2]
    =\operatorname{tr}(Q\Sigma_xQ\Sigma_x)
    =\left\|\Sigma_x^{1/2}Q\Sigma_x^{1/2}\right\|_F^2.
\label{eq:app_cgs_trace_identity}
\end{equation}
Moreover,
\begin{equation}
    Q=
    \begin{bmatrix}
        D^\top D&D^\top B\\
        B^\top D&B^\top B-cI_{d_a}
    \end{bmatrix}.
    \label{eq:app_cgs_Q_blocks}
\end{equation}
Since the singular
values of the linear map
\(Q\mapsto\Sigma_x^{1/2}Q\Sigma_x^{1/2}\) lie between \(\rho\) and
\(\bar\rho\),
\begin{equation}
    \rho^2\|Q\|_F^2
    \leq
    \left\|\Sigma_x^{1/2}Q\Sigma_x^{1/2}\right\|_F^2
    \leq
    \bar\rho^2\|Q\|_F^2.
\end{equation}
Finally,
\begin{equation}
    \|Q\|_F^2
    =\|D^\top D\|_F^2+2\|D^\top B\|_F^2
    +\|B^\top B-cI_{d_a}\|_F^2=\mathcal R(A,B).
\end{equation}
Together with Eq.~\eqref{eq:app_cgs_trace_identity}, this proves the
spectral equivalence in Proposition~\ref{prop:cgs_joint_straightening}.
Under positive definite joint coverage, \(\mathcal L_G(c)=0\) if and only if
\(Q=0\).
The upper-left block gives \(D^\top D=0\), hence \(D=0\) and \(A=I_d\);
the lower-right block gives \(B^\top B=cI_{d_a}\).  Conversely, these two
conditions make every block of \(Q\) zero.  When \(d\geq d_a\), such a
full-column-rank \(B\) exists, so zero loss is attainable.

\subsubsection{Proof of Corollary~\ref{cor:cgs_loss_to_epsilon}}
\label{app:cgs_loss_to_geometry}

Proposition~\ref{prop:cgs_joint_straightening} gives
\(\mathcal R(A,B)\leq\ell/\rho^2\). It first yields the Frobenius-norm
defect bounds
\begin{equation}
    \|D^\top D\|_F\leq\frac{\sqrt\ell}{\rho}, \quad
    \|D^\top B\|_F\leq\frac{\sqrt\ell}{\sqrt2\,\rho}, \quad
    \|B^\top B-cI_{d_a}\|_F\leq\frac{\sqrt\ell}{\rho}.
\end{equation}
Using
\(\|D\|_{\op}^2=\|D^\top D\|_{\op}\leq\|D^\top D\|_F\) and
\(\|\cdot\|_{\op}\leq\|\cdot\|_F\) gives the spectral-norm bounds
\begin{equation}
\begin{aligned}
    \|D\|_{\op}^2
    &\leq\frac{\sqrt\ell}{\rho},\\
    \|D^\top B\|_{\op}
    &\leq\frac{\sqrt\ell}{\sqrt2\,\rho},\\
    \|B^\top B-cI_{d_a}\|_{\op}
    &\leq\frac{\sqrt\ell}{\rho}.
\end{aligned}
    \label{eq:app_cgs_three_defect_bounds}
\end{equation}
Setting
\begin{equation}
    \varepsilon:=\frac{\sqrt\ell}{\rho}\max\{1,(2c)^{-1/2},c^{-1}\}.
    \label{eq:app_cgs_loss_to_epsilon}
\end{equation}
gives the three scaled conditions in
Corollary~\ref{cor:cgs_loss_to_epsilon}. Finally, when
\(\varepsilon<1\),
\begin{equation}
    c(1-\varepsilon)I_{d_a}
    \preceq B^\top B\preceq
    c(1+\varepsilon)I_{d_a},
\end{equation}
so \(B\) has full column rank and
\begin{equation}
    \kappa(B^\top B)
    \leq
    \frac{1+\varepsilon}{1-\varepsilon}.
    \label{eq:app_cgs_B_condition}
\end{equation}

The mixed defect is stronger than a bound obtained only from
\(\|D\|_{\op}^2\leq\varepsilon\).  It directly controls
\(\|D^\top B\|_{\op}=O(\sqrt c\,\varepsilon)\), suppressing the first-order
term \(B^\top DB\) in \(G_H\).  This is what produces
the \(O_H(\varepsilon)\) leading term in
Theorem~\ref{thm:cgs_horizon_curvature}.

\subsubsection{Coverage and dimension assumptions}

\paragraph{Joint coverage and state--action correlation.}
\label{app:cgs_joint_coverage}
Assumption~\ref{ass:cgs_joint_coverage} adapts the bounded-feature and
feature coverage conditions of \citet[Definition~2.1 and Assumption~2.2]{yin2022nearoptimal}
to the concatenated state--action feature \(x_t=[z_t^\top,a_t^\top]^\top\).
Our boundedness condition is \(\|x_t\|_2^2\leq\bar\rho\) almost surely,
which directly implies the second-moment upper bound
\(\Sigma_x\preceq\bar\rho I_{d+d_a}\).
The bound \(\rho I_{d+d_a}\preceq\Sigma_x\preceq\bar\rho I_{d+d_a}\) means that
\(\rho\|v\|_2^2\leq\mathbb E[(v^\top x_t)^2]\leq\bar\rho\|v\|_2^2\)
for every joint state--action direction \(v\).  Its lower bound requires
variation in every such direction; its upper bound controls the second
moment.  This is a population excitation condition of the kind used in
linear regression and system identification, where lower bounds on
covariate Gram matrices control estimation error \citep{matni2019tutorial}.
It is not a claim of empirical excitation for a dependent trajectory.

The proof requires neither independence of \(z_t\) and \(a_t\) nor zero
cross-covariance.  Define \(\Sigma_z=\mathbb E[z_tz_t^\top]\),
\(\Sigma_a=\mathbb E[a_ta_t^\top]\), and
\(\Sigma_{za}=\mathbb E[z_ta_t^\top]=\Sigma_{az}^\top\).
When \(\Sigma_z\succ0\), the Schur complement gives
\(\Sigma_x\succ0\) if and only if the linear action innovation has
positive definite second-moment matrix,
\begin{equation}
    \Sigma_{a\mid z}^{\mathrm{lin}}
    :=\Sigma_a-\Sigma_{az}\Sigma_z^{-1}\Sigma_{za}
    \succ0.
    \label{eq:app_linear_action_innovation}
\end{equation}
Here the innovation \(a_t-\Sigma_{az}\Sigma_z^{-1}z_t\) is the residual
after the best linear prediction of the action from the state.
This uncentered second moment need not equal a conditional covariance.
Thus state-dependent actions are allowed.
For standardized scalar \(z_t,a_t\) with correlation \(\gamma\), the
eigenvalues of \(\Sigma_x\) are \(1-|\gamma|\) and \(1+|\gamma|\).
Strong correlation is compatible with coverage, but makes \(\rho\) small
and weakens the loss-to-geometry bounds.  A fixed positive lower bound is
therefore a quantitative assumption, not merely the exclusion of exact
linear dependence.  Neither \(d=d_a\) nor invertibility of \(B\) implies it.

\paragraph{Singular joint second moment and the data subspace.}
\label{app:cgs_data_subspace}
For a transition distribution with \(\mathbb E\|x_t\|_2^2<\infty\), define
\(\mathcal S_x:=\operatorname{range}(\Sigma_x)
=(\ker\Sigma_x)^\perp\).
This is the smallest linear subspace containing \(x_t\) almost surely:
\(v\in\ker\Sigma_x\) if and only if
\(\mathbb E[(v^\top x_t)^2]=0\), or \(v^\top x_t=0\) almost surely.
It is a population subspace, distinct from a finite sample span.
Let \(U_x\in\mathbb R^{(d+d_a)\times r_x}\) have orthonormal columns spanning
\(\mathcal S_x\), where \(r_x=\operatorname{rank}(\Sigma_x)>0\), and set
\(\widetilde x_t=U_x^\top x_t\),
\(\widetilde\Sigma_x=U_x^\top\Sigma_xU_x\), and \(\Pi_x=U_xU_x^\top\).
Coverage on this subspace means
\(\rho I_{r_x}\preceq\widetilde\Sigma_x\preceq\bar\rho I_{r_x}\),
equivalently \(\rho\Pi_x\preceq\Sigma_x\preceq\bar\rho\Pi_x\).
For a fixed finite-dimensional distribution, some such constants always
exist on its nonzero support; a useful uniform lower bound is additional.

Writing \(\widetilde Q=U_x^\top Q U_x\), the trace identity gives
\(\mathcal L_G(c)=
\|\widetilde\Sigma_x^{1/2}\widetilde Q
\widetilde\Sigma_x^{1/2}\|_F^2\), hence
\(\rho^2\|\widetilde Q\|_F^2\leq\mathcal L_G(c)
\leq\bar\rho^2\|\widetilde Q\|_F^2\).
In particular, zero loss identifies \(U_x^\top Q U_x=0\), not necessarily
the separate blocks of \(Q\), because \(U_x\) can mix states and actions.
For example, take \(d=d_a=1\), \(c=1\), and \(z_t=a_t\) uniformly in
\(\{-1,1\}\).  With \(A=-1/2\) and \(B=1/2\), we have \(\Delta z_t=-a_t\) and
\(\mathcal L_G(1)=0\), although \(A\ne I_1\) and \(B^\top B\ne I_1\);
even \(B\) is invertible in this example.
Thus arbitrary projection onto \(\mathcal S_x\) does not preserve the
full straightening or rollout conclusions.  Assumption~\ref{ass:cgs_joint_coverage} retains
positive definite joint coverage in the original state--action coordinates.

\paragraph{Dimension requirement.}
Under joint coverage, the zero-loss geometry \(B^\top B=cI_{d_a}\)
is feasible whenever \(d\geq d_a\).  For \(d=d_a\), \(B/\sqrt c\) is
orthogonal; for \(d>d_a\), its columns are orthonormal and its image has
dimension \(d_a\).  If \(d<d_a\), rank deficiency gives
\(\|B^\top B-cI_{d_a}\|_F^2\geq c^2(d_a-d)\), so
\(\mathcal L_G(c)\geq\rho^2c^2(d_a-d)>0\).
Likewise, the \(\varepsilon\)-CGS regime with \(\varepsilon<1\) requires
\(d\geq d_a\), since \(B^\top B\succeq c(1-\varepsilon)I_{d_a}\).
This dimensional restriction concerns the action map \(B\), independently
of whether \(\Sigma_x\) is singular.

\subsection{Planning Hessian and Dimension Cases}
\label{app:cgs_horizon_proof}

Recall
\(\Gamma_H=[A^{H-1}B,\ldots,B]\),
\(G_H=\Gamma_H^\top\Gamma_H\), and
\(P_H=H^{-1}(\mathbf1_H\mathbf1_H^\top)\otimes I_{d_a}\), with
\(n_H=H \times d_a\). For an integer \(H\geq1\) and \(0\leq\varepsilon<1\), define
\begin{equation}
\begin{aligned}
    s_H(\varepsilon)
    &:=(1+\sqrt\varepsilon)^{H-1}-1,\\
    r_H(\varepsilon)
    &:=s_H(\varepsilon)-(H-1)\sqrt\varepsilon,\\
    \xi_H(\varepsilon)
    &:={}
    \varepsilon+2(H-1)\sqrt{1+\varepsilon}\,\varepsilon\\[-2pt]
    &\quad+(1+\varepsilon)
      \left(2r_H(\varepsilon)+s_H(\varepsilon)^2\right).
\end{aligned}
\label{eq:cgs_horizon_distortion_app}
\end{equation}

\subsubsection{Proof of Theorem~\ref{thm:cgs_horizon_curvature}}
\label{app:cgs_horizon_curvature_proof}

\paragraph{The ideal projector.}
The averaging matrix \(H^{-1}\mathbf1_H\mathbf1_H^\top\) is symmetric,
idempotent, and rank one. Hence \(P_H\) is an orthogonal projector of
rank \(d_a\), with range
\(\{\mathbf1_H\otimes u:u\in\mathbb R^{d_a}\}\) and kernel
\(\{\mathbf a:\sum_{t=0}^{H-1}a_t=0\}\).
When \(A=I_d\) and \(B^\top B=cI_{d_a}\), every block of \(G_H\) is \(cI_{d_a}\),
so \(G_H=cHP_H\). This gives the reference geometry with the correct
rank: equal curvature on cumulative-action directions and zero curvature
on cancelling directions. When CGS is approximate, the positive eigenspace
of \(G_H\) need not equal \(\operatorname{range}(P_H)\); the norm bound
below compares the matrices without assuming their eigenspaces coincide.

\paragraph{The CGS perturbation bound.}
Write \(A=I_d+D\).  For every integer \(k\geq0\),
\begin{equation}
    A^k=I_d+kD+R_k,
    \qquad
    R_k:=\sum_{j=2}^{k}\binom{k}{j}D^j.
    \label{eq:app_Ak_expansion}
\end{equation}
Since \(\|D\|_{\op}\leq\sqrt\varepsilon\), for
\(0\leq k\leq H-1\),
\begin{equation}
\begin{aligned}
    \|A^k-I_d\|_{\op}
    \leq(1+\sqrt\varepsilon)^k-1
    \leq s_H(\varepsilon),\\
    \|R_k\|_{\op}
    \leq(1+\sqrt\varepsilon)^k-1-k\sqrt\varepsilon
    \leq r_H(\varepsilon).
\end{aligned}
\label{eq:app_Ak_difference_bound}
\end{equation}
\label{eq:app_Ak_remainder_bound}
The third \(\varepsilon\)-CGS condition gives
\begin{equation}
    \|B\|_{\op}^2
    =\lambda_{\max}(B^\top B)
    \leq c(1+\varepsilon).
    \label{eq:app_B_operator_bound}
\end{equation}

Index the block rows and columns of \(G_H\) by action times
\(t,s\in\{0,\ldots,H-1\}\), and set \(p=H-1-t\) and \(q=H-1-s\).
This preserves the ordering used in the main context
\(\Gamma_H=[A^{H-1}B,\ldots,B]\), so
\begin{equation}
    [G_H]_{ts}=B^\top(A^p)^\top A^qB.
\end{equation}
Let \(F_k:=A^k-I_d=kD+R_k\).  Since every block of \(cHP_H\) is \(cI_{d_a}\),
\begin{equation}
    [G_H]_{ts}-cI_{d_a}
    ={}(B^\top B-cI_{d_a})+B^\top F_p^\top B+B^\top F_qB
       +B^\top F_p^\top F_qB.
\label{eq:app_horizon_block_expansion}
\end{equation}
Using Corollary~\ref{cor:cgs_loss_to_epsilon} and
Eq.~\eqref{eq:app_B_operator_bound},
\begin{equation}
\begin{aligned}
    \|B^\top F_p^\top B\|_{\op}
    &\leq p\|B\|_{\op}\|D^\top B\|_{\op}
       +\|B\|_{\op}^2\|R_p\|_{\op},\\
    \|B^\top F_qB\|_{\op}
    &\leq q\|B\|_{\op}\|D^\top B\|_{\op}
       +\|B\|_{\op}^2\|R_q\|_{\op},\\
    \|B^\top F_p^\top F_qB\|_{\op}
    &\leq\|B\|_{\op}^2\|F_p\|_{\op}\|F_q\|_{\op}.
\end{aligned}
\end{equation}
Therefore, uniformly in \(t,s\),
\begin{equation}
    \|[G_H]_{ts}-cI_{d_a}\|_{\op}
    \leq{}c\varepsilon
      +2(H-1)c\sqrt{1+\varepsilon}\,\varepsilon
      +c(1+\varepsilon)
        \left(2r_H(\varepsilon)+s_H(\varepsilon)^2\right)
    ={}c\,\xi_H(\varepsilon).
\label{eq:app_horizon_block_bound}
\end{equation}
Let \(E_H:=G_H-cHP_H\).  For a block vector
\(v=[v_0^\top,\ldots,v_{H-1}^\top]^\top\in\mathbb R^{n_H}\),
\begin{equation}
    \|E_Hv\|_2^2
    \leq
    \sum_{t=0}^{H-1}
    \left(c\xi_H(\varepsilon)
          \sum_{s=0}^{H-1}\|v_s\|_2\right)^2
    \leq c^2H^2\xi_H(\varepsilon)^2\|v\|_2^2.
\end{equation}
Hence
\(\|G_H-cHP_H\|_{\op}\leq cH\xi_H(\varepsilon)\).
To verify the leading coefficient, put \(h=H-1\). For fixed \(H\),
\(s_H=h\sqrt\varepsilon+\tfrac12h(h-1)\varepsilon
+O_H(\varepsilon^{3/2})\) and
\(r_H=\tfrac12h(h-1)\varepsilon+O_H(\varepsilon^{3/2})\).
The coefficient of \(\varepsilon\) in \(\xi_H\) is therefore
\(1+2h+h(h-1)+h^2=2H^2-3H+2\), proving the stated expansion.
This establishes the bound on \(\|G_H/(cH)-P_H\|_{\op}\) for any \(d\geq d_a\).

\paragraph{Comparison with temporal straightening and control isotropy.}
Retain
\(\|D\|_{\op}^2\leq\varepsilon\) and
\(\|B^\top B-cI_{d_a}\|_{\op}\leq c\varepsilon\).
These conditions already control the mixed drift--control defect indirectly:
\begin{equation}
    \|D^\top B\|_{\op}
    \leq
    \|D\|_{\op}\|B\|_{\op}
    \leq
    \sqrt{c(1+\varepsilon)\varepsilon}
    =
    O(\sqrt{\varepsilon}).
\end{equation}
Using this bound in Eq.~\eqref{eq:app_horizon_block_expansion} gives
\begin{equation}
\begin{aligned}
    \left\|\frac{G_H}{cH}-P_H\right\|_{\op}
    &\leq\varepsilon+(1+\varepsilon)(2s_H+s_H^2)\\
    &=2(H-1)\sqrt\varepsilon+O_H(\varepsilon).
\end{aligned}
\label{eq:app_cgs_temporal_only_hessian_bound}
\end{equation}
For \(H\geq2\), this square-root order is generally unavoidable under
these two conditions alone.

CGS directly controls the same mixed defect at the geometric error scale,
\begin{equation}
    \|D^\top B\|_{\op}\leq \sqrt{c}\,\varepsilon,
\end{equation}
and hence
\[
    \|B^\top D B\|_{\op}
    \leq
    \|D^\top B\|_{\op}\|B\|_{\op}
    \leq
    c\sqrt{1+\varepsilon}\,\varepsilon.
\]
The linear drift--control terms are therefore \(O(\varepsilon)\), yielding
the \(O_H(\varepsilon)\) Hessian perturbation bound in
Theorem~\ref{thm:cgs_horizon_curvature}. This compares the geometric
consequences of the two bounds under matched drift and control-isotropy
tolerances; it does not identify the CGS and temporal-straightening
training objectives. At \(H=1\), both bounds reduce to the
control-isotropy error \(\varepsilon\).

\paragraph{Controllability Gramian and effective conditioning.}
Recall the finite-horizon controllability Gramian and extend the
effective condition number from the main context to any nonzero matrix:
\begin{equation}
\begin{aligned}
    W_H&=\Gamma_H\Gamma_H^\top
      =\sum_{k=0}^{H-1}A^kBB^\top(A^\top)^k,\\[-2pt]
    \kappa_{\mathrm{eff}}(M)&:=
      \sigma_{\max}(M)/\sigma_{\min}^{+}(M).
\end{aligned}
    \label{eq:app_cgs_gramian_kappa_eff}
\end{equation}
Here \(M\neq0\), \(\sigma_{\min}^{+}(M)\) is its smallest positive
singular value, and \(\kappa\) denotes the Euclidean condition number
on a nonsingular restriction. For symmetric positive semidefinite \(M\),
this ratio also equals \(\lambda_{\max}(M)/\lambda_{\min}^{+}(M)\).
The matrices \(G_H=\Gamma_H^\top\Gamma_H\) and
\(W_H=\Gamma_H\Gamma_H^\top\) have the same positive eigenvalues and
singular values, namely the squared positive singular values of
\(\Gamma_H\). Hence
\(\kappa_{\mathrm{eff}}(G_H)=\kappa_{\mathrm{eff}}(W_H)\).
In particular,
\(\kappa_{\mathrm{eff}}(G_H)=\kappa_{\mathrm{eff}}(\Gamma_H)^2\);
conditioning the horizon map itself therefore takes the square root of
the Gram-matrix condition-number bound.

\subsubsection{Square and rectangular action maps}
\label{app:cgs_low_dim_actions}

We distinguish the dimension cases used by all three planners below,
extending Remark~\ref{rem:cgs_low_dim_actions}. Throughout this discussion,
assume \(d\geq d_a\), the three \(\varepsilon\)-CGS bounds with
\(0\leq\varepsilon<1\), and \(\xi_H<1\).
In particular, \(B^\top B\succeq c(1-\varepsilon)I_{d_a}\), so \(B\) has
full column rank.

\paragraph{State and action subspaces.}
Define the finite-horizon controllable state subspace and the action
subspace that affects the terminal state by
\begin{equation}
\begin{aligned}
    \mathcal C_H&:=\operatorname{range}(W_H)
    =\operatorname{span}\{\operatorname{range}(A^kB):0\leq k<H\},\\
    \mathcal T_H&:=\operatorname{range}(G_H)
    =\operatorname{range}(\Gamma_H^\top).
\end{aligned}
    \label{eq:app_cgs_controllable_subspaces}
\end{equation}
Their common dimension is
\(r_{\mathrm{ctrl}}:=\operatorname{rank}(\Gamma_H)\), with
\(d_a\leq r_{\mathrm{ctrl}}\leq\min\{d,n_H\}\).
These subspaces are determined by the dynamics and horizon, and differ
from the data subspace \(\mathcal S_x\).
A thin singular value decomposition
\(\Gamma_H=U_\Gamma\Sigma_\Gamma V_\Gamma^\top\), where
\(\Sigma_\Gamma=\diag(\sigma_1(\Gamma_H),\ldots,
\sigma_{r_{\mathrm{ctrl}}}(\Gamma_H))\) retains its positive singular values,
gives \(W_H=U_\Gamma\Sigma_\Gamma^2U_\Gamma^\top\) and
\(G_H=V_\Gamma\Sigma_\Gamma^2V_\Gamma^\top\).
Thus both restrictions are positive definite and
\(\kappa_{\mathrm{eff}}(G_H)
=\kappa(G_H|_{\mathcal T_H})
=\kappa(W_H|_{\mathcal C_H})\).
Restricting to these subspaces removes the exact null directions.

If $d=d_a$, full column rank makes $B$ invertible, and the last block
of $\Gamma_H$ gives $r_{\rm ctrl}=d_a$.
If $d>d_a$, $r_{\rm ctrl}$ can exceed $d_a$ because propagation through
$A$ can reach state directions outside $\operatorname{range}(B)$.

\paragraph{Bounds inherited from CGS.}
The proof of Theorem~\ref{thm:cgs_horizon_curvature} applies to rectangular
\(B\), giving \(\|G_H/(cH)-P_H\|_{\op}\leq\xi_H\).
To extend the notation from the main context to rectangular action maps, let
\(\mathcal S_H\) be a subspace spanned by orthonormal eigenvectors
associated with the largest \(r=d_a\) eigenvalues of \(G_H\).
In the square case this is exactly the subspace used in the main context
\(\mathcal S_H=\operatorname{range}(G_H)=\mathcal T_H\); in the
rectangular case \(\mathcal S_H\) can be a strict subspace of
\(\mathcal T_H\).
In orthonormal coordinates on \(\mathcal S_H\),
Theorem~\ref{thm:cgs_horizon_curvature} and Weyl's inequality give
\begin{equation}
    cH(1-\xi_H)I_r
    \preceq G_H|_{\mathcal S_H}
    \preceq cH(1+\xi_H)I_r,
    \qquad
    \kappa(G_H|_{\mathcal S_H})
    \leq\frac{1+\xi_H}{1-\xi_H}.
    \label{eq:app_cgs_low_dim_dominant_bound}
\end{equation}
Let $\Pi_H$ denote the orthogonal projector onto $\mathcal S_H$.
Unlike the fixed ideal projector $P_H$, $\Pi_H$ depends on the learned
dynamics; the two agree at ideal CGS geometry. These symbols are used
throughout the planner analysis.
When $r_{\rm ctrl}=r$, $\mathcal S_H=\mathcal T_H$ and every direction
that contributes to the planning cost is included. When $r_{\rm ctrl}>r$,
the spectral bound above covers only the leading $r$ directions.
The corresponding leading eigenspace of \(W_H\) has the same spectral
bounds, by the singular value decomposition above. This proves the
geometric assertions of Remark~\ref{rem:cgs_low_dim_actions}.

\paragraph{Bounds on all controllable directions.}
The following conditions extend the bound from \(\mathcal S_H\) to
\(\mathcal T_H\). Let \(\Pi_{\mathcal C_H}\) denote the orthogonal
projector onto \(\mathcal C_H\).
\begin{corollary}[Conditioning on the full controllable subspace]
\label{cor:cgs_full_controllable_bound}
Under the conditions of this discussion, if
\(r_{\mathrm{ctrl}}=d_a\), then
\begin{equation}
    \kappa_{\mathrm{eff}}(G_H)
    =\kappa(W_H|_{\mathcal C_H})
    \leq\frac{1+\xi_H}{1-\xi_H}.
    \label{eq:app_cgs_full_rank_da_bound}
\end{equation}
A sufficient condition is
\(A\,\operatorname{range}(B)\subseteq\operatorname{range}(B)\).
More generally, if a quantitative controllability bound
\(W_H\succeq\gamma_H cH\Pi_{\mathcal C_H}\) holds with a specified
\(\gamma_H>0\), then
\(\kappa_{\mathrm{eff}}(G_H)\leq(1+\xi_H)/\gamma_H\).
\end{corollary}

\begin{proof}
If \(r_{\mathrm{ctrl}}=d_a\), the leading \(d_a\) eigenvalues are all
of the positive eigenvalues, so
Eq.~\eqref{eq:app_cgs_low_dim_dominant_bound} proves
Eq.~\eqref{eq:app_cgs_full_rank_da_bound}.
Under the stated invariance condition, induction gives
\(\operatorname{range}(A^kB)\subseteq\operatorname{range}(B)\) for all
\(k\). Since \(\Gamma_H\) contains \(B\) as its last block and \(B\)
has full column rank, \(r_{\mathrm{ctrl}}=d_a\).
For the general bound, quantitative controllability gives
\(\lambda_{\min}^{+}(W_H)\geq\gamma_H cH\), while
Theorem~\ref{thm:cgs_horizon_curvature} and Weyl's inequality give
\(\lambda_{\max}(W_H)=\lambda_1(G_H)\leq cH(1+\xi_H)\).
Taking the ratio proves the claim.
\end{proof}

\begin{remark}[Additional controllable directions]
\label{rem:cgs_additional_directions}
When \(r_{\mathrm{ctrl}}>d_a\), the remaining positive eigenvalues satisfy
\(0<\lambda_j(G_H)\leq cH\xi_H\) for
\(d_a<j\leq r_{\mathrm{ctrl}}\).
Controllability makes these eigenvalues positive, but a uniform lower
bound requires further information, even after restriction to
\(\mathcal C_H\).
For example, let \(d=2\), \(d_a=1\), \(H=2\), \(c=1\), and
\(0<\delta<1/4\). With \(e_1,e_2\) the standard basis, take
\(B=e_1\) and \(A=I_2+\delta e_2e_1^\top\).
Then \(\|D\|_{\op}^2=\delta^2\), \(D^\top B=0\), and \(B^\top B=1\),
so all three CGS bounds hold with \(\varepsilon=\delta^2\) and
\(\xi_2(\varepsilon)=\delta^2(1+\sqrt{1+\delta^2})^2<1\).
Direct calculation gives
\begin{equation}
    \Gamma_2=\begin{bmatrix}1&1\\\delta&0\end{bmatrix},
    \qquad
    G_2=\begin{bmatrix}1+\delta^2&1\\1&1\end{bmatrix},
    \qquad
    W_2=\begin{bmatrix}2&\delta\\\delta&\delta^2\end{bmatrix}.
    \label{eq:app_cgs_weak_direction_example}
\end{equation}
For every \(\delta>0\), both \(G_2\) and \(W_2\) are positive definite,
so \(\mathcal C_2=\mathbb R^2\). Their eigenvalues are
\(\lambda_\pm=(2+\delta^2\pm\sqrt{4+\delta^4})/2\).
As \(\delta\downarrow0\), \(\lambda_+\to2\) and
\(\lambda_-\sim\delta^2/2\), hence
\(\kappa_{\mathrm{eff}}(G_2)\sim4/\delta^2\), while
\(\|G_2/2-P_2\|_{\op}=\delta^2/2\).
The additional reachable direction becomes weak as the dynamics approach
ideal CGS geometry, consistently with the dominant-spectrum bound.
\end{remark}

\paragraph{Relation to temporal straightening.}
The controllable-subspace formulation follows
\citet[Remark~4.5 and Appendix~C.2]{wang2026ts}. Their Remark~C.6 also
notes that a lower bound on the smallest positive Gramian eigenvalue
requires additional controllability assumptions.
In the square case, their lower-bound argument uses
\(\|B^\top u\|_2\geq\sigma_{\min}(B)\|u\|_2\) for every state
vector \(u\). For rectangular \(B\), this inequality holds on
\(\operatorname{range}(B)\), which may be smaller than \(\mathcal C_H\).
Thus the shared nonzero-spectrum identity extends to all dimensions,
while quantitative conditioning on the full controllable subspace uses
conditions such as those in Corollary~\ref{cor:cgs_full_controllable_bound}.

\paragraph{Implications for the planner guarantees.}
For either dimension case, the MPPI, CEM, and gradient-descent
statements below use \(r=d_a\) and the leading eigenspace
\(\mathcal S_H\) defined above. Their respective hypotheses are unchanged;
Eq.~\eqref{eq:app_cgs_low_dim_dominant_bound} supplies the spectral bounds
on these directions. CEM additionally requires the perturbation and mean
neighborhood conditions in Proposition~\ref{prop:cgs_cem_population_full}. When \(r_{\mathrm{ctrl}}=d_a\),
\(\mathcal S_H=\mathcal T_H\), so these guarantees cover all action
directions that affect the terminal state.
Under the more general quantitative controllability condition, the exact
MPPI map in Proposition~\ref{prop:cgs_mppi_population_full} also gives
\(\|J_{\mathrm M}|_{\mathcal T_H}\|_{\op}\leq(1+s\gamma_H)^{-1}\), where
\(s=\sigma_{\mathrm p}^2cH/\tau\): its eigenvalues on \(\mathcal T_H\) are
\((1+s\lambda_j(G_H)/(cH))^{-1}\).

\subsection{Quadratic Planning and Gaussian Proposals}
\label{app:cgs_mpc_background}

\subsubsection{Planning cost and nonlinear rollouts}
\label{app:cgs_quadratic_nonlinear}

\paragraph{Exact quadratic expansion.}
Let \(\mathbf a^\star\) be any unconstrained minimizer of
\(C(\mathbf a)=\tfrac12\|A^Hz_0+\Gamma_H\mathbf a-z_g\|_2^2\).  The normal
equation gives
\begin{equation}
    \Gamma_H^\top
    (A^Hz_0+\Gamma_H\mathbf a^\star-z_g)=0.
\end{equation}
Therefore
\begin{equation}
    C(\mathbf a)
    =C(\mathbf a^\star)
     +\tfrac12(\mathbf a-\mathbf a^\star)^\top
       G_H
       (\mathbf a-\mathbf a^\star).
    \label{eq:app_exact_quadratic_expansion}
\end{equation}
This identity is exact even when the goal is not reachable.

\paragraph{Planning cost gap.}
For a candidate sequence $\mathbf a$, define
$\Delta(\mathbf a):=C(\mathbf a)-C(\mathbf a^\star)$.
Equation~\eqref{eq:app_exact_quadratic_expansion} gives
$\Delta(\mathbf a)=\tfrac12\|G_H^{1/2}(\mathbf a-\mathbf a^\star)\|_2^2$.
For a proposal mean $\mu_i$, write $\Delta_i:=\Delta(\mu_i)$.
The minimum $C(\mathbf a^\star)$ need not be zero.

\paragraph{Nonlinear rollout.}
In practice, for a nonlinear terminal predictor \(z_H=\widehat F_H(\mathbf a)\),
\begin{equation}
    \nabla_{\mathbf a}^2C
    =J_{\widehat F_H}^\top J_{\widehat F_H}
     +\sum_{j=1}^d
       (\widehat F_H(\mathbf a)-z_g)_j
       \nabla_{\mathbf a}^2(\widehat F_H)_j.
    \label{eq:app_nonlinear_cost_hessian}
\end{equation}
Nonlinear rollouts introduce an additional residual-weighted second-order
term beyond \(J_{\widehat F_H}^\top J_{\widehat F_H}\).
This term need not be positive semidefinite and can alter the Hessian
spectrum or introduce negative curvature, complicating the relationship
between control geometry and sampling-based planner updates.
Our analysis focuses on linear dynamical models; extending it to nonlinear
rollouts is an exciting direction for future work.

\paragraph{Spectral assumptions for the planner bounds.}
The CGS-dependent results below assume the three $\varepsilon$-CGS bounds
with $0\leq\varepsilon<1$ from Corollary~\ref{cor:cgs_loss_to_epsilon}.
We use $M_H=G_H/(cH)$, $r=d_a$, and the subspaces of
Section~\ref{app:cgs_low_dim_actions}; each result states its additional
requirement on $\xi_H$.

\subsubsection{Gaussian proposals and update rules}

Let \(\mathsf P_i:=\mathcal N(\mu_i,\Sigma_i)\) be the proposal distribution
at iteration \(i\).
Draw \(K\) independent action sequences
\(\mathbf a^{(1)},\ldots,\mathbf a^{(K)}\sim\mathsf P_i\).

\paragraph{Proposal whitening.}
For a general positive-definite covariance \(\Sigma\), set
\begin{equation}
    X:=\Sigma^{-1/2}(\mathbf a-\mathbf a^\star),
    \qquad
    \nu:=\Sigma^{-1/2}(\mu-\mathbf a^\star).
\end{equation}
In these proposal-whitened coordinates, the planning Hessian becomes
\begin{equation}
    \Sigma^{1/2}G_H\Sigma^{1/2}.
    \label{eq:app_proposal_whitened_hessian}
\end{equation}
For a general proposal, isotropy relevant to the sampling update is therefore
measured by this proposal-whitened Hessian, not by \(G_H\) alone.
Even an isotropic planning cost can become anisotropic in these coordinates
if the proposal scales directions unequally. The MPPI and CEM bounds in Sections~\ref{app:cgs_mppi_population}
and~\ref{app:cgs_cem_population} use $\Sigma=\sigma_{\mathrm p}^2I_{n_H}$, where
$\sigma_{\mathrm p}>0$ is the sampling standard deviation, as in the main
text, so that Theorem~\ref{thm:cgs_horizon_curvature} applies directly.

\paragraph{MPPI.}
Following path-integral control and MPPI
\citep{kappen2005path,theodorou2010path,williams2017mppi}, for temperature
\(\tau>0\), MPPI assigns soft weights
\begin{equation}
    w_k
    :=\exp\!\left(-\frac{C(\mathbf a^{(k)})}{\tau}\right)
\end{equation}
and updates the proposal mean by the self-normalized estimator
\begin{equation}
    \widehat\mu_{i+1}^{\mathrm M}
    :=\frac{\sum_{k=1}^Kw_k\mathbf a^{(k)}}
            {\sum_{k=1}^Kw_k}.
    \label{eq:app_mppi_empirical_update}
\end{equation}
Adding a constant to \(C\) does not change the normalized weights.  The
population map is
\begin{equation}
    \mu_{i+1}^{\mathrm M}
    :=\frac{\mathbb E_{\mathsf P_i}[w(\mathbf a)\mathbf a]}
            {\mathbb E_{\mathsf P_i}[w(\mathbf a)]}.
    \label{eq:app_mppi_population_update_background}
\end{equation}
The analysis below conditions on the proposal covariance.

\paragraph{CEM.}
Following the cross-entropy method
\citep{rubinstein1999crossentropy,deboer2005cem}, for an elite fraction
\(\alpha\in(0,1)\), let
\(K_{\rm elite}=\alpha K\in\mathbb N\) be the number of samples with lowest cost.  The empirical
mean and covariance updates are
\begin{equation}
\begin{aligned}
    \widehat\mu_{i+1}^{\mathrm C}
    &:=\frac{1}{K_{\rm elite}}\sum_{j\in\mathcal E_i}\mathbf a^{(j)},\\[-2pt]
    \widehat\Sigma_{i+1}^{\mathrm C}
    &:=\frac{1}{K_{\rm elite}}\sum_{j\in\mathcal E_i}
       (\mathbf a^{(j)}-\widehat\mu_{i+1}^{\mathrm C})
       (\mathbf a^{(j)}-\widehat\mu_{i+1}^{\mathrm C})^\top.
\end{aligned}
    \label{eq:app_cem_empirical_update}
\end{equation}
where \(\mathcal E_i\) is the elite index set.  The theory conditions on the
current covariance and analyzes the population mean map.  A full coupled
mean--covariance convergence result is not claimed.

\subsubsection{A shared bound for repeated mean updates}

\begin{corollary}[Mean error with finite samples and updates]
\label{cor:cgs_finite_budget_app}
Fix $A,B,z_0,z_g$, the proposal covariance $\Sigma$, and the projector
$\Pi_H$. Let $\mu_i$ be the proposal mean after $i$ updates and define
$\nu_i:=\Sigma^{-1/2}(\mu_i-\mathbf a^\star)$ and $e_i:=\Pi_H\nu_i$.
For fixed $I\geq1$ and $\zeta\in(0,1)$, define
\begin{equation}
    \mathcal D:=\{\nu:\|\nu\|_2\leq R_\nu,\ \|\Pi_H\nu\|_2\leq R_e\},
    \qquad
    T_{\rm exit}:=\inf\{i\geq0:\nu_i\notin\mathcal D\},
    \label{eq:app_planner_first_exit}
\end{equation}
with $\inf\varnothing=\infty$. Let $\mathcal F_i$ be the history before
batch $i$ is drawn, so $\nu_i$ is $\mathcal F_i$-measurable.
Let the population map $T$ satisfy
$\|\Pi_HT(\nu)\|_2\leq q\|\Pi_H\nu\|_2$, with $0<q<1$, for $\nu\in\mathcal D$.
Suppose each fresh $K$-sample update obeys
$\|\Pi_H(\widehat T_{K,i}(\nu)-T(\nu))\|_2\leq\epsilon_K(\zeta/I)$
with conditional probability at least $1-\zeta/I$ given $\mathcal F_i$,
uniformly for the current mean in $\mathcal D$.
The one-step sampling estimates used below take $R=R_\nu$.
Assume $\|e_0\|_2\leq R_e$ and $\epsilon_K(\zeta/I)\leq(1-q)R_e$.
There is an event of probability at least $1-\zeta$ on which every update
with $0\leq i<\min(I,T_{\rm exit})$ satisfies
$\|e_{i+1}\|_2\leq q\|e_i\|_2+\epsilon_K(\zeta/I)$,
and, whenever $I\leq T_{\rm exit}$,
\begin{equation}
    \|e_I\|_2
    \leq q^I\|e_0\|_2
      +\frac{1-q^I}{1-q}\epsilon_K(\zeta/I).
    \label{eq:app_finite_iteration_bound}
\end{equation}
For the budget statement, let $0<\epsilon_{\rm mean}<\|e_0\|_2$ be a
mean-error tolerance, distinct from the planning-cost tolerance in the main context
$\epsilon_{\rm tar}$. If
$\epsilon_K(\eta)\leq C_{\mathrm{samp}}\sqrt{[r+\log(1/\eta)]/K}$,
then $\|e_I\|_2\leq\epsilon_{\rm mean}$ is ensured on the same event whenever
$I\leq T_{\rm exit}$, up to logarithmic factors, by
\begin{equation}
    I\gtrsim\frac{\log(2\|e_0\|_2/\epsilon_{\rm mean})}{-\log q},
    \qquad
    K\gtrsim\frac{C_{\mathrm{samp}}^2}{(1-q)^2\epsilon_{\rm mean}^2}
      \left(r+\log\frac I\zeta\right).
    \label{eq:app_cgs_KI_budget}
\end{equation}
\end{corollary}

\begin{proof}
Let $B_i$ be the event that the batch-$i$ sampling bound fails.
On $\{i<T_{\rm exit}\}$, the conditional failure probability given
$\mathcal F_i$ is at most $\zeta/I$. Since this event is
$\mathcal F_i$-measurable, the tower property and a union bound give
\begin{equation}
\begin{aligned}
    \Prob\!\left(\bigcup_{i=0}^{I-1}
        (\{i<T_{\rm exit}\}\cap B_i)\right)
    &\leq\sum_{i=0}^{I-1}
        \E\!\left[\mathbf1_{\{i<T_{\rm exit}\}}
        \Prob(B_i\mid\mathcal F_i)\right]
    \leq\zeta.
\end{aligned}
    \label{eq:app_planner_stopped_failure}
\end{equation}
On the complementary event,
$\|e_{i+1}\|_2\leq q\|e_i\|_2+\epsilon_K(\zeta/I)$
for every $0\leq i<\min(I,T_{\rm exit})$.
The invariance condition keeps the active iterates inside the radius-$R_e$
ball through these updates, including the output of the last such update.
Thus, on this event, exit cannot be caused by the active-radius condition
alone; full-mean boundedness is not implied.
Unrolling the recursion gives
Eq.~\eqref{eq:app_finite_iteration_bound} whenever $I\leq T_{\rm exit}$.
It suffices to require
$q^I\|e_0\|_2\leq\epsilon_{\rm mean}/2$ and
$\epsilon_K(\zeta/I)\leq(1-q)\epsilon_{\rm mean}/2$;
solving these inequalities gives Eq.~\eqref{eq:app_cgs_KI_budget}.
\end{proof}

\paragraph{Probability interpretation.}
Let $E_I:=\{I\leq T_{\rm exit}\}$ and let $G_I$ denote the event that
Eq.~\eqref{eq:app_finite_iteration_bound} holds. The guarantee is
$\Prob(E_I\cap G_I^c)\leq\zeta$, not a conditional success probability
$\Prob(G_I\mid E_I)\geq1-\zeta$.
If $\Prob(E_I)>0$, it only implies
$\Prob(G_I^c\mid E_I)\leq\zeta/\Prob(E_I)$.
When full-mean boundedness holds almost surely throughout the input
iterations, active-radius invariance yields the unconditional
$1-\zeta$ error guarantee.
No exit-probability bound is asserted: bounding only $e_i$ does not control
drift in the other coordinates. The condition $I\leq T_{\rm exit}$ concerns
$\nu_0,\ldots,\nu_{I-1}$ and does not require $\nu_I\in\mathcal D$.
Updating the proposal covariance, changing the dynamics, or replanning
from a new initial state requires corresponding uniform bounds.

\subsection{MPPI: Mean Update and Planning Cost}
\label{app:cgs_mppi_population}

We use $r$, $\mathcal S_H$, and $\Pi_H$ as defined in
Section~\ref{app:cgs_low_dim_actions}; the formulas below apply to both
square and rectangular action maps.

\subsubsection{Population mean update}
\label{app:cgs_mppi_population_proof}

\begin{proposition}[Exact MPPI population map]
\label{prop:cgs_mppi_population_full}
Let \(X\sim\mathcal N(\nu,I_{n_H})\),
\(M_H=G_H/(cH)\), and
\(s=\sigma_{\mathrm p}^2cH/\tau\).  Define the tilted law
\begin{equation}
    \pi_\nu^{\mathrm M}(x)
    \propto
    \exp\!\left(-\frac{s}{2}x^\top M_Hx\right)
    \mathcal N(x;\nu,I_{n_H}).
    \label{eq:app_mppi_tilted_law}
\end{equation}
Then
\begin{equation}
\begin{aligned}
    T_{\mathrm M}(\nu)&:=\mathbb E_{\pi_\nu^{\mathrm M}}[X]
      =(I_{n_H}+sM_H)^{-1}\nu,\\[-2pt]
    J_{\mathrm M}&=(I_{n_H}+sM_H)^{-1}.
\end{aligned}
    \label{eq:app_mppi_exact_map}
\end{equation}
If \(\xi_H<1\), then on the top-\(r\) eigenspace \(\mathcal S_H\),
\begin{equation}
    q_{\mathrm M}:=\|J_{\mathrm M}|_{\mathcal S_H}\|_{\op}
    =\frac{1}{1+(\sigma_{\mathrm p}^2/\tau)\lambda_r(G_H)}
    \leq\frac{1}{1+s(1-\xi_H)},
    \label{eq:app_mppi_contraction}
\end{equation}
and
\begin{equation}
    \kappa(J_{\mathrm M}|_{\mathcal S_H})
    \leq
    \frac{1+s(1+\xi_H)}{1+s(1-\xi_H)}.
    \label{eq:app_mppi_update_condition}
\end{equation}
At ideal geometry,
\begin{equation}
    J_{\mathrm M}^0
    =\frac{1}{1+s}P_H+(I_{n_H}-P_H),
    \label{eq:app_mppi_ideal_jacobian}
\end{equation}
and
\begin{equation}
    \|J_{\mathrm M}-J_{\mathrm M}^0\|_{\op}
    \leq s\xi_H.
    \label{eq:app_mppi_jacobian_perturbation}
\end{equation}
\end{proposition}

\begin{proof}
The log unnormalized density in Eq.~\eqref{eq:app_mppi_tilted_law} is
\begin{equation}
    -\frac12\|x-\nu\|_2^2-\frac{s}{2}x^\top M_Hx.
\end{equation}
Completing the square gives a Gaussian with precision \(I_{n_H}+sM_H\) and mean
\((I_{n_H}+sM_H)^{-1}\nu\), proving
Eq.~\eqref{eq:app_mppi_exact_map}.  Since \(J_{\mathrm M}\) is a matrix
function of \(M_H\), they share eigenvectors.  On \(\mathcal S_H\),
Theorem~\ref{thm:cgs_horizon_curvature} gives
\begin{equation}
    1-\xi_H\leq\lambda_j(M_H)\leq1+\xi_H.
\end{equation}
Applying \(u\mapsto(1+su)^{-1}\) proves
Eqs.~\eqref{eq:app_mppi_contraction} and
\eqref{eq:app_mppi_update_condition}; the smallest dominant eigenvalue also
gives \(q_{\mathrm M}=[1+(\sigma_{\mathrm p}^2/\tau)\lambda_r(G_H)]^{-1}\).
Substituting \(M_H=P_H\)
and using \(P_H^2=P_H\) proves
Eq.~\eqref{eq:app_mppi_ideal_jacobian}.  Finally, the resolvent identity
gives
\begin{equation}
    J_{\mathrm M}-J_{\mathrm M}^0
    =(I_{n_H}+sM_H)^{-1}s(P_H-M_H)(I_{n_H}+sP_H)^{-1}.
\end{equation}
Both inverse factors have operator norm at most one because
\(M_H,P_H\succeq0\), which proves
Eq.~\eqref{eq:app_mppi_jacobian_perturbation}.
\end{proof}

For \(d=d_a\), \(r=d\) and
\(\mathcal S_H=\operatorname{range}(G_H)\), so this proves
Theorem~\ref{thm:cgs_mppi_population} with the same \(s\) and
\(q_{\mathrm M}\) as in the main context.

The contraction depends on both shape and scale.  CGS controls spectral
spread through \(\xi_H\), but cosine matching does not identify \(c\).
Therefore \(\tau\) and \(\sigma_{\mathrm p}^2\) must still be calibrated so that
\(s=\sigma_{\mathrm p}^2cH/\tau\) is non-negligible.

\subsubsection{Sampling error in one update}
\label{app:cgs_finite_sample}

Assume $\xi_H<1$ and $\|\nu\|_2\leq R$.
All samples in a batch are independent.

\begin{proposition}[Sampling error in one MPPI update]
\label{prop:cgs_finite_k_concentration}
For MPPI, draw \(X_1,\ldots,X_K\sim\mathcal N(\nu,I_{n_H})\) and define
\begin{equation}
\begin{aligned}
    \widetilde w_k&:=\exp\!\left(-\frac{s}{2}X_k^\top M_HX_k\right),\\[-2pt]
    \widehat T_{\mathrm M,K}
    &:=\frac{\sum_{k=1}^K\widetilde w_kX_k}{\sum_{k=1}^K\widetilde w_k}.
\end{aligned}
\end{equation}
These weights use the cost gap $\Delta$ in place of $C$, which leaves
normalized weights unchanged. Let
\begin{equation}
    \underline Z_{\mathrm M}(R)
    :=
    \frac{e^{-R^2/2}}
    {(1+s(1+\xi_H))^{r/2}
     (1+s\xi_H)^{(n_H-r)/2}}.
    \label{eq:app_mppi_normalizer_lower_bound}
\end{equation}
There is a universal constant \(C_{\mathrm{conc}}>0\) such that, for every
\(\zeta\in(0,1)\), if
\begin{equation}
    K\geq
    \frac{8}{\underline Z_{\mathrm M}(R)^2}
    \log\frac4\zeta,
    \label{eq:app_mppi_minimum_K}
\end{equation}
then, with probability at least \(1-\zeta\),
\begin{equation}
    \left\|
        \Pi_H(\widehat T_{\mathrm M,K}-T_{\mathrm M}(\nu))
    \right\|_2
    \leq
    \frac{C_{\mathrm{conc}}(1+R)}{\underline Z_{\mathrm M}(R)}
    \sqrt{\frac{r\log(4r/\zeta)}{K}}.
    \label{eq:app_mppi_finite_K_bound}
\end{equation}

\end{proposition}

\begin{proof}
Let \(Z:=\mathbb E[\widetilde w_1]\).  A Gaussian integral gives
\begin{equation}
\begin{aligned}
    Z
    =\det(I_{n_H}+sM_H)^{-1/2}\exp\!\left(
      -\tfrac12\nu^\top sM_H(I_{n_H}+sM_H)^{-1}\nu\right).
\end{aligned}
    \label{eq:app_mppi_normalizer_exact}
\end{equation}
The top \(r\) eigenvalues of \(M_H\) are at most \(1+\xi_H\), the
remaining eigenvalues are at most \(\xi_H\), and
\(0\preceq sM_H(I_{n_H}+sM_H)^{-1}\preceq I_{n_H}\).  Therefore
\(Z\geq\underline Z_{\mathrm M}(R)\).  Since \(0<\widetilde w_k\leq1\),
Hoeffding's inequality~\citep{hoeffding1963probability} and
Eq.~\eqref{eq:app_mppi_minimum_K} imply, with
probability at least \(1-\zeta/2\),
\begin{equation}
    \left|K^{-1}\sum_{k=1}^K\widetilde w_k-Z\right|
    \leq\sqrt{\frac{\log(4/\zeta)}{2K}}
    \leq\frac{\underline Z_{\mathrm M}(R)}{2}.
    \label{eq:app_mppi_denominator_event}
\end{equation}

In an orthonormal basis of \(\mathcal S_H\), each coordinate of
\(\widetilde w_k\Pi_HX_k\) has centered sub-Gaussian norm at most \(C_0(1+R)\).
Coordinate-wise concentration and a union bound give, with probability at
least \(1-\zeta/2\),
\begin{equation}
\begin{aligned}
    &\left\|K^{-1}\sum_{k=1}^K\widetilde w_k\Pi_HX_k
      -\mathbb E[\widetilde w_1\Pi_HX_1]\right\|_2\\[-2pt]
    &\qquad\leq C_1(1+R)
      \sqrt{\frac{r\log(4r/\zeta)}{K}}.
\end{aligned}
    \label{eq:app_mppi_numerator_event}
\end{equation}
On the intersection of these events, the empirical denominator is at least
\(Z/2\).  Since
\(\mathbb E[\widetilde w_1\Pi_HX_1]=Z\Pi_HT_{\mathrm M}(\nu)\) and
\(\|T_{\mathrm M}(\nu)\|_2\leq\|\nu\|_2\leq R\), the standard ratio
decomposition and \(Z\geq\underline Z_{\mathrm M}(R)\) prove
Eq.~\eqref{eq:app_mppi_finite_K_bound} after absorbing numerical constants
into \(C_{\mathrm{conc}}\).
\end{proof}

The bound controls the sampling error in $\mathcal S_H$.
The remaining coordinates can still affect the importance weights, which
is why $n_H-r$ appears in the normalizer bound.
Combining it with Corollary~\ref{cor:cgs_finite_budget_app} gives a
projected mean-error bound for repeated updates. The next result states
the corresponding planning cost bound when all nonzero eigenvalues are
covered.

\subsubsection{Planning cost gap for MPPI}
\label{app:cgs_mppi_cost_bounds}
The result in the main context reports the planning cost at the proposal mean.
Assume $\xi_H<1$ and $r_{\rm ctrl}=r$, so
$\mathcal S_H=\operatorname{range}(G_H)$.
Keep $A,B,z_0,z_g$, $\sigma_{\mathrm p}$, and $\tau$ fixed, and define
\begin{equation}
    \nu_i:=\sigma_{\mathrm p}^{-1}(\mu_i-\mathbf a^\star),\qquad
    \Delta_i:=C(\mu_i)-C(\mathbf a^\star)
             =\tfrac12\sigma_{\mathrm p}^2\nu_i^\top G_H\nu_i.
\end{equation}
The identity follows from Eq.~\eqref{eq:app_exact_quadratic_expansion}.
The rank condition includes $d=d_a$ and the rectangular case with
$r_{\rm ctrl}=d_a$, as characterized in Section~\ref{app:cgs_low_dim_actions}.
Use a fresh batch of $K$ samples at each update and define $T_{\rm exit}$
as in Eq.~\eqref{eq:app_planner_first_exit}, here with
$\mathcal D=\{\nu:\|\nu\|_2\leq R\}$.
Let $\epsilon_K(\zeta/I)$ be the
right-hand side of Eq.~\eqref{eq:app_mppi_finite_K_bound} with failure
probability $\zeta/I$. Define the importance-weight factor
$\mathcal A_{\rm M}(R):=C_{\mathrm{conc}}(1+R)/\underline Z_{\rm M}(R)$,
with $\underline Z_{\rm M}(R)$ from
Eq.~\eqref{eq:app_mppi_normalizer_lower_bound}, and set
\begin{equation}
    b_K:=\sigma_{\mathrm p}\sqrt{cH(1+\xi_H)/2}\,\epsilon_K(\zeta/I).
    \label{eq:app_mppi_cost_sampling_error}
\end{equation}
The minimum sample count in Eq.~\eqref{eq:app_mppi_minimum_K} is also
required with failure probability $\zeta/I$.

\begin{corollary}[Planning cost with finite samples and updates]
Under these conditions, there is an event of probability at least
$1-\zeta$ on which the following bounds hold in the indicated ranges:
\begin{equation}
\begin{aligned}
    \sqrt{\Delta_{i+1}}&\leq q_{\rm M}\sqrt{\Delta_i}+b_K,
        &&0\leq i<\min(I,T_{\rm exit}),\\
    \Delta_I&\leq\left[
        q_{\rm M}^I\sqrt{\Delta_0}
        +\frac{1-q_{\rm M}^I}{1-q_{\rm M}}b_K
    \right]^2,
        &&I\leq T_{\rm exit}.
\end{aligned}
    \label{eq:app_mppi_cost_iteration_bound}
\end{equation}
For $0<\epsilon_{\rm tar}<\Delta_0$, it suffices to choose
\begin{equation}
    I\geq\left\lceil
        \frac{\log(4\Delta_0/\epsilon_{\rm tar})}{-2\log q_{\rm M}}
    \right\rceil,
    \qquad
    b_K\leq\frac{1-q_{\rm M}}{2}\sqrt{\epsilon_{\rm tar}}
    \label{eq:app_mppi_cost_budget_conditions}
\end{equation}
to obtain $\Delta_I\leq\epsilon_{\rm tar}$ on the same event whenever
$I\leq T_{\rm exit}$. These guarantees use the same stopped-process interpretation as
Corollary~\ref{cor:cgs_finite_budget_app}, not success probabilities
conditional on non-exit. Substituting the sampling bound
gives sufficient sample and update budgets for the cost bound in the main context.
\end{corollary}

\begin{proof}
Apply the stopped union-bound argument of
Eq.~\eqref{eq:app_planner_stopped_failure} to the full-mean radius-$R$
region. The fresh-batch estimate holds conditionally on the past whenever
the input mean is in this region. Thus, on an event of probability at
least $1-\zeta$,
$\|\Pi_H(\nu_{i+1}-J_{\rm M}\nu_i)\|_2\leq\epsilon_K(\zeta/I)$
for all $0\leq i<\min(I,T_{\rm exit})$.
Since $\mathcal S_H=\operatorname{range}(G_H)$,
$G_H^{1/2}(I_{n_H}-\Pi_H)=0$. Theorem~\ref{thm:cgs_horizon_curvature} gives
$\lambda_{\max}(G_H)\leq cH(1+\xi_H)$, hence
\begin{equation}
    \frac{\sigma_{\mathrm p}}{\sqrt2}\|G_H^{1/2}(\nu_{i+1}-J_{\rm M}\nu_i)\|_2\leq b_K.
\end{equation}
Moreover, $J_{\rm M}$ commutes with $G_H^{1/2}$ and has operator norm
$q_{\rm M}$ on $\operatorname{range}(G_H)$, so
\begin{equation}
    \frac{\sigma_{\mathrm p}}{\sqrt2}\|G_H^{1/2}J_{\rm M}\nu_i\|_2
    \leq q_{\rm M}\frac{\sigma_{\mathrm p}}{\sqrt2}\|G_H^{1/2}\nu_i\|_2
    =q_{\rm M}\sqrt{\Delta_i}.
\end{equation}
The triangle inequality proves the one-update bound before exit.
Summing the geometric series proves the bound after $I$ updates whenever
$I\leq T_{\rm exit}$.
The choices in Eq.~\eqref{eq:app_mppi_cost_budget_conditions} bound each
term in the square brackets by $\sqrt{\epsilon_{\rm tar}}/2$.
\end{proof}

In the rectangular case, if $\mathcal S_H$ contains only the leading
$d_a$ eigenvectors and additional positive eigenvalues remain outside it,
the same argument bounds the cost contribution
$\tfrac12\sigma_{\mathrm p}^2\nu_i^\top\Pi_HG_H\Pi_H\nu_i$.
A bound for the full planning cost gap requires control of all positive
eigenvalues, as in Corollary~\ref{cor:cgs_full_controllable_bound}.

\subsection{CEM: Local Mean Update}
\label{app:cgs_cem_population}

Fix an elite fraction $\alpha\in(0,1)$ and use
$r=\operatorname{rank}(P_H)=d_a$, $\mathcal S_H$, and $\Pi_H$ from
Section~\ref{app:cgs_low_dim_actions}.

\subsubsection{Population mean update near the optimum}

For \(X\sim\mathcal N(\nu,I_{n_H})\), let \(t_\alpha(M,\nu)\) be the
\(\alpha\)-quantile of \(X^\top MX\), and define
\begin{equation}
    T_{\mathrm C}(\nu;M)
    :=\mathbb E\!\left[
        X\,\middle|\,
        X^\top MX\leq t_\alpha(M,\nu)
    \right].
    \label{eq:app_cem_population_map}
\end{equation}

\begin{proposition}[Exact local CEM Jacobian and projector limit]
\label{prop:cgs_cem_population_full}
For every nonzero \(M\succeq0\), \(T_{\mathrm C}(0;M)=0\), and, for
\(X\sim\mathcal N(0,I_{n_H})\),
\begin{equation}
\begin{aligned}
    J_{\mathrm C}(M)
    &:=D_\nu T_{\mathrm C}(0;M)\\[-2pt]
    &=\frac1\alpha\mathbb E\!\left[
      XX^\top\mathbf1\{X^\top MX\leq t_\alpha(M,0)\}\right].
\end{aligned}
    \label{eq:app_cem_exact_jacobian}
\end{equation}
Let \(q_{\alpha,r}\) be the \(\alpha\)-quantile of
\(\chi_r^2\), and define
\begin{equation}
    \beta_{\alpha,r}
    :=\frac1r\mathbb E\!\left[\chi_r^2\mid
      \chi_r^2\leq q_{\alpha,r}\right]\in(0,1).
    \label{eq:app_cem_beta}
\end{equation}
At \(M=P_H\),
\begin{equation}
    J_{\mathrm C}(P_H)
    =\beta_{\alpha,r}P_H+(I_{n_H}-P_H).
    \label{eq:app_cem_ideal_jacobian}
\end{equation}

For \(0\leq\delta<1/2\), let \(\Pi_r(M)\) project onto the top \(r\)
eigenvectors of \(M\), and define
\begin{equation}
    \omega_{\alpha,n_H,r}(\delta)
    :=\sup_{\substack{M\succeq0\\\|M-P_H\|_{\op}\leq\delta}}
    \bigl\|\Pi_r(M)J_{\mathrm C}(M)\Pi_r(M)
    -\beta_{\alpha,r}\Pi_r(M)\bigr\|_{\op}.
    \label{eq:app_cem_modulus}
\end{equation}
Then
\begin{equation}
    \omega_{\alpha,n_H,r}(\delta)\longrightarrow0
    \qquad\text{as }\delta\downarrow0.
    \label{eq:app_cem_modulus_limit}
\end{equation}
Consequently, if \(\xi_H<1/2\), the top-\(r\) subspace
\(\mathcal S_H\) of \(M_H\) is invariant under \(J_{\mathrm C}(M_H)\) and,
in orthonormal coordinates on this subspace,
\begin{equation}
    \left\|J_{\mathrm C}(M_H)|_{\mathcal S_H}
      -\beta_{\alpha,r}I_r\right\|_{\op}
    \leq\omega_{\alpha,n_H,r}(\xi_H).
    \label{eq:app_cem_active_balance}
\end{equation}
If \(\beta_{\alpha,r}+\omega_{\alpha,n_H,r}(\xi_H)<1\), then every
\begin{equation}
    q_{\mathrm C}\in
    (\beta_{\alpha,r}+\omega_{\alpha,n_H,r}(\xi_H),1)
\end{equation}
satisfies
\(\|\Pi_H T_{\mathrm C}(\nu;M_H)\|_2
\leq q_{\mathrm C}\|\Pi_H\nu\|_2\)
for all sufficiently small full mean vectors \(\nu\), where
\(\Pi_H=\Pi_r(M_H)\). In particular, it is a local contraction factor
for means in \(\mathcal S_H\).
\end{proposition}

\begin{proof}
For fixed nonzero \(M\succeq0\), define
\begin{equation}
    F(\nu,t)
    :=\mathbb P_{X\sim\mathcal N(\nu,I_{n_H})}(X^\top MX\leq t).
\end{equation}
For nonzero $M\succeq0$, the quadratic form has a positive continuous
density for $t>0$, and its $\alpha$-quantile is strictly positive.
The Gaussian density is smooth in $\nu$, so the implicit function theorem
applies to $F(\nu,t_\alpha(M,\nu))=\alpha$ near $\nu=0$ and gives
a differentiable quantile.
By central symmetry,
\begin{equation}
    \nabla_\nu F(0,t)
    =\mathbb E[X\mathbf1\{X^\top MX\leq t\}]
    =0,
\end{equation}
so \(D_\nu t_\alpha(M,0)=0\).  Differentiating the unnormalized elite
first moment under the integral sign gives
Eq.~\eqref{eq:app_cem_exact_jacobian}.

For \(M=P_H\), decompose
\(X=P_HX+(I_{n_H}-P_H)X\).  The components are independent, and the elite event
is \(\|P_HX\|_2^2\leq q_{\alpha,r}\).  Rotational symmetry on
\(\operatorname{range}(P_H)\) gives conditional second moment
\(\beta_{\alpha,r}P_H\); the orthogonal component remains standard
Gaussian and the cross moments vanish.  This proves
Eq.~\eqref{eq:app_cem_ideal_jacobian}.  Conditioning a non-degenerate
\(\chi_r^2\) variable on a strict lower-tail event reduces its mean, so
\(0<\beta_{\alpha,r}<1\).

To prove continuity, let \(M_n\to P_H\).  For
\(X\sim\mathcal N(0,I_{n_H})\),
\(X^\top M_nX\to X^\top P_HX\) almost surely.  The limit has the
continuous \(\chi_r^2\) law, so the corresponding \(\alpha\)-quantiles
converge to \(q_{\alpha,r}\).  The elite indicators then converge almost
surely outside a null boundary event.  Since
\(\|XX^\top\|_{\op}=\|X\|_2^2\) is integrable, dominated convergence in
Eq.~\eqref{eq:app_cem_exact_jacobian} gives
\(J_{\mathrm C}(M_n)\to J_{\mathrm C}(P_H)\).  The spectral gap of
\(P_H\) also gives \(\Pi_r(M_n)\to P_H\) whenever
\(\|M_n-P_H\|_{\op}<1/2\).  These two facts imply
Eq.~\eqref{eq:app_cem_modulus_limit} by contradiction.

In an eigenbasis of \(M_H\), the elite event depends only on squared
coordinates.  All off-diagonal conditional second moments vanish, so
\(J_{\mathrm C}(M_H)\) commutes with \(M_H\) and leaves
\(\mathcal S_H\) invariant.  The balance bound follows from the definition
of \(\omega_{\alpha,n_H,r}\). To obtain the projected contraction, write
\(\nu=e+\eta\), with \(e=\Pi_H\nu\) and
\(\eta=(I_{n_H}-\Pi_H)\nu\). Reflection symmetry of the active coordinates gives
\(\Pi_HT_{\mathrm C}(\eta;M_H)=0\). Continuity of the derivative near zero
bounds \(\|\Pi_HD_\nu T_{\mathrm C}(\nu;M_H)\Pi_H\|_{\op}\) by
\(q_{\mathrm C}\) on a sufficiently small ball. Integrating along
\(\eta+te\), \(0\leq t\leq1\), proves the claimed contraction.
\end{proof}

Multiplying the cost by a positive scalar does not change the CEM elite set.
Thus its ideal contraction depends on \(\alpha\) and \(r\), while CGS
controls the deviation from a balanced active update through \(\xi_H\).
When CEM also updates covariance, repeated application requires a uniform
bound on the proposal-whitened matrices
\(\Sigma_i^{1/2}G_H\Sigma_i^{1/2}\); no full coupled
mean--covariance contraction is asserted.

\paragraph{Conditions and dimension cases.}
The modulus $\omega_{\alpha,n_H,r}$ gives continuity at the ideal
projector; no linear rate in $\xi_H$ is asserted.
Contraction requires both $\xi_H<1/2$ and
$\beta_{\alpha,r}+\omega_{\alpha,n_H,r}(\xi_H)<1$, with the full mean
inside the neighborhood specified in the proposition.
When $r_{\rm ctrl}=r$, the result covers every direction that affects the
cost. When $r_{\rm ctrl}>r$, it controls the leading $r$ directions;
the additional positive eigenvalues still enter the elite event and need
not have a comparable contraction factor.

\subsubsection{Finite-sample and repeated mean updates}

For the one-step sampling bound, assume $\xi_H<1$ and $\|\nu\|_2\leq R$.
All samples in a batch are independent.

\begin{proposition}[Sampling error in one CEM update]
\label{prop:cgs_cem_finite_k_concentration}
For CEM, draw a fresh batch
\(X_1,\ldots,X_K\sim\mathcal N(\nu,I_{n_H})\), assume
\(K_{\rm elite}:=\alpha K\in\mathbb N\), let \(\widehat t_\alpha\) be the
\(K_{\rm elite}\)-th order statistic of \(Y_j:=X_j^\top M_HX_j\), and define
\begin{equation}
    \widehat T_{\mathrm C,K}
    :=\frac{1}{K_{\rm elite}}\sum_{j=1}^K
      X_j\mathbf1\{Y_j\leq\widehat t_\alpha\}.
\end{equation}
There is a universal constant \(C_{\mathrm{conc}}>0\) such that, for every \(\zeta\in(0,1)\), with probability at least
\(1-\zeta\),
\begin{equation}
\begin{aligned}
    &\left\|
        \Pi_H(\widehat T_{\mathrm C,K}-T_{\mathrm C}(\nu;M_H))
    \right\|_2\\
    &\quad\leq
    \frac{C_{\mathrm{conc}}(1+R)}{\alpha}
    \sqrt{\frac{r\log(6r/\zeta)}{K}}
    +\frac{R_{\mathrm C,K}(R,\zeta)}{\alpha}
     \sqrt{\frac{\log(6/\zeta)}{2K}},
\end{aligned}
    \label{eq:app_cem_finite_K_bound}
\end{equation}
where
\begin{equation}
    R_{\mathrm C,K}(R,\zeta)
    :=R+\sqrt r+\sqrt{2\log(3K/\zeta)}.
    \label{eq:app_cem_RK}
\end{equation}
Thus the projected CEM mean update converges to its population map at the
canonical \(K^{-1/2}\) rate, up to logarithmic factors.
\end{proposition}

\begin{proof}
Let \(F_Y\) and \(F_{Y,K}\) be the population and empirical CDFs of
\(Y=X^\top M_HX\), and let \(t_\alpha\) be its population
\(\alpha\)-quantile.  The law of \(Y\) is continuous because \(M_H\) has
at least \(r\) positive eigenvalues.  The
Dvoretzky--Kiefer--Wolfowitz inequality gives, with probability at least
\(1-\zeta/3\),
\begin{equation}
    \sup_t|F_{Y,K}(t)-F_Y(t)|
    \leq
    u_K:=\sqrt{\frac{\log(6/\zeta)}{2K}}.
    \label{eq:app_cem_DKW_event}
\end{equation}
With probability one there are no ties, so
\(F_{Y,K}(\widehat t_\alpha)=\alpha\).  On the event above,
\(|F_{Y,K}(t_\alpha)-\alpha|\leq u_K\); because threshold sets are nested,
the two selected index sets differ in at most \(Ku_K\) samples.

A Gaussian norm tail bound and a union bound give, with probability at least
\(1-\zeta/3\),
\begin{equation}
    \max_{1\leq j\leq K}\|\Pi_HX_j\|_2
    \leq R_{\mathrm C,K}(R,\zeta).
    \label{eq:app_cem_max_event}
\end{equation}
Thus replacing \(t_\alpha\) by \(\widehat t_\alpha\) changes the
projected empirical first moment by at most \(R_{\mathrm C,K}(R,\zeta)u_K\).

At the fixed threshold \(t_\alpha\), every coordinate of
\(\Pi_HX_j\mathbf1\{Y_j\leq t_\alpha\}\) is sub-Gaussian with centered
norm at most \(C_0(1+R)\).  Therefore, with probability at least
\(1-\zeta/3\),
\begin{equation}
\begin{aligned}
    &\left\|
        K^{-1}\sum_{j=1}^K
        \Pi_HX_j\mathbf1\{Y_j\leq t_\alpha\}
        -\mathbb E[\Pi_HX\mathbf1\{Y\leq t_\alpha\}]
    \right\|_2\\
    &\qquad\leq
    C_1(1+R)
    \sqrt{\frac{r\log(6r/\zeta)}{K}}.
\end{aligned}
\end{equation}
Both empirical and population CEM means divide their first moments by
\(\alpha\).  Combining the three events proves
Eq.~\eqref{eq:app_cem_finite_K_bound}.
\end{proof}

\paragraph{Repeated updates.}
Assume the contraction conditions of
Proposition~\ref{prop:cgs_cem_population_full}, and choose $q_{\rm C}$
and a full-mean neighborhood radius $R_\nu$ as in that proposition.
Keep the proposal covariance fixed and draw a fresh independent batch
at each update. Let $e_i=\Pi_H\nu_i$ and let $\epsilon_K^{\rm C}(\eta)$ denote the
right-hand side of Eq.~\eqref{eq:app_cem_finite_K_bound} with $R=R_\nu$
and failure probability $\eta$.
Assume $\|e_0\|_2\leq R_e$ and
$\epsilon_K^{\rm C}(\zeta/I)\leq(1-q_{\rm C})R_e$, and define
$T_{\rm exit}$ by Eq.~\eqref{eq:app_planner_first_exit} with these radii.
Corollary~\ref{cor:cgs_finite_budget_app} gives an event of probability
at least $1-\zeta$ on which, whenever $I\leq T_{\rm exit}$,
\begin{equation}
    \|e_I\|_2\leq q_{\rm C}^I\|e_0\|_2
        +\frac{1-q_{\rm C}^I}{1-q_{\rm C}}\epsilon_K^{\rm C}(\zeta/I).
    \label{eq:app_cem_finite_iteration_bound}
\end{equation}
This controls the projected mean error, with the additional logarithms
from the empirical elite threshold retained in $\epsilon_K^{\rm C}$.
As in Corollary~\ref{cor:cgs_finite_budget_app}, this statement bounds
failure before exit; it does not condition on the trajectory remaining
in the full-mean neighborhood.
The covariance remains fixed throughout the analyzed updates.

\subsection{Gradient Descent: Planning Cost Convergence}
\label{app:cgs_ts_style}

We consider unconstrained gradient descent on the quadratic cost $C$ of
Section~\ref{app:cgs_mpc_background}. Let $\mathbf a_i$ be its iterates,
$\Delta_i:=C(\mathbf a_i)-C(\mathbf a^\star)$, and $h>0$ the step size.

\subsubsection{Convergence from the CGS spectral bounds}

\begin{proposition}[Gradient descent under CGS geometry]
\label{prop:cgs_gd_cost_convergence}
Assume the three $\varepsilon$-CGS bounds with $0\leq\varepsilon<1$
and $\xi_H<1$.
Gradient descent with $h=1/(cH)$ satisfies
\begin{equation}
\begin{aligned}
    \mathbf a_{i+1}&=\mathbf a_i-h\nabla C(\mathbf a_i),\\
    \mathbf a_{i+1}-\mathbf a^\star
        &=\left(I_{n_H}-\frac{G_H}{cH}\right)(\mathbf a_i-\mathbf a^\star).
\end{aligned}
    \label{eq:app_gd_exact_update}
\end{equation}
On the leading $r=d_a$ eigenspace $\mathcal S_H$,
\begin{equation}
    \left\|\left(I_{n_H}-\frac{G_H}{cH}\right)\bigg|_{\mathcal S_H}\right\|_{\op}
    \leq\xi_H.
    \label{eq:app_gd_dominant_contraction}
\end{equation}
Define the cost contribution of these directions by
\begin{equation}
    \Delta_i^{\rm dom}:=\tfrac12(\mathbf a_i-\mathbf a^\star)^\top
        \Pi_HG_H\Pi_H(\mathbf a_i-\mathbf a^\star).
\end{equation}
For every integer $I\geq1$,
\begin{equation}
    \Delta_I^{\rm dom}\leq\xi_H^{2I}\Delta_0^{\rm dom}.
    \label{eq:app_gd_dominant_cost_bound}
\end{equation}
If $r_{\rm ctrl}=r$, then $\Delta_i^{\rm dom}=\Delta_i$, and the same
bound holds for the full planning cost gap.
\end{proposition}

\begin{proof}
Equation~\eqref{eq:app_exact_quadratic_expansion} gives
$\nabla C(\mathbf a)=G_H(\mathbf a-\mathbf a^\star)$, proving the
update formula. The projector $\Pi_H$ commutes with $G_H$.
In an eigenbasis, each of the leading $r$ error coordinates is multiplied
by $1-\lambda_j(G_H)/(cH)$, whose absolute value is at most $\xi_H$
by Theorem~\ref{thm:cgs_horizon_curvature}.
Squaring each coordinate and weighting it by $\lambda_j(G_H)/2$ proves
the one-step cost bound; iteration gives
Eq.~\eqref{eq:app_gd_dominant_cost_bound}.
When $\mathcal S_H=\operatorname{range}(G_H)$, all other directions
have zero cost contribution, so $\Delta_i^{\rm dom}=\Delta_i$.
\end{proof}

\paragraph{Step size and full-spectrum conditioning.}
More generally, let $L=\lambda_{\max}(G_H)$ and
$m=\lambda_{\min}^{+}(G_H)>0$. The update on a positive-eigenvalue
direction has multiplier $1-h\lambda_j(G_H)$.
Every such direction contracts when $0<h<2/L$.
Minimizing $\max\{|1-hm|,|1-hL|\}$ gives
$h_\star=2/(L+m)$ and contraction factor
$q_{\rm GD}=(\kappa_{\rm eff}(G_H)-1)/(\kappa_{\rm eff}(G_H)+1)$.
Consequently, $\Delta_I\leq q_{\rm GD}^{2I}\Delta_0$.
When $r_{\rm ctrl}=r$, the CGS spectral bound gives $q_{\rm GD}\leq\xi_H$.
At ideal geometry, one step with $h=1/(cH)$ reaches the minimum cost.

\paragraph{Dimension cases.}
If $d=d_a$, $r_{\rm ctrl}=r$ and the full cost bound applies.
It also applies when $d>d_a$ but $r_{\rm ctrl}=d_a$, for example under
the invariance condition in Corollary~\ref{cor:cgs_full_controllable_bound}.
If $r_{\rm ctrl}>d_a$, the bound with factor $\xi_H$ covers the dominant
cost contribution. The remaining positive eigenvalues may be arbitrarily
small, so their gradient updates may contract arbitrarily slowly.
A quantitative bound on all positive eigenvalues, such as the one in
Corollary~\ref{cor:cgs_full_controllable_bound}, instead controls
$q_{\rm GD}$ for the full cost.
These statements concern ordinary gradient descent on $C$; the Adam
optimizer used in our experiments requires a separate optimizer analysis.

\subsubsection{Comparison using temporal straightening and control isotropy}

The bound below uses only $\|A-I_d\|_{\op}$ and $\|B^\top B-cI_{d_a}\|_{\op}$.
It recovers the temporal-straightening-style conditioning argument in the
square-action case. Proposition~\ref{prop:cgs_gd_cost_convergence} above
uses the full CGS spectral bound, including drift--control decoupling.

\begin{corollary}[Bound in the square-action case]
\label{cor:cgs_ts_style_condition}
Suppose \(d_a=d\), and define
\(\delta_A:=\|A-I_d\|_{\op}<1\) and
\(\delta_B:=c^{-1}\|B^\top B-cI_{d_a}\|_{\op}<1\).  Then \(B\) is invertible and
\begin{equation}
\begin{aligned}
    \kappa_{\mathrm{eff}}(G_H)
    &=\kappa(W_H)\\
    &\leq
    \frac{1+\delta_B}{1-\delta_B}
    \frac{\sum_{k=0}^{H-1}(1+\delta_A)^{2k}}
         {\sum_{k=0}^{H-1}(1-\delta_A)^{2k}}\\
    &\leq
    \frac{1+\delta_B}{1-\delta_B}
    \left(\frac{1+\delta_A}{1-\delta_A}\right)^{2(H-1)}.
\end{aligned}
    \label{eq:app_cgs_ts_style_condition}
\end{equation}
Under the \(\varepsilon\)-CGS bounds,
\begin{equation}
    \kappa_{\mathrm{eff}}(G_H)
    \leq
    \frac{1+\varepsilon}{1-\varepsilon}
    \left(\frac{1+\sqrt\varepsilon}{1-\sqrt\varepsilon}\right)^{2(H-1)}.
    \label{eq:app_cgs_ts_style_epsilon}
\end{equation}
\end{corollary}

\begin{proof}
For every unit vector \(x\),
\begin{equation}
    x^\top W_Hx
    =\sum_{k=0}^{H-1}\|B^\top(A^\top)^kx\|_2^2.
\end{equation}
Thus
\begin{equation}
\begin{aligned}
    \lambda_{\max}(W_H)
    &\leq\sigma_{\max}(B)^2
      \sum_{k=0}^{H-1}\sigma_{\max}(A)^{2k},\\
    \lambda_{\min}(W_H)
    &\geq\sigma_{\min}(B)^2
      \sum_{k=0}^{H-1}\sigma_{\min}(A)^{2k}.
\end{aligned}
\end{equation}
The assumptions give
\(\sigma_{\max}(A)\leq1+\delta_A\),
\(\sigma_{\min}(A)\geq1-\delta_A\), and
\(\kappa(B)^2=\kappa(B^\top B)\leq(1+\delta_B)/(1-\delta_B)\).
Dividing the two bounds proves the first inequality in
Eq.~\eqref{eq:app_cgs_ts_style_condition}.  The ratio of the two positive
sums is at most the largest termwise ratio, proving the second.  Under the
\(\varepsilon\)-CGS bounds, \(\delta_A\leq\sqrt\varepsilon\) and
\(\delta_B\leq\varepsilon\), which gives
Eq.~\eqref{eq:app_cgs_ts_style_epsilon}.
\end{proof}

\subsection{Geometry under Other Training Objectives}
\label{app:cgs_prediction_constraint}

In the trained world model, \(A\) and \(B\) are jointly induced by encoder
and predictor parameters \(\vartheta=(\theta,\phi)\). Prediction and
other training objectives also shape these parameters, and their preferred representations
may differ from those that minimize CGS alone. CGS can still improve control
geometry while preserving the structure required by the other training
objectives, even when its loss remains nonzero. The presence of another loss
does not itself imply a positive
loss floor: the objectives may share a zero-CGS solution.

Let \(\mathcal L_{\mathrm{other}}
:=\mathcal L_{\mathrm{pred}}+\mathcal L_{\mathrm{aux}}\) collect prediction
and all weighted non-CGS auxiliary objectives, with
\(\mathcal L_{\mathrm{aux}}=0\) when none are present. For LeWM,
\(\mathcal L_{\mathrm{aux}}=\lambda_{\mathrm{sig}}\mathcal L_{\mathrm{sig}}\),
so \(\mathcal L_{\mathrm{other}}=\mathcal L_{\mathrm{LeWM}}\).
Using the same total-loss notation as in
Eq.~\eqref{eq:cgs-lewm-training}, consider
\begin{equation}
    \mathcal L(\vartheta)
    =\mathcal L_{\mathrm{other}}(\vartheta)
     +\lambda_{\mathrm{CGS}}\mathcal L_{\mathrm{CGS}}(\vartheta),
    \label{eq:app_cgs_training_objective}
\end{equation}
and the nonempty set of models meeting a prescribed tolerance
\(\ell_{\mathrm{other}}\geq0\) for these other objectives,
\begin{equation}
    \Theta_{\ell_{\mathrm{other}}}
    :=\{\vartheta:\mathcal L_{\mathrm{other}}(\vartheta)\leq\ell_{\mathrm{other}}\}.
    \label{eq:app_prediction_feasible_set}
\end{equation}
A positive loss floor arises if these requirements keep the Gram defect
uniformly away from zero. One sufficient condition is that every
\(\vartheta\in\Theta_{\ell_{\mathrm{other}}}\) must retain temporal drift
\begin{equation}
    \|A_\vartheta-I_d\|_{\op}\geq\delta_{\mathrm{drift}}>0,
    \label{eq:app_irreducible_temporal_drift}
\end{equation}
and Assumption~\ref{ass:cgs_joint_coverage} holds uniformly on this
set with coverage constant at least \(\rho\). Write
\(D_\vartheta:=A_\vartheta-I_d\) and keep \(c>0\) fixed. Then
Proposition~\ref{prop:cgs_joint_straightening} implies
\begin{equation}
    \inf_{\vartheta\in\Theta_{\ell_{\mathrm{other}}}}
    \mathcal L_G^\vartheta(c)
    \geq\rho^2\delta_{\mathrm{drift}}^4.
    \label{eq:app_cgs_prediction_floor}
\end{equation}
Indeed,
\(\|D_\vartheta^\top D_\vartheta\|_F^2
 \geq\|D_\vartheta\|_{\op}^4\geq\delta_{\mathrm{drift}}^4\).
This gives a positive CGS Gram-loss floor whenever fitting the other
objectives requires irreducible temporal drift. More generally, the same
argument applies whenever those objectives keep \(\mathcal R(A,B)\)
uniformly bounded away from zero over \(\Theta_{\ell_{\mathrm{other}}}\).
A large value of \(\|(A-I_d)z_t\|_2\) alone does not establish such a floor:
\(A-I_d\) may still be reducible without compromising the other objectives.

A positive loss floor limits how closely the model can approach ideal CGS
geometry. Feasible adjustments to \(A\) and \(B\) can still improve control
geometry. CGS can favor feasible reductions in drift--control coupling and
adjustments to the control map \(B\), even when temporal drift must be
retained. The fixed-drift calculation below illustrates this remaining
freedom through the Gram-defect surrogate; its optimizer describes the
jointly trained model only when compatible with the other objectives.

\subsubsection{Effect of irreducible drift on the control map}

For fixed \(D=A-I_d\), the \(B\)-dependent part of
\(\mathcal R(A,B)\) is
\begin{equation}
    \mathcal R_B(D,B)
    :=2\|D^\top B\|_F^2+\|B^\top B-cI_{d_a}\|_F^2.
    \label{eq:app_fixed_D_objective}
\end{equation}
Let
\(\lambda_k^D:=\lambda_{d-k+1}(DD^\top)\), \(1\leq k\leq d\).
Thus \(0\leq\lambda_1^D\leq\cdots\leq\lambda_d^D\) lists the drift
eigenvalues in nondecreasing order, while \(\lambda_j(M)\) retains the
nonincreasing convention from the main context.

\begin{proposition}[Optimal control map under fixed temporal drift]
\label{prop:cgs_fixed_d_optimal_b}
For fixed \(D\) and \(d\geq d_a\), a minimizer of
Eq.~\eqref{eq:app_fixed_D_objective} selects the left singular directions
of \(B\) from eigenspaces associated with the \(d_a\) smallest eigenvalues
of \(DD^\top\), with squared singular values
\begin{equation}
    \sigma_k^2(B_\star)
    =(c-\lambda_k^D)_+,
    \qquad
    k=1,\ldots,d_a,
    \label{eq:app_fixed_D_optimal_B}
\end{equation}
where \((u)_+:=\max\{u,0\}\).  The right singular vectors are arbitrary.
\end{proposition}

\begin{proof}
Let \(\upsilon_k:=\sigma_k^2(B)\), ordered as
\(\upsilon_1\geq\cdots\geq \upsilon_{d_a}\geq0\).  Since
\begin{equation}
\begin{aligned}
    \|D^\top B\|_F^2&=\operatorname{tr}(DD^\top BB^\top),\\[-2pt]
    \|B^\top B-cI_{d_a}\|_F^2&=\sum_{k=1}^{d_a}(\upsilon_k-c)^2.
\end{aligned}
\end{equation}
the second term is independent of the singular vectors.  For fixed
\((\upsilon_k)\), von Neumann's trace inequality pairs the largest \(\upsilon_k\) with
the smallest eigenvalues of \(DD^\top\).  Writing \(U_B\) for the left
singular vectors,
\begin{equation}
    \min_{U_B}\mathcal R_B(D,B)
    =\sum_{k=1}^{d_a}\bigl[2\lambda_k^D\upsilon_k+(\upsilon_k-c)^2\bigr].
\end{equation}
Each term is strictly convex on \(\upsilon_k\geq0\) and is minimized at
\(\upsilon_k=(c-\lambda_k^D)_+\).  Since \(\lambda_k^D\) is nondecreasing, these
minimizers are nonincreasing and respect the ordering of \((\upsilon_k)\).
\end{proof}

If \(\dim\ker(D^\top)\geq d_a\), the defect objective can place the action
map inside \(\ker(D^\top)\) while retaining \(B^\top B=cI_{d_a}\).
Thus, in this surrogate, nonzero temporal drift can coexist with decoupled
and balanced action effects. When action-controlled
directions overlap with non-removable temporal dynamics, CGS favors weaker
drift directions. Unequal \(\lambda_k^D<c\) yield unequal positive singular
values in \(B_\star\), whereas all \(\lambda_k^D\geq c\) yield zero singular
values, even when the drift eigenvalues differ.  A common reduction of all singular values is only
a scale change; unequal reductions directly worsen control conditioning.
This calculation minimizes the unweighted defect \(\mathcal R\).
For a general correlated \(\Sigma_x\), spectral equivalence bounds
\(\mathcal L_G\) but does not make its minimizer identical to that of
\(\mathcal R\); the calculation illustrates the tradeoff in this surrogate.

\paragraph{Reacher as an illustrative case.}

Reacher illustrates a regime in which one-step action geometry is only an
indirect reference for visual latent change. Joint torques affect
$\Delta z_t$ through acceleration, velocity, configuration, and momentum, so
accurate prediction may require a non-negligible drift term in
\[
    \Delta z_t = D z_t + B a_t,
    \qquad D=A-I_d .
\]
For fixed $D$, Proposition~\ref{prop:cgs_fixed_d_optimal_b} gives
\[
    \sigma_k^2(B^\star)=(c-\lambda_k^D)_+ .
\]
Thus stronger drift lowers the preferred action sensitivity along the
corresponding latent direction. For fixed $\lambda_k^D$, overly large
sensitivities receive stronger downward pressure, although differences in
$\lambda_k^D$ prevent any general ordering of how much the largest and
smallest singular values change. The spectral spread can therefore shrink
mainly through suppression of strong action sensitivities, while the weakest
sensitivity changes little.

This distinction can affect planners differently. For quadratic planning,
optimal fixed-step gradient descent depends on
$\kappa_{\mathrm{eff}}(G_H)=L/m$, whereas the slowest active-direction
population MPPI contraction is
\[
    q_M=\frac{1}{1+(\sigma_p^2/\tau)m}.
\]
Hence a change dominated by reducing $L$ with little increase in $m$ can
improve gradient-based conditioning more clearly than MPPI contraction.
CEM may show a similarly limited gain because its elite update depends on
the overall active spectral shape and the evolving proposal covariance, so
reducing excessive curvature alone need not make elite selection substantially
more balanced. This is consistent with the Reacher results, where straightening methods improve more significant under GD but modest on sampling-based planners.

\end{document}